\documentclass{article}

\usepackage{iclr/fancyhdr}
\usepackage{iclr/iclr2027_conference}
\usepackage{times}
\usepackage{thmtools}
\usepackage{thm-restate}
\usepackage{wrapfig}
\usepackage{algorithm}
\usepackage{booktabs}
\usepackage{algpseudocode}
\usepackage{amsmath,amssymb}
\newcommand{\X}{\mathcal{X}}
\newcommand{\W}{\mathcal{W}}

\newcommand{\mP}{\mathcal{P}}
\newcommand{\R}{\mathbb{R}}
\usepackage{amsthm}
\usepackage{algorithm}
\usepackage{algpseudocode}
\newtheorem{theorem}{Theorem}

\newtheorem{lemma}{Lemma}

\usepackage{booktabs}

\usepackage{multirow}
\usepackage{subcaption}
\usepackage{graphicx}
\usepackage{xcolor}

\usepackage[table]{xcolor}

\definecolor{bestgreen}{RGB}{220,245,220}
\definecolor{secondblue}{RGB}{222,235,255}

\newcommand{\best}[1]{\cellcolor{bestgreen}\textbf{#1}}
\newcommand{\second}[1]{\cellcolor{secondblue}#1}
\newcommand{\ourmha}{\our}

\usepackage{hyperref}
\usepackage{todonotes}
\usepackage{url}
\usepackage{iclr/defs}

\allowdisplaybreaks

\author{%
  {\fontsize{9}{10.5}\selectfont
  \mbox{Soutrik Sarangi$^{1}$\thanks{Equal contribution.}\hspace{0.55em}%
  Yonatan Sverdlov$^{2}$\footnotemark[1]\hspace{0.55em}%
  Adir Dayan$^{3}$\hspace{0.55em}%
  Haggai Maron$^{3,4}$\hspace{0.55em}%
  Nadav Dym$^{2}$}}\\[0.6em]
  $^1$Microsoft \\
  $^2$Faculty of Mathematics, Technion -- Israel Institute of Technology \\
  $^3$Faculty of Electrical and Computer Engineering, Technion -- Israel Institute of Technology \\
  $^4$NVIDIA \\[0.6em]
  \texttt{soutriksarangi14@gmail.com} \quad \texttt{yonatans@campus.technion.ac.il} \\
  \texttt{adir.dayan@campus.technion.ac.il} \quad \texttt{haggaimaron@technion.ac.il} \\
  \texttt{nadavdym@technion.ac.il}
}

\iclrfinalcopy
\let\originalmaketitle\maketitle
\renewcommand{\maketitle}{%
  \originalmaketitle
  \lhead{Preprint}%
}

\begin{document}
\bibliographystyle{iclr/iclr2027_conference}

\newcommand{\RR}{\mathbb{R}}
\newcommand{\dmax}{d_{\text{max}}}
\newcommand{\probe}{p}
\newcommand{\probeseq}{P}
\newcommand{\paramdim}{q}
\newcommand{\learningparam}{\phi}
\newcommand{\mhaalgo}{\textsc{RecoverMHA}\xspace}
\newcommand{\ys}[1]{\textcolor{green}{[YS: #1]}}
\newcommand{\nd}[1]{\textcolor{red}{[ND: #1]}}
\renewcommand{\ss}[1]{\textcolor{purple}{[SS: #1]}}
\newcommand{\fix}{\marginpar{FIX}}
\newcommand{\new}{\marginpar{NEW}}
\newtheorem{proposition}[theorem]{Proposition}

\title{Finite Probes Suffice: Identifiability and Universality for Weight-Space Learning}

\maketitle

\begin{abstract}
Learning properties of neural networks has recently attracted growing interest, with existing approaches operating either directly on network parameters or through probe-based representations of network behavior. While probing methods have shown strong empirical performance, their theoretical foundations remain limited.

In this work, we study when finite probe-based representations are sufficient for learning neural functionals. We establish general identification and universality results for probing, and show that using intermediate hidden representations can provide significantly more informative representations than relying only on final outputs. Motivated by these results, we introduce \our, a simple architecture for learning from hidden probe responses. Across a range of neural functional benchmarks, including both MLPs and Transformers, \our consistently improves over existing probing methods and achieves state-of-the-art performance. Our code is publicly available on \href{https://github.com/yonatansverdlov/HiddenProbe}{GitHub}.

\end{abstract}
\allowdisplaybreaks
\section{Introduction}
In recent years, there has been increasing interest in the topic of learning neural functionals (also known as weight space learning~\cite{haggaisurvey}). A neural functional is a mapping $\psi$ taking a neural network $f_\theta$  to a vector $\psi(\theta)\in\RR^m$. For example, one could consider a neural functional predicting the test accuracy of a given neural network. Another example is where $f_\theta$ is an INR representing an image, and then standard image tasks like image classification become neural functionals. 

A popular approach to learning neural functionals $\psi(f_\theta)$ is via \emph{weight space methods}, which think of the functionals as a function of the parameters $\tilde \psi(\param)=\psi(f_\theta)$, and then apply neural architectures directly to the parameter weights to learn $\tilde \psi$. Several recent works have adopted this approach with some success \cite{dws_paper, zhou2023neural,kofinas, zhou2023permutation}.
However, weight-space learning must contend with the fact that many distinct parameter vectors can represent the same function. Accounting for these parameter symmetries is often important for performance, but typically requires symmetry-aware architectures that are more complex to design and implement.

Probing provides an alternative way to process neural models that avoids the parameterization ambiguities of weight-space methods. Rather than representing $f_\theta$ through its parameters $\theta$, probing represents it through its evaluations on a set of inputs, called probes, $p_1,\ldots,p_N$. The learning problem is then to find a function $\hat{\psi}\bigl(f_\theta(p_1),\ldots,f_\theta(p_N)\bigr)$
that approximates the desired neural functional $\psi(f_\theta)$.
Probe-based methods, such as \cite{kahana}, have achieved strong empirical performance on several benchmarks, but currently lack a theoretical justification. In particular, while weight space methods like \cite{dws_paper} can provably approximate all neural functionals \cite{dayan2026on}, analogous universality guarantees are not known for probe-based methods. Moreover, whereas the parameters $\theta$ fully specify $f_\theta$, a finite probe set reveals only the values of $f_\theta$ at finitely many inputs and therefore appears, a priori, to provide only partial information about the realized function. This raises two closely related questions:


\textbf{Main Questions:} \begin{enumerate}
\item Can a function $f_\theta$ be uniquely reconstructed from a finite number of probes? If so, how many probes are needed? 
\item Can any neural functional $\Lambda(f_\theta) $ be approximated by a probe-based model?
\end{enumerate}

\textbf{Main Results:}
As we shall see, our answer to both questions will be positive. Under the assumption that $f_\theta$ is analytic, we show that the function $f_\theta$ can be uniquely recovered from $2q+1 $ probes, where $q$ denotes the dimension of the parameter space. As a result, we will show that any neural functions $\Lambda(f_\theta) $ can be approximated by a probe-based model with $2q+1 $ probes.

In addition, we will discuss probe based methods which represent $f_\theta$ not only through the final probe value $f_\theta(p) $, but also through the value of all intermediate layers on the probes. For both MLP and self-attention, we will show that this representation requires fewer probes to obtain the same guarantees. For fully connected neural networks, this brings down the sufficient probe count from $O(q)$ to $O(d_{\max})$, where $d_{\max}$ is the maximal width, and for self-attention blocks this brings down the count to $O(d_{in}h)$, with $d_{in}$ and $h$ being the input dimension and number of heads, respectively. 
Motivated by this we suggest a specialized model to learn properties of neural networks with such probes, which we name \our (see Figure \ref{fig:model}). Empirically, we will show that \our attains state of the art results on several standard neural functional datasets.

\begin{figure}[h]
    \centering
    \includegraphics[width=0.8\linewidth]{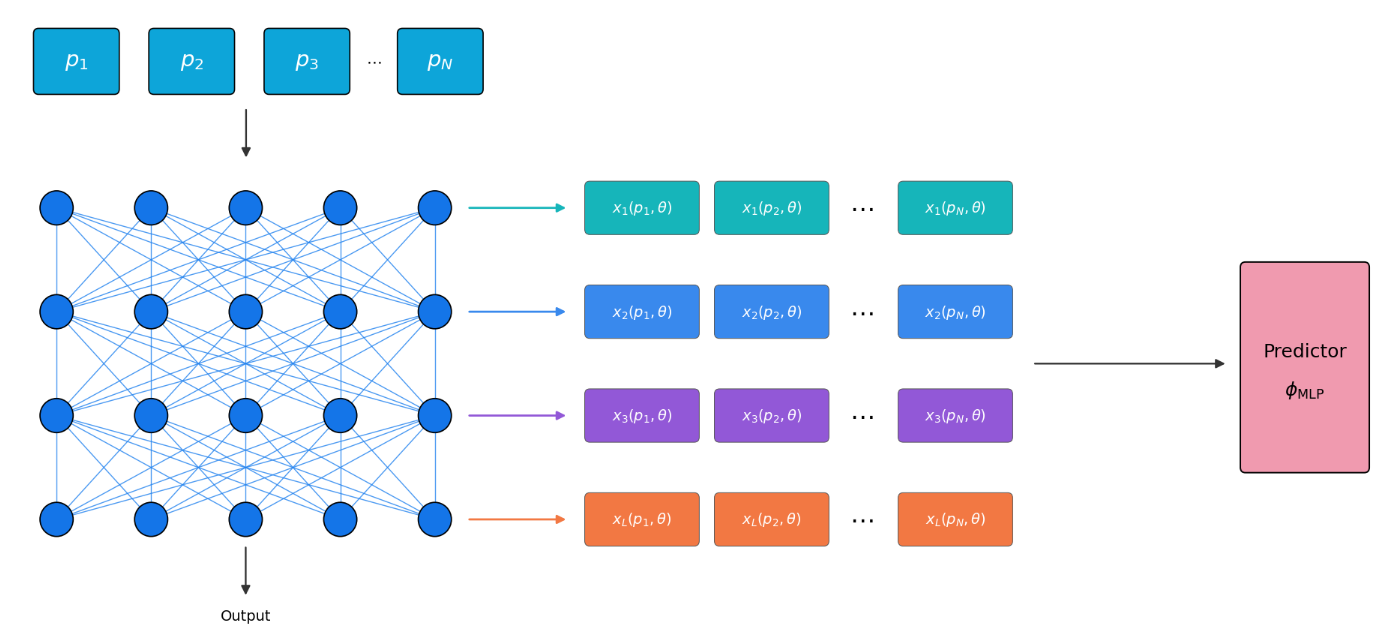}
    \caption{\textbf{Overview of HiddenProbe.}
    A set of learnable probes $p_1,\ldots,p_N$ is evaluated on a target neural network $f_\theta$, producing intermediate activations $x_i(p_j,\theta)$ and final outputs $x_L(p_j,\theta)$, which are processed by an invariant model and an MLP to predict a property of the target network.}
    \label{fig:model}
\end{figure}

\vspace{-10pt}
\subsection{Previous Work}

\paragraph{Weight-space learning.}
Weight-space learning studies models that process the parameters of other neural networks directly; we refer to \citet{haggaisurvey} for a recent survey. A central challenge in this setting is that different parameterizations can represent the same function, motivating architectures that explicitly respect weight-space symmetries. DWSNets~\citep{dws_paper} and Neural Functional Networks~\citep{zhou2023permutation} construct linear layers equivariant to hidden-neuron permutations, while Neural Functional Transformers~\citep{zhou2023neural} incorporate the same symmetry into attention-based architectures. More recent graph-based approaches, including Graph Metanetworks~\citep{lim2023graph} and Neural Graphs~\citep{kofinas}, represent networks as computational graphs, enabling weight-space processing across a broader range of architectures.
Beyond permutations, subsequent methods incorporate additional symmetries, such as ScaleGMN~\citep{kalogeropoulos2024scale} and Monomial-NFN~\citep{tran2024monomial}, which account for scaling and sign symmetries induced by particular activation functions.

\paragraph{Identification and learning from probes}
Prior work in weight-space learning has explored combining network weights with responses to input probes to predict network properties~\citep{rnnprobes,kofinas,Horwitz_2025_CVPR}. Complementary approaches rely solely on final outputs evaluated at learned probes~\citep{kahana}. Beyond weight-space learning, interpolation and identification of neural-network from input--output observations has also been studied in learning theory~\citep{vardi2021optimalmemorizationpowerrelu}. More recently, \citet{bhattamishra2026,kimkim2026} have investigated recovering the weights of frozen transformers from query outputs. Beyond final outputs, hidden activations provide additional information about a network's internal representations. Prior work has used these activations to study representation properties~\citep{alain2018understandingintermediatelayersusing,dynamo} and, more recently, to investigate model behavior through mechanistic interpretability~\citep{neelnanda}.
In natural language processing, probing classifiers similarly predict
linguistic properties from a model's hidden representations to study
the information they encode~\citep{belinkov2022probing}.
Within weight-space learning, \citet{heo2026linearprobesmissmultiview} similarly probe selected hidden layers to predict network properties.

To the best of our knowledge, prior work has not systematically studied property prediction from complete activation traces across all hidden layers, collected by evaluating a network on input probes. We address this gap by studying when neural functions can be identified from the activations induced by input queries, and how universal approximators can be designed to predict network properties from final outputs and hidden activation traces.

\subsection{Preliminaries}
Consider a parametric family
$f:\Xcal\times\paramspace\to\R^m$,
where $\Xcal\subseteq\R^d$ is the input domain and
$\paramspace\subseteq\R^q$ is the parameter domain.
Every fixed parameter $\param\in\paramspace$ defines a function
$f_\param(x)=f(x;\param)$.
We use $f(\cdot;\param)$ and $f_\param$ interchangeably.
We denote the collection of all such functions by
$\mathcal F_\paramspace$.

Our motivating application is the problem of learning neural functionals\footnotemark $\Psi: \mathcal{F}_\paramspace \to \RR^m$, and our focus is on analyzing methods which rely on probing. We will now define these methods, and focus on two variants: final-layer probing and hidden-layer probing.

\footnotetext{Our focus is on neural functionals which depend only on the function $f_\theta$ and not on the specific parameterization $\theta$. In the terminology of \cite{dayan2026on} these are called  function-space functionals}

\paragraph{Final-layer probing}
Final-layer probing represents a function $f_\theta$ by considering its output on a number of input points, called probes. Formally, given an $N$ tuple of points $\Pcal\coloneqq(p_1,\cdots, p_N)\in\Xcal^N$, the model $f_\theta$ is mapped to an $m\times N$ dimensional matrix
$$\final(\Pcal;\theta)\coloneqq(f_{\theta}(p_1),\cdots,f_{\theta}(p_N))\in\R^{m\times N}.$$

This matrix is processed to learn properties of the neural network, e.g., by applying an MLP $\phi_{\text{MLP}}\circ \final(\Pcal;\theta) $ to the output, to obtain a neural model for learning the functional $\Lambda(f_\theta) $. A representative example of this approach is the ProbeGen method of \cite{kahana}, which uses this approach precisely with a carefully designed method for learning the probes. 

\paragraph{Hidden-Layer Probing} \cite{heo2026linearprobesmissmultiview,Horwitz_2025_CVPR}, 
In applications, the parametric functions $f_\theta$ are typically constructed by a composition of simpler functions (e.g. layers), that is 
 $f_\theta=f_{L,\theta}\circ\ldots\circ f_{2,\theta}\circ f_{1,\theta} $. When this compositional structure is present, hidden layer probing relates to the practice of retaining the output of all intermediate computations on the probes, and not only the final output. Formally, for every input $p\in \X$ and parameter $\theta \in \Theta$ we define
$$x_0(p,\theta)=p, \, x_1(p,\theta)=f_{1,\theta}\left(x_0(p,\theta)\right), \ldots,  x_L(p,\theta)=f_{L,\theta}\left(x_{L-1}(p,\theta)\right).$$

The hidden-probe representation of $f_\theta$ is defined via the mapping
$$\all(\theta,\P):=\left(x_i(p_j,\theta), 0\leq i \leq L, \, 1\leq j \leq N  \right) $$


\subsection{Overview} In the remainder of the paper, we will prove that analytic neural networks of any kind can be identified by a number of final-layer probes proportional to the network's parameters (Section \ref{sec:general}), and will show, for both MLPs and transformers, that the required number of probes can be greatly reduced when allowing Hidden-Layer probing (Sections \ref{sec:mlp}-\ref{sec:mha}). We then describe \our, a method for learning using hidden probes (Section \ref{sec:hiddenprobe-master}), prove the universality of probe-based methods (Section \ref{sec:approx}), and finally demonstrate the effectiveness of \our on neural functional learning (Section \ref{sec:expt-main}).
\section{Identifying Analytic Neural Networks from Probe evaluations}\label{sec:general}
Our first result is that final-layer probe evaluations are enough to identify \emph{analytic} parametric functions, and the number of probes is proportional to the dimension of the parameter space. This theorem applies to all analytic parametric families including MLPs and Transformers, provided that the activations used in these models are analytic. Examples of such activations are tanh, sigmoid, SiLU and sine activations.



\begin{restatable}[Unique reconstruction from finite samples]{theorem}{thmfinalprobes}
\label{thm:finalprobes}
Let $f:\RR^d \times \RR^q \to \RR^m $ be an analytic function. Then, for almost any choice of probes $p_1,\ldots,p_{2q+1} \in \RR^d$,  we have that for all
$\param_1,\param_2 \in \paramspace$, 
$$f(p_j,\theta_1)=f(p_j,\theta_2), \forall j=1,\ldots,2q+1 \Longleftrightarrow\ f(\cdot,\param_1)\equiv f(\cdot,\param_2)\ \text{on }\RR^d$$
\end{restatable}
\begin{proof}[Proof Idea]
Intuitively we can interpret this result as follows: although the family of functions $\{f_\theta| \quad \theta\in \RR^q \} $ contains infinitely many functions, it has a finite  dimensional  parameterization in $\RR^q$, and so it is reasonable that a finite number of $\sim q$ evaluations can separate the realized functions. This is similar to the well-known result that a degree $n$ polynomial can be identified by $n+1$ samples. However, this rule of thumb is not always true. For example, for ReLU networks, no finite number of probes suffices to uniquely identify the function (see Subsection~\ref{subsec:relu}). The proof that such a result does hold for analytic functions is based on the Finite Witness Theorem by \cite{amir2023neural}.
\end{proof}

In Section \ref{sec:approx} we will show that this result implies universal approximation:  probe based methods with $2q+1$ probes can approximate all neural functionals. This result is substantial, since probe based methods typically use a small number of probes, while apriori it would seem that universal approximation would require an exhaustive sampling of the input domain. At the same time, the number of probes suggested by the theorem is still much larger than what is used in practice. To alleviate this disadvantage, we will next focus on hidden probes, where function identification is possible with a much lower probe cardinality.

\section{Probing MLPs}\label{sec:mlp}
In this section we specialize to the case where $f(x;\theta) $ is an MLP. The main message of this  section is that hidden probes substantially reduce the number of probes necessary to identify an MLP.


We begin our discussion by  setting our notation for the MLP model class. An MLP $f(p;\param) $ is determined by a choice of activation $\sigma$ and a width vector 
$\Db\coloneqq(d_0, d_1, \cdots,d_{L-1}, d_L)$. An MLP $f(p;\param) $ is then defined via
\begin{equation}
    x_0 \coloneqq p,
    \,
    x_i \coloneqq \sigma\left(W_i x_{i-1} + b_i\right)
   ,
    \, \forall i\in [L],
    \,
    \label{eq:mlp}
\end{equation}
Here, $W_i\in\R^{d_i\times d_{i-1}}$ and
$b_i\in\R^{d_i}$, for $i\in[L]$, are the weights and biases, which together
define the parameters
$\param\coloneqq[W_i,b_i]_{i\in[L]}$. The intermediate features $x_i$ depend on $p,\theta$ and so we denote them by $x_i(p,\theta) $. The final function is defined by the last feature $f(p,\theta)=x_L(p,\theta) $. We denote the total dimension of the parameter space by $q_\Db $.

\subsection{Hidden probes provide useful information:}\label{sub:hidden_MLP_an}
For MLPs with analytic activations, Theorem \ref{thm:finalprobes} implies that a finite number of probes, proportional to $q_\Db$, is sufficient. For hidden probes, the number of necessary probes depends only on the maximal width:

\begin{restatable}[Reconstruction from hidden layer activations]{theorem}{thmhiddenweaker}
    \label{thm:hidden-weaker}
    Let $\sigma:\RR\to \RR$ be an  analytic, injective, non-polynomial activation.  Let $f(p,\theta) $ denote an MLP with activation $\sigma$  and widths $\Db\coloneqq(d_0, d_1, \cdots,d_{L-1}, d_L)$. Assume that $N>\max_i d_i $. Then for almost every $\mP\in\X^{N}$, we have that for almost every  $\theta_1,\theta_2$,
         $$\all(\param_1,\mP)=\all(\param_2,\mP) \implies\,\param_1=\param_2,\,\implies f(\cdot, \theta_1)=f(\cdot, \theta_2).$$ 
 
\end{restatable}
In this result we add the assumption that the  activation is non-polynomial and injective (this is the case for e.g., tanh and sigmoid, but not for sine activation). Moreover, identifiability is not on the whole domain, but rather it holds for parameters $\theta_1,\theta_2$ outside an exceptional Lebesgue-null set $\mathcal E$. The proof in fact shows a stronger statement: the  set $\mathcal E$ is a zero set of an analytic function, and hence is of strictly lower dimension than the parameter space, and has  empty interior. In general, whenever we talk about "almost everywhere" in this paper, this stronger notion applies.

\begin{proof}[Proof Idea]
The idea of the proof is based on a straightforward algorithm described in detail in in Algorithm \ref{alg:weights} of Appendix~\ref{app:proofs:mlphidden}. Intuitively, the parameters $W_i,b_i$ of the $i$-th layer satisfy the linear equations
$$\sigma^{-1}\left(x_{i}(p_j,\theta)\right)=W_ix_{i-1}(p_j,\theta)+b_i, \quad  j=1,\ldots,N $$
The proof shows that for generic probes and parameters, these linear equations  have a unique solution.
\end{proof}

\paragraph{Tightness of results}
in Appendix~\ref{app:proofs:lower} we give examples of analytic MLPs that cannot be identified with $d_{\max}$ hidden probes (exactly matching the bound of Theorem~\ref{thm:hidden-weaker}) For final-layer probes, we construct, for every
$N\le 2q_D-2$, a nonempty open set of probe tuples that
fail to separate all realized functions. (Theorem~\ref{thm:finalprobes} gives a sufficient bound of $2q_\Db+1$ generic probes).

\subsection{Identifying ReLU Networks from Probe evaluations}
\label{subsec:relu}
\begin{wrapfigure}{r}{0.35\textwidth} 
     \vspace{-20pt}
    \centering
    \includegraphics[width=\linewidth]{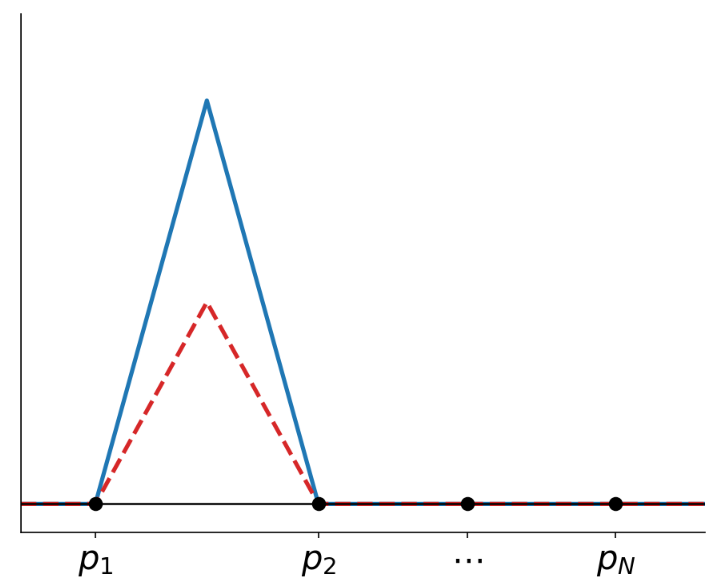} 
    \caption{For any finite set of probes, a nonzero ReLU `hat' function can be constructed whose support avoids all probes.}
    \label{fig:hat}
    \vspace{-20pt}
\end{wrapfigure}
Our results on identifiability of an MLP from a finite number of probes require analytic activations, and thus do not apply to the common ReLU activation. In fact, such a result is not possible for ReLU activations:

\begin{restatable}[Finite probes do not identify ReLU networks]{proposition}{proprelunon}
\label{prop:relu-nonidentifiability}
For any finite set of probes $p_1,\ldots,p_N\in[0,1]$, there exist two distinct ReLU networks, each with three hidden neurons, that agree on all the probes. Consequently, no finite set of probes can uniquely identify all ReLU networks.
\end{restatable}
\begin{proof}[Proof Idea]
For given probes $p_1,\ldots,p_N$, it is possible to construct two different hat functions which are zero on the probes: for example, the functions plotted in blue and orange in \autoref{fig:hat}. Full proof is in Appendix~\ref{app:proofs:relu}. 
\end{proof}

Since exact identification of ReLU networks from finitely many probes is impossible, we instead study approximate identification in $L^2(\mu)$. We show that, with probability at least $1-\delta$ over the sampled probes, any network that agrees with the target network on those probes is within $\epsilon$ of the target in $L^2(\mu)$. This is similar in spirit to results from PAC learning. The following theorem gives a semi-formal description of these results, which are described and proved in more detail in Theorem \ref{thm:relu-pac} of Appendix~\ref{app:proofs:relu}.
\begin{restatable}[semi-formal-PAC identification of $\relu$ networks]{theorem}{relupacinformal}
\label{thm:relu-pac-informal}
Consider the family of $\relu$ MLPs with widths $\Db\coloneqq(d_0, d_1, \cdots,d_{L-1}, d_L)$ and $q$ parameters. Let $\mu$ be a compactly supported probability distribution on $\RR^{d_0}$, and assume that the parameter set is compact.  Fix $\epsilon,\delta\in (0,1)$ and a parameter vector $\param_0$. Then there exist constants $C_1(\epsilon,\delta),\,C_2(\epsilon,\delta)>0$ such that 
\begin{itemize}
\item For $N_1=C_1\cdot L q \ln(q)$ final layer samples, with probability of at least $1-\delta$ on the samples, any $\param$ which agrees with $\param_0$ on the samples will satisfy $\|f(\cdot;\param)-f(\cdot;\param_0)\|_{L^2(\mu)}\le\epsilon.$
\item When hidden layer samples are used, only $N_2=C_2\cdot q\ln(q) $ samples are sufficient.
\end{itemize}
\end{restatable}

As  in the analytic activation case, this result suggests that hidden activations can reduce the number of probes necessary to (almost) identify a ReLU network, especially for deeper ReLU nets. An experimental verification of this is given in the Appendix~\ref{app:expt:synthetic}.


\section{Identifying Multi-Head Attention from Probes}\label{sec:mha}
Multi-head attention (MHA) forms the backbone of modern transformer architectures. In this section we discuss the identification of multiple-layer multihead attention (MHA) blocks from probes. Additional components typically used in transformers, like FFNs and LayerNorm, are omitted for mathematical tractability.

We begin with some basic definitions (details in Appendix~\ref{app:proofs:mha}). We consider $d$ dimensional inputs and $h$ attention heads. An MHA layer
has parameters $
\param\coloneqq\{(\Q_a,\K_a,\V_a,\O_a)\}_{a=1}^h,
$
For an input probe $\probeseq\in\R^{n\times d}$ whose rows are the
token vectors $\probe_1^\top,\ldots,\probe_n^\top$, the layer computes
\begin{equation}
\textstyle
\label{eq:softmax}
\begin{aligned}
F_n(\probeseq;\param)=\sum_{a=1}^h
    \operatorname{softmax}(\probeseq\M_a\probeseq^\top)\,\probeseq\W_a,\,
\M_a\coloneqq\frac{\Q_a\K_a^\top}{\sqrt{k}},\,\W_a\coloneqq\V_a\O_a.
\end{aligned}
\end{equation}

The softmax acts rowwise:$
\left[\operatorname{softmax}(\probeseq\M_a\probeseq^\top)\right]_{ij}=
\frac{\exp(\probe_i^\top\M_a \probe_j)}{\sum_{s=1}^n\exp(\probe_i^\top\M_a \probe_s)}.$ A probe specifies $\probeseq$ and returns the complete matrix
$F_n(\probeseq;\param)\in\R^{n\times d}$. We denote by $\boldsymbol{\param}=(\param_1,\ldots,\param_L)$ the parameters of an $L$-layer MHA stack and let $\Theta=\prod_{\ell=1}^L\Theta_\ell$ denote the parameter space of the $L$-layer MHA architecture, with $\Theta_\ell$ being the parameter space of the $\ell$-th layer, let $q\coloneqq\dim(\Theta)$.Writing $F_n^{(\ell)}$ for the $\ell$-th layer map, we define
\begin{equation}
\label{eq:mha-stack}
\begin{aligned}
x_{0,n}(\probeseq;\boldsymbol{\param})=\probeseq,\,
x_{\ell,n}(\probeseq;\boldsymbol{\param})=F_n^{(\ell)}
    \left(x_{\ell-1,n}(\probeseq;\boldsymbol{\param});\param_\ell\right),
    \, \ell=1,\ldots,L.
\end{aligned}
\end{equation}
Thus $x_{L,n}$ is the stack's final output. For a collection of probes
$\mathbf{\Pcal}=(\probeseq^{(i)})_{i=1}^N$, where
$\probeseq^{(i)}\in\R^{n_i\times d}$. As usual, we define the output-only and
hidden-access probe maps for transformers by $\final(\boldsymbol{\param}, \Pcal) $ and $\all(\boldsymbol{\param}, \Pcal) $, respectively.

Since the softmax attention is analytic, we can apply a variant of Theorem~\ref{thm:finalprobes}, by using the Finite Witness Theorem for the effective manifold dimension of $\Theta$ to give function-level identification from probes:

\begin{restatable}[Generic probes]{corollary}{corattgeneric}
\label{prop:att-generic}
Fix a token length $n\ge1$. For almost every tuple of
$1+2q$ probes in $\R^{n\times d}$ and all
$\param,\param'\in\Theta$, $\final(\param)=\final(\param'),\,\iff\,x_{L,n}(\cdot;\param)\equiv x_{L,n}(\cdot;\param')$.
\end{restatable}

Note that all of these parameter counts are $O(d^2)$ per layer for the standard multi-head attention blocks. But softmax attention has more structure than analyticity, and we can query probes of varying token lengths to an MHA block, which lets us identify the function using only $O(dh)$ probes:

\begin{restatable}[Structured probing for softmax attention]{theorem}{thmattfixedinformal}
\label{thm:att-fixed-informal}
Let $h_{\max}\coloneqq\max_\ell h_\ell$ and $N\coloneqq d(2h_{\max}+1)+1$. There exists an open set of probe tuples $\mP\in\X^N$ of token length at most $\max\{3,1+d/h_{\max}\}$, such that for almost every target $\param_0\in\Theta$ we have: 
$$\all(\param;\mP)=\all(\param_0;\mP)\,\implies\,x_{\ell,n}(\cdot;\param)\equiv x_{\ell,n}(\cdot;\param_0),\, \forall\,\param\in\Theta, \forall \, \ell\in[L],\forall \, n\in \N,$$ 
\end{restatable}

\begin{proof}[Proof Idea]
Theorem~\ref{thm:att-fixed-informal} gives a specific way of choosing the probes for identifying Multilayer MHA blocks with hidden access, which we call \mhaalgo. We elaborate on the single layer single head case of \mhaalgo (see also Algorithm~\ref{alg:att-sketch} in the appendix). In this case, the goal is to reconstruct two matrices: the matrices $\M_0$ and $\W_0$ from \eqref{eq:softmax}. The key trick is to do this in first stages: we first use  single-token input: For such input the output of the softmax is always the scalar $1\in \RR^{1\times 1}$, and so $\M_0$ is eliminated, and one-token probes provide linear equations that can be used to determine $\W_0$. Once $\W$ is known, we query a single probe consisting of affinely independent tokens, which uniquely recovers the attention weights for a full rank target value matrix. We then recover $\M_0$ from log-ratios by utilizing the softmax structure. Algorithm~\ref{alg:att-sketch} implements these steps using
$d+1$ probes. The extension of these results to multiple heads and layers is discussed in Appendix~\ref{app:proofs:mha}.
\end{proof}

\section{\our: Learning using Hidden Neural Activations}
\label{sec:hiddenprobe-master}
Our analysis shows that using hidden probes can greatly enhance the identifiability of neural networks from probes. Accordingly, we design a method for learning from hidden probes which we name \our. We have two different instantiations of this model, one for MLPs and one for transformers. The exact details are discussed in Appendix \ref{app:hiddenprobe-master}. However, the basic idea is simple and identical. We use learned probes, and utilize all hidden layers. The probes of each hidden layer are processed a permutation-invariant aggregator, and then the results are aggregated using a global MLP to predict the target property. 

We design \our for MLPs so that it respects the permutation invariance of the MLP weight space. Namely, one can apply an independent permutation to the neurons of each hidden layer, and obtain a functionally equivalent MLP with permuted weight matrices. We attain permutation invariance by using permutation invariant Set Transformer~\cite{settransformer}.

In general, an interesting advantage of (hidden) probe methods over weight space methods is that weight space symmetries become much simpler. Indeed, while in the hidden layer representation the permutations to each hidden layer are independent, in the weight space representation these symmetries are entangled, and designing permutation invariant models for this representation can be cumbersome \cite{dws_paper,sverdlov}. For transformers, the weight space approach needs to deal with $GL(n)$ symmetries \cite{tran2025b} which probe based methods avoid since they depend  on the functional behavior of MHA and not on the exact parameterization. 

\vspace{-10pt}

\section{Universality of probe-based learning}
\label{sec:approx}
\vspace{-10pt}
In this section, we use our identification results from previous sections to prove universality of probe based models with a fixed number of probes. These results are based on the standard observation that separation implies universality \cite{chen2019equivalence,dym2025low}. For our theorem, we will consider analytic functions $f:\RR^d \times \RR^q \to \RR^m $ as in Theorem \ref{thm:finalprobes}, but we will assume that we are only interested in a compact input domain $\X \subseteq \RR^d $ with non-empty interior, and a compact parameter domain $\paramspace \subseteq \RR^q $. Accordingly, we are interested in the function space 
$ 
    \mathcal{F}_\paramspace(\mathcal{X},\RR^m)
    :=
    \left\{
        f_{\param}:\X \to \RR^m \right\}
$
which is a subset of the set of all continuous functions from  $\X $ to $\RR^m$,  endowed with the infinity norm $\|f_\theta\|_{\infty}=\max_{x\in \X}\|f(x,\param)\|_2 $. A  functional $\Lambda:\mathcal{F}_\paramspace(\mathcal{X},\RR^m)\to \RR $ is continuous, if it is continuous with respect to the distance induced by this norm.

Our first result uses Theorem \ref{thm:finalprobes} to show that for general analytic models with $q$ parameters, $2q+1$ final-layer probes are sufficient for universality:
\begin{restatable}[Universality of probe-based learning]{theorem}{thmfuncinformal}
\label{thm:functional-informal}
With the assumptions listed above, let $\Lambda:\mathcal{F}_\paramspace(\mathcal{X},\RR^m)\to \RR^k $ be a continuous functional. Then there exists $p_1,\ldots,p_{2q+1}\in \X $, such that  $\forall\,\epsilon>0$, there exists an MLP $\phi_{\mathrm{MLP}} $ such that
$
     \sup_{\theta\in\paramspace}||\Lambda(f_\theta)-\phi_{\mathrm{MLP}}\left(f_\theta(p_1),\ldots,f_\theta(p_{2q+1}) \right)||<\epsilon 
$
\end{restatable}
When using hidden probes, we can obtain universality using our model $\our$ with a significantly smaller number of probes, due to our identifiability results:



\begin{restatable}[Universality of \our]
{theorem}{thmhiddenfuncinformal}
\label{thm:hidden-functional-informal}
With our previous notation, assume that  
\begin{enumerate}
\item  $f(x;\theta)$ is an MLP architecture with analytic, injective, non-polynomial activation, and   maximal width $\dmax$, and denote $N=\dmax+1 $. Or
\item $f(x;\theta)$ is a  pure softmax MHA stack with $d$ dimensional inputs and at most $H$
    heads per layer. Set 
    $N=d(2H+1)+1$.
    \end{enumerate}
Then, if $K\subseteq\paramspace$ is a compact subspace of the parameters that does not intersect an appropriate Lebesgue-null set $E$, there exist $N$ probes, and a hidden probe model $\mathcal H_{\our}$, such that 
$
    \sup_{\param\in K}
\left\|
\Lambda(f_\param)
-\mathcal \mathcal H_{\our}(\param)
\right\|_2<\epsilon.
$
\end{restatable}
We restate formal versions of the above theorems and prove them in  Appendix~\ref{app:proofs:approx}, together with guarantees on operator learning from probes.

\vspace{-10pt}
\section{Experiments}
\label{sec:expt-main}
\vspace{-8pt}
We evaluate \our for weight space lerning tasks and compare with other prominent baselines. To understand the influence of multi-layer probes, we put a special emphasis on comparison with ProbeGen, which is  a successful method relying only on final layer probes. We will show that  \our typically outperforms ProbeGen and other  weight space and probe-based models on weight space learning baselines.   

\paragraph{INR classification}
 In this setting, each input network is an MLP trained as an implicit neural representation (INR) of a single image, and the goal is to predict the class of the represented image directly from the network. We evaluate on the MNIST, Fashion-MNIST, CIFAR-10 (singke and augmented views). The INR networks across datasets contain two hidden layers.

In Table~\ref{tab:inr_classification}, we compare \our against the baselines: ScaleGMN~\citep{kalogeropoulos2024scale}, NG-T and Neural Graphs~\citep{kofinas}, NFT~\citep{zhou2023neural}, MAGEP-NFN~\citep{vo2025magep}, and ProbeGen~\citep{kahana}. A complete comparison with all available methods is provided in Table~\ref{tab:inr_classification_full} in the Appendix.

\begin{table}[h]
\centering
\resizebox{\textwidth}{!}{
\begin{tabular}{lcccc}
\toprule
\textbf{Method}
& \textbf{MNIST}
& \textbf{FMNIST}
& \textbf{CIFAR-10}
& \textbf{CIFAR-10 Aug} \\
\midrule

ScaleGMN
& $0.966 \pm 0.002$
& $0.808 \pm 0.001$
& $0.388 \pm 0.001$
& $0.570 \pm 0.005$ \\

Neural Graphs (128 probes)
& $0.976 \pm 0.001$
& $0.745 \pm 0.008$
& --
& -- \\

NG-T
& $0.924 \pm 0.003$
& $0.727 \pm 0.006$
& --
& -- \\

MAGEP-NFN
& --
& --
& $0.3718 \pm 0.003$
& -- \\

NFT
& --
& --
& --
& \best{$0.634 \pm 0.000$} \\

ProbeGen
& \second{$0.984 \pm 0.001$}
& \second{$0.877 \pm 0.003$}
& \second{$0.570 \pm 0.006$}
& $0.563 \pm 0.003$ \\
HiddenProbe
& \best{$0.986\pm 0.001$}
& \best{$0.886 \pm 0.005$}
& 
\best{$0.580 \pm 0.002$}
&  \second{$0.606 \pm 0.002$} \\

\bottomrule
\end{tabular}}
\caption{Classification accuracy ($\uparrow$) on INR weight-space benchmarks. Uncertainties denote the standard error over five seeds. Baseline numbers are copied from the respective papers}
\label{tab:inr_classification}
\vspace{-10pt}
\end{table}

As observed, \textbf{(1)} \our achieves the best performance on MNIST, Fashion-MNIST, and CIFAR-10, and the second-best performance on augmented CIFAR-10. \textbf{(2)} It outperforms ProbeGen everywhere, highlighting the importance of having hidden layer information. \textbf{(3)} Probegen comes second in $3/4$ datasets, proving the efficacy of probe-based learning in general

\paragraph{Test accuracy prediction}
We next evaluate \our on the model accuracy prediction task. In this setting, the input is a CNN trained for image classification, and the goal is to predict its test accuracy directly from the network. We evaluate on the MNIST, Fashion-MNIST, SVHN, CIFAR-10-GS, and CIFAR-10-WP benchmarks, using Kendall's correltion($\tau$) as the evaluation metric. In Table~\ref{tab:regression}, we compare against a representative set of strong baselines: NFN~\citep{zhou2023permutation}, ScaleGMN~\citep{kalogeropoulos2024scale}, NG-T and Neural Graphs~\citep{kofinas}, DNG-Encoder~\citep{wu2026dynamic}, and ProbeGen~\citep{kahana}. A complete comparison with all available methods is provided in Table~\ref{tab:regression_full} in the Appendix.

The networks in the MNIST, Fashion-MNIST, SVHN, and CIFAR-10-GS benchmarks contain three hidden layers, whereas CIFAR-10-WP contains heterogeneous architectures with up to four hidden layers. \our achieves the best performance on all five benchmarks and outperforms ProbeGen in every case.

\begin{table}[h]
\centering
\resizebox{\textwidth}{!}{
\begin{tabular}{lccccc}
\toprule
\textbf{Method}
& \textbf{MNIST}
& \textbf{FMNIST}
& \textbf{SVHN}
& \textbf{CIFAR-WP}
& \textbf{CIFAR-GS} \\
\midrule

NFN
& $0.942 \pm 0.001$
& $0.935 \pm 0.000$
& \second{$0.931 \pm 0.005$}
& --
& $0.934 \pm 0.001$ \\

ScaleGMN
& --
& --
& --
& --
& $0.941 \pm 0.006$ \\

NG-T
& --
& --
& $0.872 \pm 0.001$
& $0.817 \pm 0.007$
& $0.935 \pm 0.000$ \\

DNG-Encoder
& --
& --
& $0.867 \pm 0.002$
& $0.874 \pm 0.002$
& $0.936 \pm 0.000$ \\

Neural Graphs
& --
& --
& --
& $0.885 \pm 0.005$
& $0.938 \pm 0.000$ \\

\midrule

ProbeGen
& \second{$0.954 \pm 0.001$}
& \second{$0.947 \pm 0.003$}
& $0.873 \pm 0.008$
& \second{$0.932 \pm 0.006$}
& \second{$0.957 \pm 0.001$} \\
HiddenProbe
& \best{$ 0.966 \pm 0.004
$}
& \best{$0.957 \pm 0.002
$}
& \best{$ 0.957 \pm 0.005
$}
& \best{$0.940 \pm 0.006$}
& \best{$0.965 \pm 0.000$} \\
\bottomrule
\end{tabular}}
\caption{Test-accuracy prediction performance measured by Kendall's $\tau$ ($\uparrow$). Uncertainties denote standard error over five seeds. Baseline numbers are copid from the respective papers.}
\label{tab:regression}
\vspace{-10pt}
\end{table}

As observed: \textbf{(1)} \our is the best across all datasets, and probegen comes second in $4/5$ datasets. \textbf{(2)} In Cifar-WP where the model architecture varies, probe-based models significantly outperform weight space models since they are more robust to architectural changes.

\paragraph{Transformer Accuracy Prediction.}
We evaluate \ourmha on test-accuracy prediction for Transformers trained on MNIST and AGNews. Performance for the fuill dataset is shown in Table~\ref{tab:transformers}.Following the evaluation protocol of~\citet{tran2026quasi}, we additionally report results across different accuracy thresholds in Appendix~\ref{app:transformer_thresholds} in tables~\ref{tab:mnist_transformers},\ref{tab:agnews_transformers}.  
\begin{table}[h]
\centering
\small
\begin{tabular}{lcc}
\toprule
\textbf{Method}
& \textbf{MNIST-Transformers}
& \textbf{AGNews-Transformers} \\
\midrule

MLP
& $0.866\pm0.002$
& $0.879\pm0.006$ \\

STATNN \citep{unterthiner2020predicting}
& $0.881\pm0.001$
& $0.841\pm0.002$ \\

XGBoost \citep{chen2016xgboost}
& $0.860\pm0.002$
& $0.859\pm0.001$ \\

LightGBM \citep{ke2017lightgbm}
& $0.858\pm0.002$
& $0.835\pm0.001$ \\

Random Forest \citep{breiman2001random}
& $0.772\pm0.002$
& $0.774\pm0.003$ \\

\midrule

Transformer-NFN \citep{tran2025b}
& $0.905\pm0.002$
& $0.910\pm0.001$ \\

Transformer-NFN large \citep{tran2026quasi}
& $0.907\pm0.001$
& $0.913\pm0.001$ \\

Transformer-NFN Quasi \citep{tran2026quasi}
& \second{$0.911\pm0.001$}
& \second{$0.914\pm0.001$} \\

\midrule

ProbeGen
& $0.898\pm0.001$
& $0.885\pm0.001$ \\

HiddenProbe
& \best{$0.920\pm0.003$}
& \best{$0.916\pm0.002$} \\

\bottomrule
\end{tabular}

\caption{
Test accuracy prediction on the MNIST-Transformers and
AGNews-Transformers benchmarks, evaluated on the full datasets. Performance is measured by Kendall's
$\tau$, reported as mean $\pm$ standard error over five seeds.
The best and second-best results are highlighted.
}
\label{tab:transformers}
\end{table}
As observed, \textbf{(1)}\ourmha achieves state-of-the-art performance on both benchmarks, beating Transformer-NFN models with lesser parameter count \textbf{(2)} Even output-only ProbeGen is competitive with the benchmarks.

\paragraph{Varying performance with the number of probes}
In the first part of the experiments, we showed that using the intermediate layers adds significant empirical improvement. Now we want to see how the performance of the probe-based models, namely ProbeGen and \our, changes as we vary the number of probes. We compare \our with ProbeGen using $N\in\{16,32,64,128,256\}$ probes. We take seed 0 results for all plots. Results are in \autoref{fig:number_of_probes}. 

\begin{figure}[htbp]
    \centering
    \includegraphics[width=\textwidth]{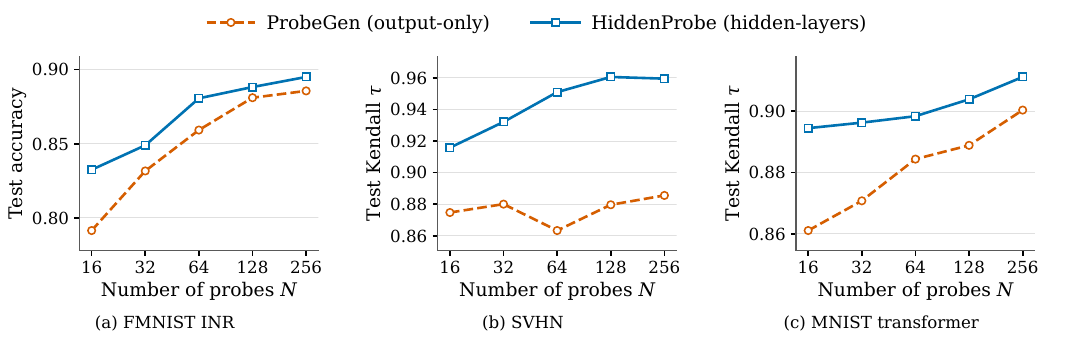}
    \caption{Performance of HiddenProbe and ProbeGen as the number of probes $N$ varies. FMNIST INR classification is evaluated by test accuracy, while regression tasks are evaluated by test Kendall's $\tau$. Higher is better in all panels.}
    \label{fig:number_of_probes}
    \vspace{-15pt}
\end{figure}

We make the following observations: \textbf{(1)} \our  generally improves as the probe count increases. Such a trend is also observed for ProbeGen but to a weaker extent \textbf{(2)} \our is comparable to, or better than, ProbeGen at a smaller (~$0.5\times$ or less) probe count. This supports the efficiency of including hidden layer information in probe-based learning.

\paragraph{Additional Information}
In all our experiments, \our uses the same number of parameters as ProbeGen (see \ref{num_param} for a complete parameter count comparison). A runtime comparison of \our, ProbeGen, and other DWS models is provided in Section~\ref{sec:times}. Synthetic experiments evaluating model reconstruction from finite probes are presented in Appendix~\ref{app:expt:synthetic}.

\vspace{-5pt}
\section{Conclusion}
\vspace{-5pt}
In this work, we have done a rigorous theoretical analysis of function-level identification of neural networks from probes. We have shown that one can use probe-based methods to identify a neural network from its outputs and also design universal approximators to learn properties of the neural net from probes. Across architectures like MLPs and Multi head attention blocks, we have shown that this identification can be done with much less number of probes if hidden layer information is incorporated, and we have accordingly proposed \our, which learns properties of neural nets from full hidden traces.  Our experiments demonstrate consistent improvements of \our over ProbeGen and strong weight-space baselines across diverse tasks, achieving state-of-the-art performance in several. We also demonstrated that incorporating hidden-layer information enables our models to match or surpass ProbeGen using substantially fewer probes, often requiring only half as many. Together, these results establish hidden-layer probing as an effective approach to learning neural functionals across different architectures and tasks. An interesting direction of future work could be designing such methods for larger architectures like language models or diffusion kernels.
\section{Acknowledgment}
N.D. and Y.S. are funded by the Israeli Science Foundation grant no.272/23 and DFG-ISF grant no. 3597/26. H.M. is supported by the Israel Science Foundation through
a personal grant (ISF 264/23) and an equipment grant (ISF
532/23), and by the Career Advancement Chairs in Artificial
Intelligence – Schmidt Futures. 

\subsection*{AI use statement}

In this work, we used ChatGPT to assist in refining the exposition of
mathematical proofs and to check mathematical arguments. We also used ChatGPT
for language editing and improving the readability of the manuscript.
All AI-assisted mathematical arguments, claims, and edits were manually
reviewed and verified by the authors. We have also used ChatGPT and Claude for AI-assisted coding but all the data and codes were manually checked and verified.
Generative AI tools were not used to
generate experimental results, numerical results, or data reported in the
paper. We take responsibility for the final content of this work, including any text, claims, and mathematical exposition and code produced with the aid of generative AI.
\newpage
\bibliography{iclr/iclr2027_conference}
\newpage
\appendix

\section{\our: Probe Based learning using Hidden Neural Activations}
\label{app:hiddenprobe-master}
In this appendix we describe in full detail \our, our probe-based architecture for learning from hidden-layer activations. 
\subsection{\our: A probe-based model for learning on weight spaces}
\label{app:subsec:hiddenprobe}
We begin by discussing the permutation symmetries of hidden activations induced by weight-space symmetries, which motivate the design of \our.
\paragraph{Symmetries in weight spaces and hidden Layers}Consider a depth-$L$ MLP. Let $P_0=P_L=I$, and let
$P_1,\ldots,P_{L-1}$ be arbitrary permutation matrices. The transformation
\[\textstyle
W_i \mapsto P_{i+1}W_iP_i^\top,
\,
b_i \mapsto P_{i+1}b_i,
\,
0\le i<L,
\]
permutes the neurons within the hidden layers while leaving the realized function unchanged. As shown in prior work~\cite{dws_paper}, architectures that explicitly respect these symmetries can substantially outperform models that ignore them. We therefore design \our to be invariant to hidden-neuron permutations.

Under the above transformation, the representation at hidden layer $i$ transforms as $x_i(p;\param)\mapsto P_i x_i(p;\param),\, i\in[L-1]$
Thus, each hidden layer is equivariant to its corresponding permutation $P_i$.
In contrast, the input and final output are unchanged, since $P_0=P_L=I$.

We now describe how this symmetry acts on hidden probe responses. Consider a tuple of probes $\mP=(\probe_1,\ldots,\probe_N).$ For hidden layer $i$, define the response vector of neuron $k$ across all probes as
$v^{i,k}
\coloneqq
\left(
x_{i,k}(\probe_1;\param),
\ldots,
x_{i,k}(\probe_N;\param)
\right)
\in \R^N.$
Since $x_i$ is equivariant to $P_i$, applying a hidden-neuron permutation simply permutes the collection $\left\{v^{i,k}:k\in[d_i]\right\}.$
Therefore, each hidden layer is naturally represented as an unordered set of
$d_i$ vectors in $\R^N$.

\paragraph{Overview of \our}Motivated by this observation, \our processes each hidden layer using a permutation-invariant set encoder. Specifically, we encode every hidden layer $i$ as $
r_\learningparam^i\coloneqq T_\learningparam^i
\left(\left\{v^{i,k}:k\in[d_i]\right\}\right),$
where $T_\learningparam^i$ is a layer-wise set encoder. In our implementation, $T_\learningparam^i$ is an invariant Set Transformer~\cite{settransformer}, making $r_\learningparam^i$ invariant to permutations of the neurons in layer $i$.

The resulting hidden-layer representations are concatenated in layer order,
together with the final-layer probe outputs:
$$\textstyle
\Phi_{\mP,\learningparam}^{\mathrm{hidden}}(\param)=
\left[
r_\learningparam^1;
\ldots;
r_\learningparam^{L-1};
f(\probe_1;\param);
\ldots;
f(\probe_N;\param)
\right].
$$

A final predictor $\rho_\learningparam$ maps this representation to the target: $\widehat y(\param)=\rho_\learningparam
(\Phi_{\mP,\learningparam}^{\mathrm{hidden}}(\param)).$ The probes $\probe_1,\ldots,\probe_N$, the layer-wise encoders
$T_\learningparam^i$ (allowed to be different across layers but can have shared parameters to reduce overall parameter count), and the predictor $\rho_\learningparam$ are trained jointly.
The learned probes are shared across all target networks.

\subsection{\our for transformers}
\label{app:subsec:transformer-hiddenprobe}
\paragraph{Learning a mixed-length probe bank.}
Here we introduce \ourmha, a probe-based model for learning on softmax-attention based transformers. Our reconstruction arguments from Theorem~\ref{thm:att-fixed-informal} and Algorithm~\ref{alg:att-sketch} use different sequence lengths: singleton probes recover the value matrix, while multi-token probes recover the attention weights. This suggests that probes of different lengths can provide complementary information about a target.
We retain this design principle, but learn a shared set of probes of varying sequence length rather than imposing the exact configurations.

\paragraph{Design of \ourmha predictor}
For each probe $P_j\in \mP$, we record the outputs of the transformer
blocks, the final normalized token states, and the decoder output:
\[
\mathcal T_{\param}
=\left(
X_1(P_j;\param),\cdots,
X_L(P_j;\param),\,X_{\mathrm{norm}}(P_j;\param),\,
f(P_j;\param)
\right)_{P_j\in \mP}.
\]
For each observed layer and channel, we concatenate its
responses across all probes and token positions, using a fixed
order. A shared MLP encodes each resulting channel trajectory
into one feature vector.
These channel features and final logit features are passed to a transformer followed by an MLP to give the output: $\widehat y_{\param}=R_{\learningparam}(\mathcal T_{\param})$.

\paragraph{Incorporating weight space symmetries}
Head permutations and the transformations
$Q_a\mapsto Q_aR_a$, $K_a\mapsto K_aR_a^{-\top}$ and
$V_a\mapsto V_aS_a$, $O_a\mapsto S_a^{-1}O_a$,
for invertible $R_a,S_a$, leave the attention output
unchanged. Consequently, they leave every recorded block
response and classifier output unchanged, so \ourmha
is invariant to these symmetries by design. Handling such symmetries remains a key challenge in works like \cite{tran2025b}, which we get for free.

\section{Proof of Technical results}
\label{app:proofs}

\thmfinalprobes*
\begin{proof}
The proof is based on the Finite Witness Theorem (Theorem A.2 in \cite{amir2023neural}), which states that, for an analytic $G:\RR^a\times \RR^b\to \RR $, and Lebesgue almost every $w_1,\ldots,w_{a+1}\in \RR^b $,
 $$\{y \in \RR^a| \, G(y;w_i)=0, \forall i=1,\ldots, a+1 \}=\{y \in \RR^a| \, G(y;w)=0, \forall w\in \RR^b \} $$
To prove our theorem, we define the real analytic function $G(\param_1,\param_2;p):=\|f(p;\param_1)-f(p;\param_2)\|_2^2$. Since the pair $(\theta_1,\theta_2) $ resides in $\RR^{2q}$, the Finite Witness Theorem implies that for $2q+1$ random samples $p_1,\ldots,p_{2q+1} $, we will have that for all $\theta_1,\theta_2$, $f(p_i;\param_1)-f(p_i;\param_2)=0 $ for all $i=1,\ldots,2q+1$, if and only if $f(p;\param_1)-f(p;\param_2)=0 $ for all $p\in \RR^d$.
\end{proof}

\subsection{Identification of Analytic MLPs from Hidden Probes}
\label{app:proofs:mlphidden}
We borrow notations from Section~\ref{sub:hidden_MLP_an} and Algorithm~\ref{alg:weights}, but occasionally them slightly. Let $L\in\NN$, $L\ge1$, and let $d_0,\ldots,d_L$ be positive
integers. Let $\X\subseteq\R^{d_0}$ be nonempty, open, and connected.
For each $i\in\{1,\ldots,L\}$, let
$\Theta_i\subseteq\R^{p_i}$ be nonempty, open, and connected,
and let $f_i:\R^{d_{i-1}}\times\Theta_i\to\R^{d_i}$
be jointly real analytic. Define the following notations:
\begin{align*}
\Theta\coloneqq\prod_{i=1}^L\Theta_i,\,\param\coloneqq(\theta_1,\ldots,\theta_L)\in\Theta,\,q\coloneqq\sum_{i=1}^L q_i,\,q_{\max}&\coloneqq\max_{1\le i\le L}q_i.
\end{align*}
The full model is the composition of these layer maps:
\begin{equation}
\label{eq:sequential}
x_0(u;\param)\coloneqq u,\,x_i(u;\param)\coloneqq
f_i\left(x_{i-1}(u;\param),\theta_i\right),\,\forall i\in [L],\,
f(u;\param)\coloneqq x_L(u;\param).
\end{equation}
Given a tuple of probes $\P=(p_1,\ldots,p_N)\in\X^N$, define
\begin{equation}
\label{eq:mlpdefs}
\begin{split}
&\mathbf X_i^\param(\P)\coloneqq
\begin{pmatrix}
x_i(p_1;\param)\cdots x_i(p_N;\param)
\end{pmatrix}\in\R^{d_i\times N},
\, i=0,\ldots,L.\\
&\overline{\mathbf X}_i^\param=\begin{pmatrix}\mathbf X_i^\param\\ \mathbf 1_N^\top\end{pmatrix}
\end{split}
\end{equation}
Thus, full intermediate observation records the sampled
states at every layer, including the final output.

\paragraph{Identification of Arbitrary Analytic Nets from Hidden Layers:}
Theorem~\ref{thm:finalprobes} shows that $2q+1$ final-layer probes suffice to distinguish any two functionally distinct models in an analytic family. Here, we show that for analytic sequential models of the form in \eqref{eq:sequential}, access to intermediate-layer observations reduces the sufficient probe count to $q+q_{\max}+1$.

\begin{restatable}[Identification from intermediate observations]
{theorem}{thmgeneralhidden}
\label{thm:general-hidden}
Under the preceding assumptions, suppose
$N\ge q+q_{\max}+1$.
Then, for Lebesgue-almost every $\P\in\X^N$,
simultaneously for every $\param,\param'\in\Theta$, one has
\begin{align*}
\mathbf X_i^\param(\P)=\mathbf X_i^{\param'}(\P)
\,\forall\,i\in\{0,\cdots,L\}\,\implies\,f(u;\param)=f(u;\param'),
\,\forall\,u\in\X
\end{align*}
\end{restatable}

\begin{proof}
For $i\in[L]$, we use the following notations:
\begin{align*}
q_{<i}\coloneqq\sum_{k=1}^{i-1}q_k,\,
\Theta_{<i}\coloneqq\prod_{k=1}^{i-1}\Theta_k,\,
Q_i\coloneqq q_{<i}+2q_i.
\end{align*}
For $i=1$, set $q_{<1}=0$ and $\Theta_{<1}=\R^0$.
The state $x_{i-1}(u;\param)$ depends only on
$(\theta_1,\cdots,\theta_{i-1})$.
For $\eta=(\eta_1,\cdots,\eta_{i-1})\in\Theta_{<i}$,
we therefore write $x_{i-1}(u;\eta)$ for the state
generated by these blocks and for $i=1$, this means
$x_0(u;\eta)=u$.

We define $\Omega_i\coloneqq\Theta_{<i}\times\Theta_i\times\Theta_i,$
and $H_i(u;\eta,\alpha,\beta)
\coloneqq
\left\|
f_i\left(x_{i-1}(u;\eta),\alpha\right)
-
f_i\left(x_{i-1}(u;\eta),\beta\right)
\right\|_2^2.$
By composition, $H_i$ is jointly real analytic on
$\X\times\Omega_i$, and $\Omega_i$ is an open connected
subset of $\R^{Q_i}$. Moreover, we have
\begin{align*}
Q_i+1=q_{<i}+2q_i+1
\le q+q_i+1
\le q+q_{\max}+1
\le N.
\end{align*}

The Finite Witness Theorem
\citep[Theorem A.2]{amir2023neural}
provides a Lebesgue-null set
$E_i\subseteq\X^{Q_i+1}$ such that, for every
$(p_1,\cdots,p_{Q_i+1})\notin E_i$ and every
$(\eta,\alpha,\beta)\in\Omega_i$,
\begin{align*}
H_i(p_n;\eta,\alpha,\beta)=0,\ \forall\,n\in[Q_i+1]
\implies
H_i(u;\eta,\alpha,\beta)=0,\ \forall\,u\in\X.
\end{align*}
We define $B_i\coloneqq E_i\times\X^{N-Q_i-1},\,
B\coloneqq\bigcup_{i=1}^L B_i,$
where $\X^0$ denotes a singleton.
Each $B_i$, and hence $B$, is Lebesgue-null.
For every $\P=(p_1,\cdots,p_N)\notin B$, the preceding
implication holds at every layer, simultaneously for
all $(\eta,\alpha,\beta)\in\Omega_i$.
In particular, $B$ does not depend on the pair of
parameters being compared.

Fix $\P\notin B$ and
$\param=(\theta_1,\cdots,\theta_L)$,
$\param'=(\theta_1',\cdots,\theta_L')$ satisfying
$\mathbf X_i^\param(\P)=\mathbf X_i^{\param'}(\P),\,
i=0,\cdots,L.$
We prove by induction that
$x_i(\cdot;\param)\equiv x_i(\cdot;\param')$ on $\X$.
For $i=0$, both maps are the identity.

Suppose the assertion holds for $i-1$, and set
$\eta=(\theta_1,\cdots,\theta_{i-1})$.
Then, for every $u\in\X$, we have
\begin{align*}
x_{i-1}(u;\eta)=x_{i-1}(u;\param)
=x_{i-1}(u;\param'),\,
\text{and}\,
H_i(u;\eta,\theta_i,\theta_i')=
\left\|
x_i(u;\param)-x_i(u;\param')
\right\|_2^2.
\end{align*}
Equality at the probes $\P$ therefore gives
$H_i(p_n;\eta,\theta_i,\theta_i')=0,\,
n=1,\cdots,N.$
Since $\P\notin B_i$, the finite witness theorem
implies
$H_i(u;\eta,\theta_i,\theta_i')=0$ for every $u\in\X$.
Thus $x_i(\cdot;\param)\equiv x_i(\cdot;\param')$,
completing the induction.
\end{proof}

\paragraph{When the activation is also injective and non polynomial:}
We now show that if the activation is analytic, injective, and non-polynomial, the sufficient number of probes can be substantially reduced. We prove a stronger version of Theorem~\ref{thm:hidden-weaker} for that. We first prove a couple of lemmas that will be useful in proving the result.

\begin{lemma}
\label{lem:finite-point-ridge-span}
Let $y_1,\ldots,y_K\in\R^q$ be pairwise distinct.  Let
$\sigma:\R\to\R$ be real analytic and non-polynomial.  Then
\begin{align*}
  \operatorname{span}
  \left\{
    \left(\sigma(w^\top y_1+b),\ldots,\sigma(w^\top y_K+b)\right)
    :
    (w,b)\in\R^q\times\R
  \right\}
  =
  \R^K.
\end{align*}
Consequently, for Lebesgue-a.e.
$
  (w_1,b_1),\ldots,(w_{K-1},b_{K-1})
  \in(\R^q\times\R)^{K-1},
$
the matrix
\begin{align*}
  M=
  \begin{pmatrix}
  1 & \cdots & 1\\
  \sigma(w_1^\top y_1+b_1) & \cdots & \sigma(w_1^\top y_K+b_1)\\
  \vdots & & \vdots\\
  \sigma(w_{K-1}^\top y_1+b_{K-1}) & \cdots &
  \sigma(w_{K-1}^\top y_K+b_{K-1})
  \end{pmatrix}
\end{align*}
is invertible.
\end{lemma}
\begin{proof}
We consider the non-trivial case $K\ge2$.
Choose $v\in\R^q$ such that the numbers $t_n\coloneqq v^\top y_n$,
$n\in[K]$, are pairwise distinct.  Such a choice is possible because the
bad set is the finite union of proper hyperplanes
$\bigcup_{n<n'}\{v\in\R^q:v^\top(y_n-y_{n'})=0\}$, each of which is
proper since $y_n\ne y_{n'}$.
We first prove the span claim.  Suppose, for contradiction, that the
span is proper.  Then there exists a nonzero vector
$c=(c_1,\ldots,c_K)\in\R^K$ such that
\begin{equation}
\label{eq:one-dimensional-annihilator}
\begin{split}
    &\sum_{n=1}^K c_n\sigma(w^\top y_n+b)=0
  \,\forall\,(w,b)\in\R^q\times\R\\
  \implies&\sum_{n=1}^K c_n\sigma(b+a t_n)=0
  \,\forall\,a,b\in\R,\,\text{by taking }w=av .
\end{split}
\end{equation}
Since $\sigma$ is real analytic and non-polynomial, none of the
derivatives $\sigma^{(k)}$, $k=0,\ldots,K-1$, is identically zero.
Indeed, if $\sigma^{(k)}\equiv0$ for some $k$, then $\sigma$ would be a
polynomial of degree at most $k-1$.  Therefore, for each
$k=0,\ldots,K-1$, the zero set of $\sigma^{(k)}$ is a proper discrete
subset of $\R$.  Hence the finite union
$\bigcup_{k=0}^{K-1}\{b\in\R:\sigma^{(k)}(b)=0\}$ does not cover $\R$.
Choose $b_0\in\R$ outside this union.  Then we have:
\begin{align*}
  \sigma^{(k)}(b_0)\ne0,
  \, k=0,\ldots,K-1.
\end{align*}
Now set $b=b_0$ in equation~\ref{eq:one-dimensional-annihilator}.  Since
$\sigma$ is analytic at $b_0$, for all sufficiently small $a$ we have:
\begin{align*}
  &0=\sum_{n=1}^K c_n\sigma(b_0+a t_n)
  =\sum_{n=1}^K c_n
  \sum_{k=0}^{\infty}
  \frac{\sigma^{(k)}(b_0)}{k!}
  a^k t_n^k\\
  \implies & 0 = \sum_{k=0}^{\infty}
  \frac{\sigma^{(k)}(b_0)}{k!}
  \left(\sum_{n=1}^K c_n t_n^k\right)a^k .
\end{align*}
Since the convergent power series is identically $0$ in a neighbourhood
around the origin, every coefficient vanishes, and since
$\sigma^{(k)}(b_0)\ne0$ for $k\le K-1$ we get
$\sum_{n=1}^K c_n t_n^k=0$ for $k=0,\ldots,K-1$.  Hence $c$ lies in the
kernel of the Vandermonde matrix
\begin{align*}
  V=
  \begin{pmatrix}
  1 & \cdots & 1\\
  t_1 & \cdots & t_K\\
  \vdots & & \vdots\\
  t_1^{K-1} & \cdots & t_K^{K-1}
  \end{pmatrix}.
\end{align*}
But the $t_n$ are pairwise distinct, so we have:
\begin{align*}
  \det V=\prod_{1\le n<n'\le K}(t_{n'}-t_n)\ne0.
\end{align*}
Therefore $c=0$, contradicting the choice of $c$.  This proves the span
claim.
Now we define
\begin{align*}
  v(w,b)\coloneqq
  \left(\sigma(w^\top y_1+b),\ldots,\sigma(w^\top y_K+b)\right)
  \in\R^K.
\end{align*}
We first show that there exists some $(K-1)$-tuple of parameters such
that the matrix $M$ is full rank.  We start with the one-dimensional
subspace $V_0\coloneqq\operatorname{span}\{\mathbf 1_K\}$.  If every
vector $v(w,b)$ belonged to $V_0$, then the span of all ridge vectors
would be contained in $V_0$, but these vectors span $\R^K$.  Hence there
exists $(w_1,b_1)$ such that $v(w_1,b_1)\notin V_0$.  Set
$V_1\coloneqq\operatorname{span}\{\mathbf 1_K,v(w_1,b_1)\}$.  Continuing
inductively (as long as $\dim V_j<K$ some ridge vector lies outside
$V_j$), after $K-1$ steps we obtain parameters
$(w_1,b_1),\ldots,(w_{K-1},b_{K-1})$ such that the $K$ vectors
\begin{align*}
  \mathbf 1_K,\,
  v(w_1,b_1),\ldots,v(w_{K-1},b_{K-1})
\end{align*}
are linearly independent in $\R^K$.  Since there are $K$ such vectors,
they form a basis of $\R^K$.  Equivalently, the corresponding matrix $M$
has nonzero determinant for at least one choice of parameters.
Now we define
\begin{align*}
    \Delta((w_1,b_1),\ldots,(w_{K-1},b_{K-1}))
  \coloneqq
  \det M.
\end{align*}
Since $\sigma$ is real analytic, every entry of $M$ is a real-analytic
function of the parameters $(w_1,b_1)$, $\ldots$, $(w_{K-1},b_{K-1})$,
and therefore $\Delta$ is real analytic on the connected open set
$(\R^q\times\R)^{K-1}$.  The preceding construction shows that $\Delta$
is not identically zero.  Hence its zero set is the zero set of a
nonzero real-analytic function on a connected open set, so it has
Hausdorff dimension strictly smaller than $(K-1)(q+1)$, implying zero
Lebesgue measure.  Therefore $M$ is invertible for Lebesgue-a.e. choice
of the ridge parameters.
\end{proof}

Now, for the following lemma, we define another notation for MLPs with width sequence $\Db\coloneqq(d=d_0,d_1,\cdots,d_{L-1}, d_L=m)$. For probes $\mP\in\X^N$ and $\theta\in\paramspace_\Db$, we recall the definitions $\overline{\mathbf X}_i^\param(\mP)\coloneqq
\begin{pmatrix}
\mathbf X_i^\param(\mP)\\ \mathbf 1_N^\top
\end{pmatrix}
\in\R^{(d_i+1)\times N}$, where $\mathbf X_i^\param(\mP)$ is as defined in equation~\ref{eq:mlpdefs}. We show that, for sufficiently large $N\coloneqq|\mP|$, the matrix $\overline{\mathbf X}_i^\param(\mP)$ is almost always full rank.

\begin{lemma}
\label{lem:generic-layerwise-rank}
Assume that $\sigma:\R\to\R$ is real analytic, injective, and
non-polynomial.  Let $N\ge\max_{0\le i\le L-1}(d_i+1)$, and define
\begin{align*}
  &E\coloneqq E_1\cup E_2,\ \text{ where}:\\
  &E_1\coloneqq
  \left\{
  \mP=(p_1,\ldots,p_N)\in\X^N:
  \operatorname{rank}
  \begin{pmatrix}
  p_1&\cdots&p_N\\
  1&\cdots&1
  \end{pmatrix}
  <d+1
  \right\},\\
  &E_2\coloneqq\bigcup_{1\le n<n'\le N}\{\mP\in\X^N:\,p_n=p_{n'}\}.
\end{align*}
Then $\dim_{\mathrm H}(E)<Nd$, and for every probe tuple
$\mP\in\X^N\setminus E$ there exists a Lebesgue-null set
$\mathcal M_\mP\subseteq\paramspace_\Db$ such that
\begin{align*}
  \operatorname{rank}\overline{\mathbf X}_i^{\param_0}(\mP)=d_i+1,
  \,\forall\,i\in\{0,\cdots,L-1\},
\end{align*}
for every $\param_0\in\paramspace_\Db\setminus\mathcal M_\mP$.
\end{lemma}
\begin{proof}
Since $\X$ has nonempty interior, choose $x_0\in\X$ and
$\epsilon>0$ such that $B(x_0,\epsilon)\subseteq\X$.  Let $e_1,\ldots,e_{d}$ be
the standard basis of $\R^{d}$.  Then the determinant
\begin{align*}
  P(p_1,\ldots,p_{d+1})
  \coloneqq
  \det
  \begin{pmatrix}
  p_1&\cdots&p_{d+1}\\
  1&\cdots&1
  \end{pmatrix}
\end{align*}
is nonzero for the choice of points
$\{x_0,x_0+\frac{\epsilon}{2}e_1,\cdots,x_0+\frac{\epsilon}{2}e_{d}\}$:
subtracting the first column from each of the others gives the value
$\pm(\epsilon/2)^{d}$.  Hence $P$ is not the zero polynomial.  If a
probe tuple belongs to $E_1$, then every $(d+1)\times(d+1)$ minor of the
matrix in the definition of $E_1$ vanishes, so
$P(p_1,\ldots,p_{d+1})=0$.  Thus $E_1\subseteq\{P=0\}\times\X^{N-d-1}$.
Since $P$ is a nonzero polynomial, $\{P=0\}$ has Hausdorff dimension
strictly smaller than $(d+1)d$.  Therefore $\dim_{\mathrm H}(E_1)<Nd$.
Moreover, each diagonal set $\{p_n=p_{n'}\}$ has Hausdorff dimension
$(N-1)d<Nd$.  Since $E$ is a finite union of such sets, we have
$\dim_{\mathrm H}(E)<Nd$.
 
Fix a probe tuple $\mP=(p_1,\ldots,p_N)\in\X^N\setminus E$.  Then the
probes are pairwise distinct and
$\operatorname{rank}\overline{\mathbf X}_0(\mP)=d+1$.  Note that
$\overline{\mathbf X}_i^\param(\mP)$ and the sampled states
$x_i(p_n;\param)$ depend on $\param=(B_0,\ldots,B_{L-1})$ only through
the blocks $B_0,\ldots,B_{i-1}$.  We shall prove inductively that, for
Lebesgue-a.e. choice of the blocks $B_0,\ldots,B_{i-1}$, the following
two properties hold:
\begin{align}
  &\operatorname{rank}\overline{\mathbf X}_i^\param(\mP)=d_i+1,
  \tag{$R_i$}\\
  &x_i(p_1;\param),\ldots,
  x_i(p_N;\param)
  \text{ are pairwise distinct.}
  \tag{$D_i$}
\end{align}
For $i=0$, property $(R_0)$ holds by the definition of $E_1$, and
$(D_0)$ holds because the probes $p_1,\ldots,p_N$ are pairwise
distinct; neither involves any block.
Now fix $0\le i\le L-2$ and blocks $B_0,\ldots,B_{i-1}$ for which
$(D_i)$ holds.  Set
\begin{align*}
  y_n\coloneqq x_i(p_n;\param)\in\R^{d_i},
  \, n=1,\ldots,N.
\end{align*}
By $(D_i)$, the vectors $y_1,\ldots,y_N$ are pairwise distinct.
We first prove $(R_{i+1})$.  Let $K\coloneqq d_{i+1}+1$.  Since
$i+1\le L-1$, we have $K\le N$, so the first $K$ vectors
$y_1,\ldots,y_K$ are pairwise distinct.  Write $(w_j^\top,b_j)$,
$j\in[d_{i+1}]$, for the rows of $B_i$ and define
\begin{align*}
  M_{i+1}\coloneqq\begin{pmatrix}
  1&\cdots&1\\
  \sigma\left(w_1^\top y_1+b_1\right)
  &\cdots&
  \sigma\left(w_1^\top y_K+b_1\right)\\
  \vdots&&\vdots\\
  \sigma\left(w_{d_{i+1}}^\top y_1+b_{d_{i+1}}\right)
  &\cdots&
  \sigma\left(w_{d_{i+1}}^\top y_K+b_{d_{i+1}}\right)
  \end{pmatrix}.
\end{align*}
Up to moving the row of ones to the bottom, $M_{i+1}$ consists of the
first $K$ columns of $\overline{\mathbf X}_{i+1}^\param(\mP)$, i.e. it
is a $(d_{i+1}+1)\times(d_{i+1}+1)$ minor of
$\overline{\mathbf X}_{i+1}^\param(\mP)$.  Applying
Lemma~\ref{lem:finite-point-ridge-span} in $\R^{d_i}$ to the $K$ points
$y_1,\ldots,y_K$ gives that, for Lebesgue-a.e. choice of the rows
$(w_1^\top,b_1),\ldots,(w_{d_{i+1}}^\top,b_{d_{i+1}})$ of $B_i$, the
matrix $M_{i+1}$ is invertible.  Therefore, for Lebesgue-a.e. block
$B_i$, we have $\operatorname{rank}\overline{\mathbf X}_{i+1}^\param(\mP)=d_{i+1}+1$.
Thus $(R_{i+1})$ holds for Lebesgue-a.e. block $B_i$, conditional on the
previous blocks satisfying $(D_i)$.
We next prove $(D_{i+1})$.  Fix $n\ne n'$.  Since $y_n\ne y_{n'}$, if
$x_{i+1}(p_n;\param)=x_{i+1}(p_{n'};\param)$, then, by injectivity of
$\sigma$ applied coordinatewise, we get:
\begin{align*}
  W_iy_n+b_i
  =
  W_iy_{n'}+b_i .
\end{align*}
Equivalently, $W_i(y_n-y_{n'})=0$.  Since $y_n-y_{n'}\ne0$, this is a
proper linear condition on the entries of $W_i$.  Hence, for fixed
previous blocks, it defines a Lebesgue-null subset of the space of
blocks $B_i$.  Taking the finite union over all pairs $n\ne n'$, we get
that $(D_{i+1})$ holds for Lebesgue-a.e. block $B_i$.
It remains to pass from a fixed choice of the previous blocks to all
parameters.  Write $q_k\coloneqq d_{k+1}(d_k+1)$ for the size of block
$B_k$ and, for $-1\le i\le L-2$, let
$\mathcal B_i\subseteq\R^{q_0+\cdots+q_i}$ be the set of prefixes
$(B_0,\ldots,B_i)$ for which $(R_k)$ or $(D_k)$ fails for some
$k\le i+1$; thus $\mathcal B_{-1}=\emptyset$.  Each $\mathcal B_i$ is a
Borel set: the failure of $(R_k)$ is the vanishing of all
$(d_k+1)\times(d_k+1)$ minors of $\overline{\mathbf X}_k^\param(\mP)$,
finitely many real-analytic functions of the prefix, and the failure of
$(D_k)$ is a finite union of equalities between real-analytic
functions.  If $(B_0,\ldots,B_i)\in\mathcal B_i$, then either
$(B_0,\ldots,B_{i-1})\in\mathcal B_{i-1}$, or this prefix satisfies
$(D_i)$ and $B_i$ belongs to the set $F(B_0,\ldots,B_{i-1})$ of blocks
for which $(R_{i+1})$ or $(D_{i+1})$ fails, which we have just shown to
be Lebesgue-null.  Hence, by Tonelli's theorem,
\begin{align*}
  \lambda(\mathcal B_i)
  \le
  \lambda\left(\mathcal B_{i-1}\times\R^{q_i}\right)
  +\int_{(B_0,\ldots,B_{i-1})\notin\mathcal B_{i-1}}
  \lambda\left(F(B_0,\ldots,B_{i-1})\right)\,d(B_0,\ldots,B_{i-1})
  =0 ,
\end{align*}
where the first term vanishes because $\mathcal B_{i-1}$ is null by the
induction hypothesis and the second because its integrand vanishes.
Iterating this over $i=0,\ldots,L-2$, every $\mathcal B_i$ is
Lebesgue-null.  Finally, since $(R_i)$ only involves the blocks
$B_0,\ldots,B_{i-1}$, the set of parameters $\param\in\R^{q_\Db}$ for
which $(R_i)$ fails for some $i\in\{0,\cdots,L-1\}$ is contained in
$\mathcal B_{L-2}\times\R^{q_{L-1}}$, which is null; its intersection
with $\paramspace_\Db$ is the required $\mathcal M_\mP$.  This completes
the proof.
\end{proof}

We now prove Theorem~\ref{thm:hidden}, which is a stronger version of Theorem~\ref{thm:hidden-weaker}.

\begin{restatable}[Reconstruction from hidden layer activations]
{theorem}{thmhidden}
\label{thm:hidden}
Let $\X\subseteq\RR^{d_0}$ be compact with nonempty interior and  $\sigma:\RR\to\RR$ be real analytic, injective, and
non-polynomial. Consider fully connected MLPs with biases,
widths $\Db=(d_0,\ldots,d_L)$, and raw parameter space
$\paramspace_\Db=\RR^{q_\Db}$, where $q_\Db=\sum_{i=0}^{L-1}d_{i+1}(d_i+1)$. Suppose $N\ge 1+\max_{0\le i<L}d_i$. For Lebesgue-almost every probe tuple $\mP\in\X^N$, there
exists a nonzero real-analytic function
$\Delta_\mP:\RR^{q_\Db}\to\RR$ such that
$\mathcal M_\mP=\{\param:\Delta_\mP(\param)=0\}$ and we have
\begin{align*}
    \all(\param,\mP)=\all(\param_0,\mP)
\,\iff\,\param=\param_0,\,\forall\,\param_0\in\paramspace_\Db\setminus \mathcal M_\mP,\,\forall\,\param\in\paramspace_\Db
\end{align*}

The exceptional set $\mathcal M_\mP$ is closed and satisfies
$\dim_{\mathrm H}\mathcal M_\mP\le q_\Db-1$; hence it is
Lebesgue-null and nowhere dense.
\end{restatable}

\begin{proof}

Let $E\subseteq\X^N$ be the exceptional probe set in
Lemma~\ref{lem:generic-layerwise-rank}, applied with the given $N$.
Then $E$ is Lebesgue-null. Fix $\mP\in\X^N\setminus E$ and define
\begin{align*}
\Delta_\mP(\param)\coloneqq
\prod_{i=0}^{L-1}
\det\left(
\overline{\mathbf X}_i^\param(\mP)
\overline{\mathbf X}_i^\param(\mP)^\top
\right),\,\mathcal M_\mP\coloneqq
\left\{
\param\in\R^{q_\Db}:\Delta_\mP(\param)=0
\right\}.
\end{align*}
Each entry of $\overline{\mathbf X}_i^\param(\mP)$ is real
analytic in $\param$, so $\Delta_\mP$ is real analytic on
$\R^{q_\Db}$.

For any real matrix $A$, the identity
$v^\top AA^\top v=\|A^\top v\|_2^2$ shows that
$AA^\top$ is positive definite exactly when $A$ has full
row rank. Consequently,
\begin{align*}
\Delta_\mP(\param)>0
\,\Longleftrightarrow\,
\operatorname{rank}\overline{\mathbf X}_i^\param(\mP)=d_i+1
\,\text{for every }i=0,\ldots,L-1.
\end{align*}
Lemma~\ref{lem:generic-layerwise-rank} gives these ranks
simultaneously for Lebesgue-almost every
$\param\in\R^{q_\Db}$. In particular, they hold at some
$\param_*$, and hence $\Delta_\mP(\param_*)>0$.
Thus $\Delta_\mP$ is not identically zero.

By continuity, $\mathcal M_\mP$ is closed and $\dim_{\mathrm H}\mathcal M_\mP\le q_\Db-1$ since it forms the zero-set of an analytic function. Therefore $\mathcal M_\mP$ is Lebesgue-null. It has empty
interior, since every nonempty open subset of
$\R^{q_\Db}$ has positive Lebesgue measure. Being closed
with empty interior, it is nowhere dense.

Fix
$\param_0=(B_0^0,\ldots,B_{L-1}^0)
\in\paramspace_\Db\setminus\mathcal M_\mP$.
By the definition of $\Delta_\mP$,
\begin{equation}
\label{eq:generic-rank-all-layers}
\operatorname{rank}
\overline{\mathbf X}_i^{\param_0}(\mP)=d_i+1,
\, i=0,\ldots,L-1.
\end{equation}
If $\param=\param_0$ then trivially $\all(\param,\mP)=\all(\param_0,\mP)$.

Conversely, let $\param=(B_0,\ldots,B_{L-1})\in\paramspace_\Db$ satisfy
$\all(\param,\mP)=\all(\param_0,\mP)$.  We show that $B_i=B_i^0$ layer
by layer.  Fix $i\in\{0,\cdots,L-1\}$.  By definition,
$
  \mathbf X_i^\param(\mP)=\mathbf X_i^{\param_0}(\mP)
  \,\text{and}\,
  \mathbf X_{i+1}^\param(\mP)=\mathbf X_{i+1}^{\param_0}(\mP),
$.  The first equality gives
$\overline{\mathbf X}_i^\param(\mP)=\overline{\mathbf X}_i^{\param_0}(\mP)$,
and the second equality reads
$
  \sigma\left(B_i\,\overline{\mathbf X}_i^{\param_0}(\mP)\right)
  =\sigma\left(B_i^0\,\overline{\mathbf X}_i^{\param_0}(\mP)\right)
  \,\text{entrywise.}
$. Since $\sigma$ is injective and applied entrywise, this implies
$B_i\overline{\mathbf X}_i^{\param_0}(\mP)=B_i^0\overline{\mathbf X}_i^{\param_0}(\mP)$.
Equivalently, we have
\begin{align*}
  \left(W_i-W_i^0\ \ \ b_i-b_i^0\right)\,
  \overline{\mathbf X}_i^{\param_0}(\mP)
  =0 .
\end{align*}
By equation~\ref{eq:generic-rank-all-layers},
$\overline{\mathbf X}_i^{\param_0}(\mP)$ has full row rank, so the
$(d_i+1)\times(d_i+1)$ matrix
$\overline{\mathbf X}_i^{\param_0}(\mP)\,\overline{\mathbf X}_i^{\param_0}(\mP)^\top$
is invertible and
\begin{align*}
  \overline{\mathbf X}_i^{\param_0}(\mP)^{+}
  \coloneqq
  \overline{\mathbf X}_i^{\param_0}(\mP)^\top
  \left(\overline{\mathbf X}_i^{\param_0}(\mP)\,\overline{\mathbf X}_i^{\param_0}(\mP)^\top\right)^{-1}
\end{align*}
is a right inverse of $\overline{\mathbf X}_i^{\param_0}(\mP)$.
Multiplying the previous display by it on the right gives
$W_i=W_i^0$ and $b_i=b_i^0$, i.e. $B_i=B_i^0$.  Since $i$ was
arbitrary, $\param=\param_0$.  Hence
$\all(\param,\mP)=\all(\param_0,\mP)\iff\param=\param_0$, and in
particular $\all(\cdot,\mP)$ is injective on
$\paramspace_\Db\setminus\mathcal M_\mP$.
\end{proof}

\paragraph{Algorithm to recover weights from $\all$ for generic targets}
Our proof of Theorem~\ref{thm:hidden} gives us a way to actually reconstruct the weights rather than just the functions for generic targets. We formally describe the algorithm as follows:

\begin{algorithm}[H]
\caption{Weights from hidden activations}
\label{alg:weights}
\begin{algorithmic}[1]
\Require $\mP=(p_1,\cdots,p_N)$,
\Statex \hspace{\algorithmicindent} $\all(\param_0,\mP)=(\mathbf X_1,\cdots,\mathbf X_L)$
\State $\mathbf X_0\gets(p_1,\cdots,p_N)$
\For{$i=0,\cdots,L-1$}
\State $\overline{\mathbf X}_i\gets\begin{pmatrix}\mathbf X_i\\ \mathbf 1_N^\top\end{pmatrix}$
\State $\mathbf Z_{i+1}\gets\sigma^{-1}(\mathbf X_{i+1})$ \Comment{entrywise}
\State $\widehat B_i\gets\mathbf Z_{i+1}\,\overline{\mathbf X}_i^{\dagger}$ \Comment{pseudoinverse}
\State $(\widehat W_i\ \ \widehat b_i)\gets\widehat B_i$
\EndFor
\State \Return $\widehat\param\coloneqq(\widehat B_0,\cdots,\widehat B_{L-1})$
\end{algorithmic}
\end{algorithm}

\subsection{Lower bounds on probe counts}
\label{app:proofs:lower}

For analytic MLPs, we have established upper bounds on the number of probes required for unique function identification, using either final-layer or hidden-layer probes. We now establish nearly matching lower bounds for shallow networks. Since these lower bounds hold for one-hidden-layer networks of arbitrary width, they also provide necessary lower bounds for the general family of MLPs with analytic activations (or analytic, injective, non-polynomial activations when observing $\all$). The formal results are stated in Theorem~\ref{thm:tight}.

We now prove several lemmas needed to establish the tightness results in Theorem~\ref{thm:tight}.

\begin{lemma}
\label{lem:trig-zeros-tight}
Let $K\ge1$ and let $T$ be a nonzero real trigonometric polynomial
of the form $T(t)=A_0+\sum_{k=1}^K\left(A_k\cos(kt)+B_k\sin(kt)\right)$, not identically $0$. Then $T$ has at most $2K$ distinct zeros in $[0,2\pi)$.
Moreover, suppose another nonzero sum $H$ of the same form has $2K$ distinct zeros in $[0,2\pi)$. If $T$ vanishes at all those points, then $T=\lambda H$ for some real number $\lambda$.
\end{lemma}
\begin{proof}
Write $T(t)=\sum_{k=-K}^K c_ke^{ikt}$ and define $P_T(z)\coloneqq\sum_{k=-K}^K c_kz^{k+K}.$
The polynomial $P_T$ is nonzero, has degree at most $2K$, and satisfies
$P_T(e^{it})=e^{iKt}T(t)$. Distinct zeros modulo $2\pi$ give distinct
roots of $P_T$, proving the first assertion. For the second assertion,
$P_H$ has degree $2K$ and $2K$ distinct roots, all shared by $P_T$.
Hence $P_T=\lambda P_H$. Therefore $T=\lambda H$, with $\lambda$
real because both functions are real-valued and $H\not\equiv0$.
\end{proof}

\begin{lemma}
\label{lem:trig-rank-tight}
Let $K\ge1$ and let $H(t)=\sum_{k=1}^K\gamma_k\sin(kt)\not\equiv0$ have $2K$
distinct zeros $r_1,\ldots,r_{2K}\in[0,2\pi)$. We define:
\begin{align*}
 \mathcal Z&\coloneqq(r_1,\ldots,r_{2K},r_1+2\pi,\ldots,r_{2K}+2\pi),\\
 \mathcal V_K&\coloneqq
 \left\{Q(t)+tR(t):
 \begin{array}{l}
 Q\in\operatorname{span}\{1,\sin(kt),\cos(kt):k\in[K]\},\\
 R\in\operatorname{span}\{\cos(kt):k\in[K]\}
 \end{array}\right\}.
\end{align*}
Writing $\mathcal Z=(z_1,\ldots,z_{4K})$, the set of sampled vectors $\left\{
\left(S(z_1),\ldots,S(z_{4K})\right): S\in\mathcal V_K \right\}$ is a $3K$-dimensional linear subspace of $\R^{4K}$. Moreover, a function $S\in\mathcal V_K$ satisfies
$S(z_n)=0$ for every $n=1,\ldots,4K$ if and only if there exists $\lambda\in\R$ such that
$S(t)=\lambda H(t)$ for every $t\in\R$.
\end{lemma}
\begin{proof}
Every $S\in\mathcal V_K$ can be written as
$S(t)=Q(t)+tR(t),\,Q(t)=a_0+\sum_{k=1}^K,\,\left(a_k\sin(kt)+b_k\cos(kt)\right),\,R(t)=\sum_{k=1}^K c_k\cos(kt).$

We first verify that these coefficients are uniquely determined
by $S$. It suffices to show that $S(t)=0$ for every $t$ forces
all coefficients to be zero. Since $Q$ and $R$ are
$2\pi$-periodic, subtracting the identity at $t$ from the
identity at $t+2\pi$ gives $0=S(t+2\pi)-S(t)=2\pi R(t).$
Thus $R=0$, and then $Q=0$. For each $\ell=1,\ldots,K$,
the elementary sine and cosine integration identities give
$
\int_0^{2\pi}R(t)\cos(\ell t)\,dt=\pi c_\ell,\,\int_0^{2\pi}Q(t)\,dt=2\pi a_0,\,\int_0^{2\pi}Q(t)\sin(\ell t)\,dt=\pi a_\ell,\,\int_0^{2\pi}Q(t)\cos(\ell t)\,dt=\pi b_\ell.
$, which uniquely identifies the coefficients, implying linear independence.

We next characterize the functions whose sampled values are
all zero. Suppose $S(z_n)=0$ for every $n=1,\ldots,4K$.
For each $j=1,\ldots,2K$, periodicity gives $0=S(r_j)=Q(r_j)+r_jR(r_j),\,0=S(r_j+2\pi)=Q(r_j)+(r_j+2\pi)R(r_j).$ Subtracting these yields $R(r_j)=0,\, j=1,\ldots,2K.$
Since $R$ is a finite cosine sum using frequencies at most
$K$, Lemma~\ref{lem:trig-zeros-tight} implies
$R=\mu H$ for some $\mu\in\R$; if $R=0$, take $\mu=0$.

For every real $t$, cosine symmetry and sine antisymmetry give $R(-t)=R(t),\, H(-t)=-H(t).$ Therefore we get: $\mu H(t)=R(t)=R(-t)=\mu H(-t)=-\mu H(t).$
It follows that $2\mu H(t)=0$ for every $t$. Because $H$ is not identically zero, $\mu=0$, and hence $R=0$.

We now have $S=Q$ and $Q(r_j)=0$ for all $j$.
A second application of Lemma~\ref{lem:trig-zeros-tight}
gives $S(t)=Q(t)=\lambda H(t)\,\forall\,t\in\R$ for some $\lambda\in\R$, including $\lambda=0$ when $Q=0$.
Conversely, every function $S=\lambda H$ vanishes at all
$r_j$ and $r_j+2\pi$, since $H$ is $2\pi$-periodic.
This proves the stated equivalence.

It remains to determine the dimension of the set of sampled
vectors. For any function $\phi$, write
$
v_\phi\coloneqq
\left(\phi(z_1),\ldots,\phi(z_{4K})\right)
\in\R^{4K}.
$
Choose $k_*\in\{1,\ldots,K\}$ such that
$\gamma_{k_*}\ne0$, which is possible because $H$ is not
identically zero. Consider the following collection of
$3K$ functions:
\begin{align*}
\mathcal B\coloneqq\{1\}\cup\{\sin(kt):1\le k\le K,\ k\ne k_*\}
\cup\{\cos(kt):1\le k\le K\}\cup\{t\cos(kt):1\le k\le K\}.
\end{align*}
We show that their sampled vectors
$\{v_\phi:\phi\in\mathcal B\}$ are linearly independent
and that every sampled vector from $\mathcal V_K$ is a linear
combination of them.

For independence, suppose real coefficients $d_\phi$
satisfy $\sum_{\phi\in\mathcal B}d_\phi v_\phi=0.$
Define $S(t)=\sum_{\phi\in\mathcal B}d_\phi\phi(t).$
Then $S\in\mathcal V_K$ and $S(z_n)=0$ for every $n$.
The equivalence already proved implies $S=\lambda H$.

The coefficient of $\sin(k_*t)$ in the displayed expression
for $S$ is zero, because this function was omitted from
$\mathcal B$. Its coefficient in $\lambda H$ is
$\lambda\gamma_{k_*}$. Uniqueness of the coefficient
representation therefore gives $0=\lambda\gamma_{k_*}.$
Since $\gamma_{k_*}\ne0$, we obtain $\lambda=0$ and hence
$S=0$. The linear independence established at the beginning
of the proof now forces $d_\phi=0$ for every
$\phi\in\mathcal B$.
Thus the $3K$ sampled vectors are linearly independent.

Finally, because $H$ vanishes at every point of $\mathcal Z$, $0=v_H=\sum_{k=1}^K\gamma_k v_{\sin(k\,\cdot)}.$
Solving for the sampled vector of the omitted sine function
gives $v_{\sin(k_*\,\cdot)}=-\sum_{\substack{k=1\\k\ne k_*}}^K\frac{\gamma_k}{\gamma_{k_*}}v_{\sin(k\,\cdot)}.$
Every function in $\mathcal V_K$ is a linear combination of
the functions in $\mathcal B$ and this one omitted sine
function. The last identity therefore shows that every
sampled vector from $\mathcal V_K$ is a linear combination
of $\{v_\phi:\phi\in\mathcal B\}$.

Conversely, every linear combination of these sampled vectors
is obtained by sampling the corresponding linear combination
of functions in $\mathcal B$, which belongs to $\mathcal V_K$.
Hence the set of sampled vectors is exactly the linear
subspace generated by these $3K$ linearly independent vectors.
Its dimension is therefore $3K$.
\end{proof}

\begin{lemma}
\label{lem:tanh-indep}
Let $\X\subseteq\R$ be a nonempty open interval. If $w_j\ne0$ for
all $j\in[m]$ and $b_j/w_j\ne b_k/w_k$ whenever $j\ne k$, then
$1,\tanh(w_1x+b_1),\ldots,\tanh(w_mx+b_m)$ are linearly independent
on $\X$.
\end{lemma}
\begin{proof}
Suppose $C+\sum_{j=1}^m A_j\tanh(w_jx+b_j)=0$ on $\X$.
Extend the left-hand side to the meromorphic function
\begin{align*}
 \Phi(z)\coloneqq C+\sum_{j=1}^m A_j\tanh(w_jz+b_j).
\end{align*}
Since $\tanh=\sinh/\cosh$ and $\cosh'=\sinh$, the poles of its
$j$th summand occur at
\begin{align*}
 z_{j,k}=\frac{i\pi(k+1/2)-b_j}{w_j},\, k\in\mathbb Z,
\end{align*}
and have residue $A_j/w_j$. Their real part is $-b_j/w_j$, so no
pole is shared by two distinct summands. The union of the pole sets
is closed and discrete, and its complement in $\mathbb C$ is connected:
a segment between any two points can be detoured around the finitely
many poles it meets. By the Identity Theorem of complex analysis, $\Phi$ is zero on this
complement because it vanishes on $\X$. Its residue at every pole
therefore vanishes. Thus $A_j/w_j=0$ for every $j$, giving $A_j=0$
and then $C=0$.
\end{proof}

\begin{restatable}[Tightness of probe complexities]{theorem}{thmtight}
\label{thm:tight}
Let $f(u,\theta)$ denote a shallow universal MLP,  with analytic activation $\sigma$, and layers of dimension  $\Db=(1,m,1)$ with $m\ge1$. Thus, the total number of parameters is $q=3m+1$ and the maximal dimension of any layer is $d_{\max}=m $. In each of the following cases, there exists
a compact parameter space $\paramspace_\Db\subseteq\R^q$
with nonempty interior for which the stated conclusion holds. We have the following:
\begin{enumerate}
\item  There exists a real-analytic activation $\sigma$ (namely, $\sigma=\sin$) for which the following holds: if  $N\le2q-2=6m$ then there exists a non-empty open interval $\mathcal{U}_N\subseteq\mathcal \RR^N$ such that for all $\mP\in\mathcal U_N$
$$\exists\,\param_1,\param_2\in\paramspace_\Db, \text{ such that } \final(\param_1,\mP)=\final(\param_2,\mP) \text{ and } f(\cdot;\param_1)\not\equiv f(\cdot;\param_2)\ .$$
Thus, a generic shared probe of length $2q-2$ cannot identify $f(\cdot;\theta)$ from the input-output evaluations.
\item There exists a real-analytic, non-polynomial and injective activation $\sigma$ (namely, $\sigma=\tanh$) for which the following holds: $N\le d_{\max}\coloneqq m$.  Then, for every $\mP\in\X^N$,
$$\exists\,\param_1,\param_2\in\paramspace_\Db, \text{ such that }  \all(\param_1,\mP)=\all(\param_2,\mP) \text{ and } f(\cdot;\param_1)\not\equiv f(\cdot;\param_2)\ .$$

Thus, probes of length $N\le d_{\max}$ cannot identify $f(\cdot;\theta)$ from the hidden neuron evaluations.
\end{enumerate}
\end{restatable}

\begin{proof}
We start with two observations:
\begin{itemize}
    \item To show that the lower bound of $2q-1$ is necessary for the family of all MLPs with analytic activations, it is sufficient to come up with one analytic activation for which the $2q-2$ probes don't suffice. We use $\sigma = \sin$ in part 1 and $\sigma = \tanh$ in part 2.

    \item We show the necessary lower bounds $2q-1$  or $d_{\max} + 1$ using the layer widths $\Db=(1,m,1)$, but we show the tightness for any $m\in\NN$. Now by taking $m\to\infty$ we can conclude the general lower bound of $2q-1$ is tight. 
\end{itemize}
With these arguments, we now prove the two constructive counter examples that establish the lower bounds.

We start with proving the first part. First we consider an interval containing $[0,4\pi]$.
We set $K\coloneqq2m$, $\epsilon\coloneqq1/(100m)$, and choose base
parameters $\param_1^0,\param_2^0$ whose affine pre-activations are
\begin{align*}
 g_1(t)&\coloneqq g(t;\param_1^0)=\epsilon\sum_{k=1}^{m}\sin(kt),\\
 g_2(t)&\coloneqq g(t;\param_2^0)=\epsilon\sum_{k=m+1}^{K-1}\sin(kt)+\sin(Kt),\\
H(t)&\coloneqq g_1(t)-g_2(t).
\end{align*}
Both networks have exactly $m$ hidden neurons, positive distinct
slopes, nonzero readout coefficients, and zero biases. 

For $t_j=(j+1/2)\pi/K$, $j=0,\ldots,2K-1$, one has
\begin{align*}
 H(t_j)=-(-1)^j+e_j,
 \, |e_j|\le(K-1)\epsilon<1/50.
\end{align*}
The signs alternate at these points, including across the final
interval to $t_0+2\pi$. The intermediate value theorem and
Lemma~\ref{lem:trig-zeros-tight} therefore give exactly $2K$
distinct zeros modulo $2\pi$. We choose representatives
$r_1,\ldots,r_{2K}\in[0,2\pi)$ and form the $4K$-tuple
$\mathcal Z$ of Lemma~\ref{lem:trig-rank-tight}.

We consider the derivative, with respect to both parameter vectors, of
the difference between the pre-activations. At the base pair it ranges over
the functions $Q(t)+tR(t),\, Q+tR\in\mathcal V_K.$
This can be seen by considering the fact that, for a neuron with base slope $k$, base bias zero, and nonzero base readout coefficient $a$, the three derivatives are:
\begin{align*}
\partial_a\left(a\sin(wt+b)\right)=\sin(kt),\,\partial_b\left(a\sin(wt+b)\right)=a\cos(kt),\,\partial_w\left(a\sin(wt+b)\right)=at\cos(kt).
\end{align*}
The two networks together contain each slope $k\in[K]$ exactly once;
the difference of their output-bias variations supplies the constant
function. Thus every coefficient of $Q$ and $R$ can be varied
independently. By Lemma~\ref{lem:trig-rank-tight}, the Jacobian of
the pre-activation difference on $\mathcal Z$ has rank
$3K=2q-2$. We choose $3K$ independent rows and let $\mP_0$ be their
probe locations.

We put $N_0\coloneqq2q-2$ and define
\begin{align*}
 \mathcal E(\param_1,\param_2,\mP)
 \coloneqq\left(g(p_n;\param_1)-g(p_n;\param_2)\right)_{n=1}^{N_0}.
\end{align*}
Then $\mathcal E(\param_1^0,\param_2^0,\mP_0)=0$, and its
parameter Jacobian has an invertible $N_0\times N_0$ minor. We take
the remaining $2q-N_0=2$ parameter coordinates at their base values.
The implicit function theorem gives parameter pairs
$\param_1(\mP),\param_2(\mP)$, continuous on a nonempty open
neighborhood $\mathcal U_{N_0}$ of $\mP_0$, such that $\mathcal E(\param_1(\mP),\param_2(\mP),\mP)=0\,\forall\,\mP\in\mathcal U_{N_0}.$
Consequently their final sine outputs agree at every probe.

Choose $t_*\in(0,4\pi)$ with $H(t_*)\ne0$. At the base pair, we have $|g_1(t)|\le m\epsilon<\pi/2,\,|g_2(t)|\le1+(m-1)\epsilon<\pi/2.$
The sine function is injective on $(-\pi/2,\pi/2)$, so
$f(t_*;\param_1^0)\ne f(t_*;\param_2^0)$. We shrink the neighbourhood
$\mathcal U_{N_0}$ so that this inequality persists, the probe
entries remain distinct, and both parameter vectors lie in compact
boxes around their base values with positive distinct slopes and
nonzero readout coefficients. Their union is a compact set
$\mathcal K$ with nonempty interior containing every constructed pair.
Thus the final functions are different for every tuple in this open set.

For a general interval $\X$, choose $[a,b]\subset\X$, $a<b$, and
$t=\alpha x+\beta$, $\alpha>0$, mapping $[a,b]$ onto $[0,4\pi]$.
Replacing each hidden row $(w_j,b_j)$ by
$(\alpha w_j,\beta w_j+b_j)$ is an invertible linear change of its
parameters. It transports the construction, its open probe set, and
the compact parameter set into $\X$.

For $N<N_0$, let $\mathcal U_N$ be the projection of
$\mathcal U_{N_0}$ onto the first $N$ probe coordinates. This set
is nonempty and open. Every tuple in $\mathcal U_N$ extends to a
tuple in $\mathcal U_{N_0}$, whose colliding pair also agrees on
the shorter tuple and belongs to the same $\mathcal K$. Since a
nonempty open set has positive Lebesgue measure, generic shared
separation is impossible for $N\le2q-2$.

Now we prove the second part. Fix any
$\param_0=(w,b,a,c)\in\operatorname{int}(\paramspace_\Db)$
such that $w_j\ne0$ for every $j$ and
$b_jw_k-b_kw_j\ne0$ whenever $j<k$.
These conditions hold for Lebesgue-almost every interior
parameter, because their failure is a finite union of zero sets
of nonzero polynomials.

Fix $1\le N\le m$ and an arbitrary probe tuple $\mP\in\X^N$.
The matrix
\begin{align*}
A
&\coloneqq \overline{\mathbf X}_1^{\param_0}(\mP)=
\begin{pmatrix}
h_1(p_1;\param_0)&\cdots&h_1(p_N;\param_0)\\
\vdots&&\vdots\\
h_m(p_1;\param_0)&\cdots&h_m(p_N;\param_0)\\
1&\cdots&1
\end{pmatrix}
\in\R^{(m+1)\times N}
\end{align*}
has more rows than columns. Hence there exists
$\delta\in\R^{m+1}\setminus\{0\}$ with $\delta^\top A=0$.
Define $\param(t)$ by keeping $w,b$ fixed and replacing the
last affine block by
\[
B_1(t)\coloneqq B_1^0+t\delta^\top,
\qquad B_1^0=(a^\top\ c).
\]
The hidden functions are unchanged, and
\begin{align*}
g(x;\param(t))-g(x;\param_0)
&=
t\left(
\delta_{m+1}
+\sum_{j=1}^m\delta_jh_j(x;\param_0)
\right).
\end{align*}
Since $\delta^\top A=0$, the right-hand side vanishes at
every probe. Therefore the sampled final outputs also agree,
and
\[
\all(\param(t),\mP)=\all(\param_0,\mP).
\]

It remains to show that the realized functions are different.
By Lemma~\ref{lem:tanh-indep}, the functions
$1,h_1(\cdot;\param_0),\ldots,h_m(\cdot;\param_0)$
are linearly independent on $\X$. Since $\delta\ne0$, the function
\[
r(x)\coloneqq
\delta_{m+1}+\sum_{j=1}^m\delta_jh_j(x;\param_0)
\]
is not identically zero on $\X$. Consequently, for every
$t\ne0$, the output preactivations
$g(\cdot;\param(t))$ and $g(\cdot;\param_0)$ differ somewhere
on $\X$. Injectivity of the outer $\tanh$ therefore gives
\[
f(\cdot;\param(t))\not\equiv f(\cdot;\param_0)
\quad\text{on }\X.
\]

Finally, $\param_0\in\operatorname{int}(\paramspace_\Db)$
and $\param(t)\to\param_0$. Thus all sufficiently small
nonzero $t$ give parameters in $\paramspace_\Db$ and in
any prescribed neighborhood of $\param_0$.
This proves that, for every probe tuple with $N\le m$,
almost every interior target admits arbitrarily close,
functionally distinct competitors with identical complete
probe observations.

For this shallow $\tanh$ family, $d_{\max}=m$ hidden probes
fail to identify almost every interior target, whereas the
upper bound guarantees identification of almost every target
from almost every tuple of $d_{\max}+1$ probes.
Thus the hidden-probe threshold is sharp in the worst case
over the class of analytic, injective, non-polynomial activations.

For the $\sin$ family, the nonempty open set of failing
probe tuples at $N\le2q-2$ rules out a general
almost-everywhere-in-probes identification guarantee at
those cardinalities. Hence the sufficient bound $2q+1$
is within two probes of the lower bound $2q-1$ for such
a guarantee. We do not claim that these bounds are tight
for every individual activation.
\end{proof}

\subsection{ReLU Network Identification from probes}
\label{app:proofs:relu}

We begin with a counterexample showing that ReLU networks cannot be exactly reconstructed from finitely many probes.

\proprelunon*

\begin{proof}
    Consider hat functions of the form
$$g_{c,\epsilon}(x)\coloneqq\relu(x-c+\epsilon)-2\relu(x-c)+\relu(x-c-\epsilon)=\max\{\epsilon-|x-c|,0\}.$$
For any $c,\epsilon$, such a function is a $\relu$ network with three hidden neurons, and it is supported in $[c-\epsilon,c+\epsilon] $. Given any probes
$p_1,\cdots,p_N$, choose $c\in(0,1)$ off the probes and
$0<\epsilon_1<\epsilon_2$ below the distance from $c$ to the probes and to
the boundary; then $g_{\epsilon_1}$ and $g_{\epsilon_2}$ vanish at every
probe and differ at $c$. Thus no finite set of probes, however chosen, can completely identify  $\relu$ networks.
\end{proof}

\paragraph{The probabilistic setting of identification:} Here we describe the probabilistic setting of function identification in ReLU Nets. Let $\mu$ be a Borel
probability measure on $\X$. The probes are drawn
from $\mu,\,\mP=(p_1,\cdots,p_N)$ with $p_1,\cdots,p_N$ i.i.d.\ $\mu$. We also consider the same measure to quantify agreement between functions.  For two parameters $\param_0,\param\in\R^{q_\Db}$ we define the disagreement sets:
\begin{equation}
    \label{eq:reluunequal}
    \begin{split}
        &\Delta_\mu(\param,\param_0)\coloneqq\mu\left\{x\in\X:f(x;\param)\ne f(x;\param_0)\right\}\\
        &\Delta^{\mathrm{all}}_\mu(\param,\param_0)\coloneqq\mu\left\{x\in\X:x_i(x;\param)\ne x_i(x;\param_0)\ \text{for some }i\in[L]\right\},
    \end{split}
\end{equation}
as the measure of the inputs at which the two networks compute different
outputs. For uniform $\mu$ these are
fractions of the volume of $\X$.  Given thresholds $\epsilon,\delta$, we consider guarantees of the form: for every target $\theta_0$, with probability at least $1-\delta$ over the probes, every network $f(\cdot,\theta)$ agreeing on those probes also agrees with $f(\cdot,\theta_0)$ outside a set of measure at most $\epsilon$. If we moreover assume that $\sup_{\param\in\paramspace, x\in\X}\|f(x;\param)\|\leq B$ for some constant $B$, then we would like to show that one can reconstruct $f_\theta$ up to any $\epsilon$ in $L^s(\mu)$ norm. Note that such boundedness assumption is automatic for the hypothesis class of ReLU networks if we assume the input space $\X$ and the parameter space $\paramspace$ to be compact.

Theorem~\ref{thm:relu-pac-informal} presents a semi-formal version of our results for ReLU networks, while Theorem~\ref{thm:relu-pac} provides the formal statement. Intuitively, this theorem says the following. Since we cannot identify the function $f_\theta$ exactly everywhere from the probes, we want to identify $f_\theta$ from $\all(\param, \mP)$ or $\final(\param, \mP)$ at a large region of the input space $\X$, namely at a set of measure at least $1-\epsilon$ (where $\mu(\X)=1$) by using $N$ many probes. Also assuming that the probes are sampled independently at random under the same probability measure $\mu$, we want the above property to hold with high probability (at least $1-\delta$). Thus, we want the target function $f_{\theta}$ and the reconstructed function $f_{\hat\theta}$ to be approximately equal with high probability. This is similar in spirit (but unrelated) to the PAC learning framework used in classical learning theory. Now, we prove Theorem~\ref{thm:relu-pac} on $\relu$ networks, which is a formal version of Theorem~\ref{thm:relu-pac-informal}. For that we first define some notations and prove a couple of helper lemmas.

\paragraph{Proof of the formal version of Theorem~\ref{thm:relu-pac-informal}:}

We first state and prove some bounds on counting polynomials and the
random-sampling bound used in the proof. The zero-pattern estimate below is a specialization of \citet[Theorem~1.1]{ronyai2001zero}.
We still prove it for completeness.

\begin{lemma}[Counting polynomial tests]
\label{lem:relu-polynomial-counts}
Let $g_1,\cdots,g_J$ be real polynomials of degree at most $r$
in $P$ variables, where $P,r,J\ge1$, and we define
$D\coloneqq\binom{P+r}{r}-1$.
The numbers of distinct vectors
$\left(\mathbf 1[g_j(v)>0]\right)_{j=1}^J,
\,
\left(\mathbf 1[g_j(v)\ne0]\right)_{j=1}^J,
\, v\in\R^P$,
are at most $\sum_{k=0}^D\binom Jk$ and
$\binom{P+rJ}{P}$, respectively.
If $J\ge D$, the first bound is at most $(eJ/D)^D$.
\end{lemma}
\begin{proof}
We first bound the number of positive vectors. For this, we choose one parameter for each realized vector. Subtract from every $g_j$ a positive number smaller
than all its positive values at the chosen parameters.
If none of these values is positive, any positive number works.
This preserves every chosen binary vector and makes all
tested values nonzero.

List the $D$ nonconstant monomials of degree at most $r$.
Each shifted polynomial is an affine function of this list.
It remains to count strict sign vectors of $J$ affine
functions on $\R^D$. A nonempty set with fixed strict
signs is an intersection of open halfspaces and is convex.
Adding one zero hyperplane splits such a set only when
the hyperplane meets it. The intersections of the split
sets with this hyperplane have distinct strict sign
vectors for the preceding functions restricted to
$\R^{D-1}$. A repeated hyperplane causes no split,
and constant functions add no choices.
Thus the maximum count $a(J,D)$ satisfies
\begin{align*}
a(J,D)\le a(J-1,D)+a(J-1,D-1).
\end{align*}
Starting from $a(0,D)=a(J,0)=1$, induction gives
$a(J,D)\le\sum_{k=0}^D\binom Jk$.
For $J\ge D$, we set $t=D/J$ and we conclude by the following bounding:
\begin{align*}
t^D\sum_{k=0}^D\binom Jk
\le\sum_{k=0}^D\binom Jk t^k
\le(1+t)^J\le e^D.
\end{align*}

Now we bound the number of non-zero vectors. Let $\mathcal S$ contain the realized
sets $\{j:g_j(v)\ne0\}$.
We choose a realizing parameter $v_S$ for each
$S\in\mathcal S$, and define $G_S\coloneqq\prod_{j\in S}g_j$,
with empty product $1$.
Then $G_S(v_T)\ne0$ exactly when $S\subseteq T$.
Ordering the sets by nondecreasing cardinality makes the matrix
$(G_S(v_T))_{S,T\in\mathcal S}$ upper triangular
with nonzero diagonal. The polynomials $G_S$ are therefore
linearly independent. Each has degree at most $rJ$,
and the space of such polynomials has
$\binom{P+rJ}{P}$ monomials as a basis.
This proves the second bound.
\end{proof}

For a class $\mathcal C$ of subsets of $\X$, write
$\Pi_{\mathcal C}(M)$ for the largest number of distinct
vectors $(\mathbf 1[p_a\in C])_{a=1}^M$, as $C\in\mathcal C$
varies, on any $M$ fixed inputs. This is the growth function of $\mathcal C$ from VC theory. Moreover, the following lemma is inspired from the classical
double-sampling argument for random $\epsilon$-nets from
\citep{haussler1987epsilon,blumer1989learnability}.

We also use the following consequence of the polynomial
sign bound of \citet[Lemma~17]{bartlett2019nearly}.

\begin{lemma}[Counting polynomial inequalities]
\label{lem:relu-polynomial-signs}
Let $g_1,\cdots,g_J$ be real polynomials of degree at most
$r$ in $P$ variables, where $r\ge1$ and $J\ge P\ge1$.
Then
\begin{align*}
\left|
\left\{
\left(\mathbf 1[g_j(v)>0]\right)_{j=1}^J:
v\in\R^P
\right\}
\right|
\le
\left(\frac{4erJ}{P}\right)^P.
\end{align*}
\end{lemma}
\begin{proof}
Choose one parameter for each realized binary vector.
For each $j$, choose $\eta_j>0$ smaller than every
positive value of $g_j$ at these finitely many parameters;
if there is no positive value, choose any $\eta_j>0$.
Replacing $g_j$ by $g_j-\eta_j$ preserves the chosen
binary vectors and makes every tested value nonzero.
The polynomial sign bound therefore gives at most
\begin{align*}
2\left(\frac{2erJ}{P}\right)^P
\le
\left(\frac{4erJ}{P}\right)^P
\end{align*}
such vectors.
\end{proof}

\begin{lemma}[Random samples detect large disagreement sets]
\label{lem:relu-random-samples}
Let $K,\X$ be compact metric spaces, let $\mu$ be a Borel
probability measure on $\X$, and let
$h:K\times\X\to[0,\infty)$ be continuous.
Set $C_\theta=\{u:h(\theta,u)>0\}$ and
$\mathcal C=\{C_\theta:\theta\in K\}$.
Suppose that $D,Q\ge1$ are integers and
\begin{align*}
\Pi_{\mathcal C}(M)
\le\left(\frac{eMQ}{D}\right)^D,
\, M\ge D.
\end{align*}
For $\epsilon,\delta\in(0,1)$, if
\begin{align*}
N\ge\frac6\epsilon
\left(D\log\frac{12eQ}{\epsilon}+\log\frac2\delta\right),
\end{align*}
then, with probability at least $1-\delta$, every
$C\in\mathcal C$ missed by $N$ i.i.d.\ samples from $\mu$
has $\mu(C)\le\epsilon$.
\end{lemma}
\begin{proof}
Let $E$ be the event that the sample misses some $C$
with $\mu(C)>\epsilon$. Draw an independent second
sample of size $N$. Let $E'$ be the event that some
$C\in\mathcal C$ contains no first-sample point and at
least $N\epsilon/2$ second-sample points.

These events are measurable. To see this, we observe that the map
$\theta\mapsto\mu(C_\theta)$ is lower semicontinuous by
Fatou's lemma. In $K\times\X^N$ or $K\times\X^{2N}$,
the no-hit conditions are closed, whereas the conditions
$\mu(C_\theta)>\epsilon$ and the required number of
second-sample hits are open. The defining sets are
therefore intersections of a closed set with an open
set in a compact metric space. Such sets are countable
unions of compact sets, as are their coordinate
projections onto the samples. Thus $E,E'$ are Borel.

Suppose the first sample belongs to $E$, and choose
a witnessing set of measure $q>\epsilon$.
Its second-sample hit count $Z$ has mean $Nq$ and
variance $Nq(1-q)$. If $N\epsilon\ge8$, by Chebyshev's
inequality we get
\begin{align*}
\Prob\{Z<N\epsilon/2\}
\le\frac{4(1-q)}{Nq}<\frac12.
\end{align*}
Consequently we have $\Prob(E)\le2\Prob(E')$.

Now we pair the two samples and independently swap the members
of each pair with probability $1/2$. Their joint distribution
is unchanged. On the fixed $2N$ points there are at
most $\Pi_{\mathcal C}(2N)$ membership vectors.
For a fixed vector to witness $E'$, all its marked
points must be assigned to the second sample.
If a pair has two marked points this is impossible;
otherwise, with $k$ marked points, its probability is
$2^{-k}$. Since $k\ge N\epsilon/2$, a union bound gives
$\Prob(E)\le2\Pi_{\mathcal C}(2N)2^{-N\epsilon/2}.$

It remains to check the stated threshold.
Put $A=D\log(12eQ/\epsilon)+\log(2/\delta)$ and
$N_0=6A/\epsilon$. Then $N_0>D$, $N_0\epsilon>8$, and
\begin{align*}
D\log\frac{2eN_0Q}{D}+\log\frac2\delta
=A+D\log\frac AD
\le2A=\frac{N_0\epsilon}{3}.
\end{align*}
The function
$N\epsilon-3D\log(2eNQ/D)-3\log(2/\delta)$
has derivative $\epsilon-3D/N\ge0$ for $N\ge N_0$.
Hence, for every such $N$,
\begin{align*}
D\log\frac{2eNQ}{D}+\log\frac2\delta
\le\frac{N\epsilon}{3}.
\end{align*}
Applying the growth bound at $M=2N$ and using
$3\log2/2>1$, we conclude that
\begin{align*}
\Prob(E)
&\le
2\exp\left(
D\log\frac{2eNQ}{D}-\frac{\log2}{2}N\epsilon
\right)
\le\delta.
\end{align*}
\end{proof}

We shall now state and prove Theorem~\ref{thm:relu-pac}, which is the formal version of Theorem~\ref{thm:relu-pac-informal}. Throughout this section, $\X\subseteq\R^d$ is compact, $L\ge2$,
$\sigma=\relu$ in layers $1,\cdots,L-1$ and the output layer is
linear. We use the following notations for ReLU nets, similar to what we used for analytic nets:
\begin{equation}
\label{eq:relu-neural-net}
\begin{split}
x_0(u;\param)
&\coloneqq u,\\
\overline x_i(u;\param)
&\coloneqq(x_i(u;\param);1)\in\R^{d_i+1},
\, 0\le i<L,\\
z_i(u;\param)
&\coloneqq B_i\overline x_{i-1}(u;\param)\in\R^{d_i},
\, i\in[L],\\
x_i(u;\param)
&\coloneqq\relu(z_i(u;\param)),
\, i\in[L-1],\\
x_L(u;\param)
&\coloneqq f(u;\param)\coloneqq z_L(u;\param),\\
H&\coloneqq\sum_{i=1}^{L-1}d_i,
\, \ell\coloneqq H+m.
\end{split}
\end{equation}
We denote by $\beta_{ij}^\top$ the $j$-th row of $B_i$. We call the rows
$(i,j)$ with $i\le L-1$ the hidden rows. $H$ is thus the total number of hidden neurons.

\begin{restatable}[PAC identification of $\relu$ networks]{theorem}{thmrelupac}
\label{thm:relu-pac}
Let  $\X\subset \R^d$ be compact and $\mu$ be a Borel probability measure on $\X$. Consider the family of $\relu$ MLPs $\mlps$ of $\ell$ neurons with widths $\Db\coloneqq(d, d_1, \cdots,d_{L-1}, m)$ with compact parameter space $\paramspace_\Db\subseteq\R^{q_\Db}$. We consider any $\epsilon, \delta \in (0,1)$ and integers:
$$\textstyle
N_1\ge\frac6\epsilon
\left(q_\Db L\log\frac{48eL(\ell-m+1)}{\epsilon}+\log\frac2\delta
\right),\,N_2 \geq \frac{6}{\epsilon}\left(q_\Db\log\frac{24e\ell}{\epsilon}+\log\frac2\delta\right).$$ 

Let $\mP_1=(p_1,\cdots,p_{N_1})\in\X^{N_1}, \mP_2=(p_1,\cdots,p_{N_2})\in\X^{N_2}$ be probes with $p_i$ i.i.d.\ $\mu$. We fix the target $\param_0\in\paramspace_\Db$. Then we have: 

(i) $\Prob_{\mP_1\sim\mu^{N_1}}\left\{\forall\param\in\paramspace_\Db,\,\final(\param,\mP_1) = \final(\param_0, \mP_1)\implies\disagfinal_\mu(\theta,\theta_0)\leq \epsilon\right\}\geq 1 - \delta$, i.e, with probability at least $1-\delta$, any parameter $\param\in\paramspace_\Db$ having the same final output on the probe tuple $\mP_1$ as the target $\param_0$ has disagreement set of measure at most $\epsilon$.

(ii)  $\Prob_{\mP_2\sim\mu^{N_2}}\left\{\forall\param\in\paramspace_\Db,\,\all(\param, \mP_2)=\all(\param_0,\mP_2)\implies\disagall_\mu(\theta,\theta_0)\leq \epsilon\right\}\geq 1 - \delta$  i.e, with probability at least $1-\delta$, any parameter $\param\in\paramspace_\Db$ having the same hidden activations on the probe tuple $\mP_2$ as $\theta_0$ can disagree on any hidden layer neuron on an input set of measure at most $\epsilon$.

Moreover, let $B> 0$ be such that $\sup_{\param\in\paramspace_\Db,\ x\in\X}\|f(x;\param)\|_2\le B$. Then we also have (a) $\Prob_{\mP_1\sim\mu^{N_1}}\left\{\forall\param\in\paramspace_\Db,\,\final(\param, \mP_1)=\final(\param_0,\mP_1)\implies\|f_\param-f_{\param_0}\|_{L^2(\mu)}\leq 2B\sqrt\epsilon\right\}\geq 1 - \delta$ and (b) $\Prob_{\mP_2\sim\mu^{N_2}}\left\{\forall\param\in\paramspace_\Db,\,\all(\param, \mP_2)=\all(\param_0,\mP_2)\implies\|f_\param-f_{\param_0}\|_{L^2(\mu)}\leq 2B\sqrt\epsilon\right\}\geq 1 - \delta$
\end{restatable}

\begin{proof}
Write $p=q_\Db$, and we recall the definition of $x_i, \overline x_i,z_i, H$ and the details of the ReLU net structure from equation~\ref{eq:relu-neural-net}. We use the following notations:
\begin{align*}
P_i&\coloneqq
\sum_{k=1}^i d_k(d_{k-1}+1),\, i\in[L],\,D_{\mathrm{final}}\coloneqq\sum_{i=1}^L P_i,\\
Q_{\mathrm{final}}
&\coloneqq
4\sum_{i=1}^{L-1}i\,d_i+4L,\,D_{\mathrm{all}}\coloneqq p,\,Q_{\mathrm{all}}\coloneqq2\ell.
\end{align*}
Here $P_i$ is the number of parameters in the first $i$
layers, and $P_L=p$.
We fix a parameter $\param_0$ and define
\begin{align*}
C^{\mathrm{final}}_\param=\{u:f(u;\param)\ne f(u;\param_0)\},\,C^{\mathrm{all}}_\param
=\{u:x_i(u;\param)\ne x_i(u;\param_0) \text{ for some }i\in[L]\}.
\end{align*}
Let $\mathcal C^\star$ consist of these sets as
$\param\in\paramspace_\Db$ varies.
We shall show that, for $\star\in\{\mathrm{final},\mathrm{all}\}$, we have 
$\Pi_{\mathcal C^\star}(M)
\le\left(\frac{eMQ_\star}{D_\star}\right)^{D_\star},
\, M\ge D_\star$, where $\Pi_{\mathcal C}(M)$ denotes the largest number of distinct
vectors $(\mathbf 1[p_a\in C])_{a=1}^M$, by varying $C\in\mathcal C$
on any $M$ fixed inputs.

Fix $u\in\X$. By \eqref{eq:relu-neural-net}, equality
with the target at every layer is equivalent to the following 3 conditions: (i) $\beta_{ij}^{\top}\overline x_{i-1}(u;\param_0)=x_{i,j}(u;\param_0),\, \left(i<L,\ x_{i,j}(u;\param_0)>0\right)$, (ii) $\beta_{ij}^{\top}\overline x_{i-1}(u;\param_0)\le0,\,\left(i<L,\ x_{i,j}(u;\param_0)=0\right)$ and (iii) $\beta_{Lj}^{\top}\overline x_{L-1}(u;\param_0)=f_j(u;\param_0),\, \left(j\in[m]\right),$
where the first two conditions range over all hidden rows.
Necessity follows directly from the network equations.
Conversely, starting from the common input $u$, these
conditions imply equality of successive states by induction.

For fixed $\param_0$ and $u$, these are affine conditions
in $\param$. Each equality $g=0$ is determined by the
two tests $g>0$ and $-g>0$, and each inequality $g\le0$
by the test $g>0$. Thus at most $2\ell$ affine tests
per input suffice. On $M$ inputs,
Lemma~\ref{lem:relu-polynomial-counts} with $P=p$ and
$r=1$ gives $\Pi_{\mathcal C^{\mathrm{all}}}(M)
\le\left(\frac{2eM\ell}{p}\right)^p,
\, M\ge p.$

Now consider final-output observations. We count the
hidden activation patterns successively over the layers,
as in \citet[Section~4]{bartlett2019nearly}.
Fix $M\ge D_{\mathrm{final}}$ inputs
$p_1,\cdots,p_M$.
For $i=0,\cdots,L-1$, let $\mathcal S_i$ be the partition
of $\paramspace_\Db$ obtained by grouping parameters
according to the bits
\begin{align*}
\mathbf 1[z_{k,j}(p_a;\param)>0],
\,
k\in[i],\quad j\in[d_k],\quad a\in[M].
\end{align*}
Only nonempty parts are retained, and
$\mathcal S_0=\{\paramspace_\Db\}$.

Fix $i\in[L-1]$ and $S\in\mathcal S_{i-1}$.
On $S$, every preceding ReLU equals multiplication by
its fixed activation bit at each sampled input.
Consequently, each $z_{i,j}(p_a;\param)$ agrees on $S$
with a polynomial of degree at most $i$ in the first
$P_i$ parameters. Indeed, first-layer preactivations
are affine, and each subsequent affine layer increases
the degree by at most one.

There are $Md_i$ such polynomials.
Since $Md_i\ge M\ge D_{\mathrm{final}}\ge P_i$,
Lemma~\ref{lem:relu-polynomial-signs} shows that $S$
is divided into at most
$\left(4eiMd_i/P_i\right)^{P_i}$ parts of
$\mathcal S_i$. Hence
\begin{align*}
|\mathcal S_i|
&\le
|\mathcal S_{i-1}|
\left(\frac{4eiMd_i}{P_i}\right)^{P_i},\\
|\mathcal S_{L-1}|
&\le
\prod_{i=1}^{L-1}
\left(\frac{4eiMd_i}{P_i}\right)^{P_i}.
\end{align*}
Zero preactivations are assigned bit zero, so these
partitions also include parameters at which some
preactivation vanishes.

Fix $S\in\mathcal S_{L-1}$.
For each $a\in[M]$ and $j\in[m]$, let
$f_{a,j}^S(\param)$ be the polynomial obtained by
replacing every hidden ReLU at $p_a$ by multiplication
by its bit on $S$. It has degree at most $L$ and equals
$f_j(p_a;\param)$ for $\param\in S$.
Define
\begin{align*}
R_{a,S}(\param)
\coloneqq
\sum_{j=1}^m
\left(
f_{a,j}^S(\param)-f_j(p_a;\param_0)
\right)^2.
\end{align*}
This polynomial has degree at most $2L$, and, for
$\param\in S$,
$
R_{a,S}(\param)\ne0
\iff
f(p_a;\param)\ne f(p_a;\param_0).
$
The second part of
Lemma~\ref{lem:relu-polynomial-counts} therefore bounds
the number of disagreement vectors on $S$ by
\begin{align*}
\binom{p+2LM}{p}
\le
\left(\frac{e(p+2LM)}p\right)^p
\le
\left(\frac{4eLM}p\right)^p,
\end{align*}
where we used $M\ge p$.

Summing over $S\in\mathcal S_{L-1}$ and applying
the weighted arithmetic--geometric mean inequality,
with weights
$P_1/D_{\mathrm{final}},\cdots,
P_{L-1}/D_{\mathrm{final}},p/D_{\mathrm{final}}$,
gives
\begin{align*}
\Pi_{\mathcal C^{\mathrm{final}}}(M)
&\le
\left[
\prod_{i=1}^{L-1}
\left(\frac{4eiMd_i}{P_i}\right)^{P_i}
\right]
\left(\frac{4eLM}p\right)^p\\
&\le
\left(
\frac{
eM\left(4\sum_{i=1}^{L-1}i\,d_i+4L\right)
}{
\sum_{i=1}^{L-1}P_i+p
}
\right)^{\sum_{i=1}^{L-1}P_i+p}\\
&=
\left(
\frac{eMQ_{\mathrm{final}}}{D_{\mathrm{final}}}
\right)^{D_{\mathrm{final}}}.
\end{align*}
Moreover, we have 
$
D_{\mathrm{final}}\le pL,
\,
Q_{\mathrm{final}}\le4L(H+1).
$

We now apply Lemma~\ref{lem:relu-random-samples}.
Set $K=\paramspace_\Db$ and define:
\begin{align*}
h_{\mathrm{final}}(\param,u)\coloneqq
\|f(u;\param)-f(u;\param_0)\|_2,\,h_{\mathrm{all}}(\param,u)
\coloneqq\max_{i\in[L]}
\|x_i(u;\param)-x_i(u;\param_0)\|_2.
\end{align*}
Both functions are continuous, and their positive
sets are $C^{\mathrm{final}}_\param$ and
$C^{\mathrm{all}}_\param$, respectively.

For part (i), the preceding bounds give
\begin{align*}
\frac6\epsilon
\left(
D_{\mathrm{final}}
\log\frac{12eQ_{\mathrm{final}}}{\epsilon}
+\log\frac2\delta
\right)\,\le
\frac6\epsilon
\left(
pL\log\frac{48eL(H+1)}{\epsilon}
+\log\frac2\delta
\right)
\le N_1.
\end{align*}
Lemma~\ref{lem:relu-random-samples} therefore supplies
an event $\mathcal G_1\subseteq\X^{N_1}$ of probability
at least $1-\delta$ on which every
$C\in\mathcal C^{\mathrm{final}}$ missed by the probes
has measure at most $\epsilon$.
For every $\mP_1\in\mathcal G_1$ and every
$\param\in\paramspace_\Db$,
\begin{align*}
\final(\param,\mP_1)=\final(\param_0,\mP_1)
\implies
C^{\mathrm{final}}_\param
\cap\{p_1,\cdots,p_{N_1}\}=\varnothing\implies
\disagfinal_\mu(\param,\param_0)\le\epsilon.
\end{align*}
This proves (i).

For part (ii), apply the same lemma with
$D=D_{\mathrm{all}}=p$ and
$Q=Q_{\mathrm{all}}=2\ell$.
The assumed bound on $N_2$ is its required threshold.
Thus, on an event $\mathcal G_2\subseteq\X^{N_2}$ of
probability at least $1-\delta$, every
$\param\in\paramspace_\Db$ satisfies
\begin{align*}
\all(\param,\mP_2)=\all(\param_0,\mP_2)\implies
C^{\mathrm{all}}_\param
\cap\{p_1,\cdots,p_{N_2}\}=\varnothing\implies
\disagall_\mu(\param,\param_0)\le\epsilon.
\end{align*}
Since
$C^{\mathrm{final}}_\param\subseteq C^{\mathrm{all}}_\param$,
the final-output disagreement is also at most $\epsilon$.
This proves (ii).

Finally, if the outputs are bounded by $B$, then, for
every $1\le s<\infty$, we have:
\begin{align*}
\|f_\param-f_{\param_0}\|_{L^s(\mu)}^s
&\le(2B)^s\mu(C^{\mathrm{final}}_\param)
\le(2B)^s\epsilon.
\end{align*}
Taking $s=2$ gives the claimed $2B\sqrt{\epsilon}$ bound.
\end{proof}

\subsection{Function-level identification of multi-head attention}
\label{app:proofs:mha}
\begin{wrapfigure}[15]{r}{0.55\textwidth} 
\vspace{-20pt}
\begin{minipage}{\linewidth}
\begin{algorithm}[H]
\caption{\mhaalgo for $h=L=1$}
\label{alg:att-sketch}
\begin{algorithmic}[1]
\Require oracle $F_n(\cdot;\param_0)$; $\rank\W_0=d$
\Ensure effective matrices $(\M_0,\W_0)$
\State query $y_i\gets F_1(e_i^\top;\param_0)$,
       $i=1,\ldots,d$
\vspace{1pt}
\State $\W\gets[y_1;\ldots;y_d]$
       \Comment{Value matrix}
\vspace{1pt}
\State $P\gets[0^\top;e_1^\top;\ldots;e_d^\top]$
\vspace{1pt}
\State query $Y\gets F_{d+1}(P;\param_0)$
\vspace{1pt}
\State $Z\gets P\W$
\vspace{1pt}
\State $A\gets[Y\ \mathbf1_{d+1}]
                  [Z\ \mathbf1_{d+1}]^{-1}$
       \Comment{Attn. wts.}
\vspace{1pt}
\State $\M_{ij}\gets
       \log\!\left(A_{i+1,j+1}/A_{i+1,1}\right)$,
       $i,j\in[d]$
\vspace{1pt}
\State \Return $(\M,\W)$
\end{algorithmic}
\end{algorithm}
\end{minipage}
\vspace{-25pt}
\end{wrapfigure}

We use the architecture, parameter spaces, and probe maps
defined in equations \ref{eq:softmax}--\eqref{eq:mha-stack}.
In particular, $\Theta$ is the raw parameter space and
$D=\dim\Theta$.
The \emph{canonical representation} of a layer is obtained
by merging heads with identical score matrices $\M_a$,
summing their value matrices $\W_a$, and removing groups
whose summed value matrix is zero.

We start the section with a brief discussion on related works and then prove our results.

\paragraph{Relation to query-based attention recovery.}
Query based recovery of attention blocks has been studied recently and we acknowledge the progress of those works and discuss what we contribute. \cite{bhattamishra2026} study recovery of single-head
attention from scalar-output queries, while
\cite{kimkim2026} recover canonical multi-head representations.

Our three-token separation argument is a matrix-output
adaptation of the canonical-identifiability result
\citep[Proposition~3.1]{kimkim2026}.

However, our main focus is generic identification of multilayer MHA stacks
from a shared, nonadaptive probe probes with hidden access.
We also consider different observation models: our queries return complete
tokenwise layer outputs for $\all$, whereas their attention oracles
return a scalar at the final token only.

\paragraph{Softmax attention setup:}Now we setup some detailed notation, some of which overlap with the same from main text, but in more details.

We begin with some definitions. We consider $d$ dimensional inputs and $h$ attention heads, with $h\mid d$. Set $r=d/h$, the value
dimension of each head, and let $k$ be the query/key
dimension with $1\le k\le d$ . A bias-free, unmasked MHA layer
has parameters $
\param\coloneqq\{(\Q_a,\K_a,\V_a,\O_a)\}_{a=1}^h,
\,
\Q_a,\K_a\in\R^{d\times k},\,
\V_a\in\R^{d\times r},\,
\O_a\in\R^{r\times d}.
$
For an input probe $\probeseq\in\R^{n\times d}$ whose rows are the
token vectors $\probe_1^\top,\ldots,\probe_n^\top$, the layer computes
\begin{equation*}
F_n(\probeseq;\param)=\sum_{a=1}^h
    \operatorname{softmax}(\probeseq\M_a\probeseq^\top)\,\probeseq\W_a,\,
\M_a\coloneqq\frac{\Q_a\K_a^\top}{\sqrt{k}},\,\W_a\coloneqq\V_a\O_a.
\end{equation*}

The softmax acts rowwise:$
\left[\operatorname{softmax}(\probeseq\M_a\probeseq^\top)\right]_{ij}=
\frac{\exp(\probe_i^\top\M_a \probe_j)}{\sum_{s=1}^n\exp(\probe_i^\top\M_a \probe_s)}.$
The matrices $\M_a,\W_a\in\R^{d\times d}$ thus have ranks at most $k$ and $r$,
respectively. A probe specifies $\probeseq$ and returns the complete matrix
$F_n(\probeseq;\param)\in\R^{n\times d}$. An $L$-layer MHA stack has parameters
$\boldsymbol{\param}=(\param_1,\ldots,\param_L)$ and we consider the case where the token dimension $d$ is preserved across layers. Layer $\ell$ has $h_\ell$ heads, query/key dimension $k_\ell$, and value dimension
$r_\ell=d/h_\ell$. Our proofs only need manifold doms of parameters.

Let $\Theta=\prod_{\ell=1}^L\Theta_\ell$ denote the parameter space of the $L$-layer MHA architecture, with $\Theta_\ell$ being the parameter space of the $\ell$-th layer. Let $\paramdim_\ell\coloneqq h_\ell\left(
k_\ell(2d-k_\ell)+r_\ell(2d-r_\ell)
\right),\,\paramdim\coloneqq\sum_{\ell=1}^L \paramdim_\ell
$. Then $\paramdim_\ell=\dim(\Theta_\ell)$ since the set of $d\times d$ matrices of rank at most $k$ is an irreducible algebraic variety of dimension $k(2d-k)$.  Thus, $\dim(\Theta_\ell)= \paramdim_\ell,\,\dim(\Theta)=q$.

\paragraph{Proof of Stronger version of Theorem~\ref{thm:att-fixed-informal}}
We use two elementary facts: a discrete subset of a Euclidean
space is countable, and a nonzero analytic function on a connected
open subset of $\R^m$ has a zero set of Hausdorff dimension at
most $m-1$. 


\begin{lemma}[Three-token separation]
\label{lem:att-three-separation}
Two MHA layers agree on every three-token input if and only if they agree
at every sequence length.
\end{lemma}
\begin{proof}
Define $\rho:\mathbb{R}\to\mathbb{R}:\rho(t)=e^t/(2+e^t)$. It is not hard to see that the functions $\rho(gt)$ for distinct ``rates'' $g\in\R$ are linearly independent.

Now, suppose the three-token maps agree. Equal-token inputs give
equal value sums. List their distinct score matrices as $B_j$,
and let $\Delta_j$ be the difference between the associated
value sums. The repeated-token row of $(q+tv,q,q)$ then gives
$\sum_j\rho(tq^\top B_jv)\,v^\top\Delta_j=0.$
For generic $(q,v)$ the rates are distinct, since
$q^\top(B_j-B_{j'})v$ is a nonzero polynomial for $j\ne j'$.
Linear independence gives $v^\top\Delta_j=0$, hence $\Delta_j=0$.
The converse, at every length, follows by merging the summands.
\end{proof}

We make a note that the above is the matrix-output counterpart of
\citet[Proposition~3.1]{kimkim2026}.

\begin{lemma}[Countably many candidates from a fixed probe batch]
\label{lem:att-countable-batch}
Let $H\ge2$, $H\mid d$, and $s=d/H$. Consider probe tuples
containing $d(2H+1)$ probes: $d$ singletons,
$d(H+1)$ triples, and $d(H-1)$ sequences of length $s+1$. Let $\mathbf{P}_0$ be the set of all such probes.

Fix an MHA layer with $h\le H$ heads, $h\mid d$,
query/key dimension $1\le k\le d$, and value dimension $r=d/h$.
Its effective parameter space is
\[
\Theta_{\mathrm{eff}}
\coloneqq
\left\{
\param_{\mathrm{eff}}
=((\M_a,\W_a))_{a=1}^h:
\operatorname{rank}\M_a\le k,\,
\operatorname{rank}\W_a\le r
\right\},
\]
with
\[
\paramdim
\coloneqq\dim\Theta_{\mathrm{eff}}
=h\bigl[k(2d-k)+r(2d-r)\bigr].
\]

Then, for almost all tuples $(\mP, \param)\in\mathbf{P}_0\times\Theta_{\mathrm{eff}}\setminus \mathcal N$ (with $
 \dim_{\mathrm H}\mathcal N
 \le \dim\mathbf P_0+\paramdim-1$) and almost all $\param\in\param_{\mathrm{eff}}$ the set
$
\left\{
\param'\in\Theta_{\mathrm{eff}}:
\final(\param',\mP)
=
\final(\param,\mP)
\right\}
$
is at most countable.
\end{lemma}
\begin{proof}
We first choose parameters and probes for which no nonzero
first-order change in the parameters leaves all outputs unchanged.
We then show that this property holds almost everywhere and
implies the claimed countability.

Let $V_{ij}=\alpha_i^{j-1},
 \, 0<\alpha_1<\cdots<\alpha_d,$
and let $P_a$ keep the $a$th consecutive coordinate block of
size $r$ and set the other coordinates to zero. Choose
 $\W_a=VP_a,\,
 \M_a=c_aM_*,
 \,
 M_*=e_1\mathbf1^\top+\sum_{j=2}^{k}e_je_j^\top,$
where $0<c_1<\cdots<c_h$.
Partition $[d]$ into $H$ blocks $J$ of size $s$, and set
$u_j=V_{:,j}$ for $j\le d-s$.
Use the probes $e_i,\,(e_1+t_\nu e_i,e_1,e_1),\,(u_j,(u_j+e_b)_{b\in J}),$
for $i\in[d]$, $\nu\in[H+1]$, and $j\in[d-s]$,
stacking the listed tokens as rows.
Their counts are exactly those specified by $\mathbf P_0$.

We will repeatedly use the following elementary fact:
a nonzero polynomial of degree below $q$ cannot vanish at
$q$ distinct points. Consequently, the matrix
$(\alpha_i^{j-1})_{i\in I,\,j\in[q]}$ is invertible whenever
$|I|=q$.

We start with showing that the singleton and triple outputs constrain the value matrices.
Write $\rho(t)=e^t/(2+e^t)$.
Singletons determine $S=\sum_a\W_a$.
For a fixed direction $e_i$, subtract $e_1^\top S$ from a
repeated output row of the triple and divide by $t_\nu$.
The result is
\[
 Y_\nu=\sum_a\rho(t_\nu g_a)b_a^\top,
 \,
 g_a=e_1^\top\M_ae_i=c_a,
 \,
 b_a=\W_a^\top e_i.
\]
We define:
\[
 C_{\nu a}=\rho(t_\nu c_a),\,
 B=(b_a^\top)_{a=1}^h,\,
 (D_a)_\nu=t_\nu\rho'(t_\nu c_a).
\]
Each row of $B$ is nonzero and occupies a different coordinate
block, so its rows are independent.

Choose the positive scales so that
$\operatorname{rank}[C\ D_a]=h+1$ for every $a$.
Here is a real-variable justification that such scales exist.
For $c>0$, as $t\to\infty$, we get $\rho(ct)-1=-2e^{-ct}+O(e^{-2ct}),\,t\rho'(ct)=2te^{-ct}+O(te^{-2ct}).$
Suppose
$\sum_b\beta_b\rho(c_bt)+\gamma t\rho'(c_at)=0$.
Taking the limit gives $\sum_b\beta_b=0$.
After subtracting this constant, consider the smallest rate
with a nonzero coefficient. Dividing by its exponential leaves
a nonzero constant or a nonzero linear function of $t$, plus
a term tending to zero. This is impossible.
Thus the $h+1$ functions are independent.
Their evaluation matrices therefore have full rank for almost
every choice of the scales, simultaneously for all $a$.

Now write $\delta\M_a,\delta\W_a$ for first-order parameter
changes, and suppose every observed output has zero
first-order change. Then $\delta S=0,\,C\delta B+\sum_aD_a\,\delta g_a\,b_a^\top=0.$
Taking row combinations that cancel $C$, and multiplying by
a right inverse of $B$, gives $\delta g_a=0$ for every $a$,
because $D_a$ is not a linear combination of the columns of $C$.
Then $C\delta B=0$, so $\delta B=0$.
Doing this for every $i$ gives $\delta\W_a=0$.

Now we show that the remaining outputs constrain the score matrices corresponding to the softmax attention.
Since $\W_bP_a=\mathbf1_{a=b}\W_a$, multiplying the output
on the right by $P_a$ isolates head $a$.
For a block input $Z$, we therefore obtain $\delta A_a\,Z\W_a=0,\, A_a=\operatorname{softmax}(Z\M_aZ^\top).$
Every $s$ selected rows of $\W_a$ are independent.
Indeed, after removing nonzero row factors, their first
$s$ nonzero columns form the polynomial-evaluation matrix
described above. Hence the rows of $Z\W_a$ are affinely
independent. Since also $\delta A_a\mathbf1=0$, it follows
that $\delta A_a=0$.

Differentiating the attention log-ratios gives
\[
 u_j^\top\delta\M_ae_b=0,\,
 e_{b'}^\top\delta\M_ae_b=0
 \,(j\le d-s,\ b,b'\in J).
\]
For each $J$, the vectors $(u_j)_{j\le d-s}$ and
$(e_b)_{b\in J}$ span $\R^d$:
deleting rows $J$ from the first $d-s$ columns of $V$
leaves an invertible polynomial-evaluation matrix.
Thus $\delta\M_ae_b=0$ for every $b$ and every $a$.
We have proved that all first-order parameter changes vanish.

Now we use this example to conclude the result for almost every parameter and probe tuple.
A rank-$t$ matrix can locally be written as
\[
 \begin{pmatrix}
 U_{11}&U_{12}\\
 U_{21}&U_{21}U_{11}^{-1}U_{12}
 \end{pmatrix},
 \, \det U_{11}\ne0.
\]
Its three displayed blocks provide $t(2d-t)$ independent
coordinates. Thus the effective parameters with ranks
exactly $k$ and $r$ have $\paramdim$ independent coordinates;
smaller ranks require fewer coordinates.

The previous discussion shows that the matrix of first
derivatives of the outputs has rank $\paramdim$ at our
chosen parameters and probes.
The same is true when differentiating with respect to the
raw factors: full-rank factors allow every first-order change
in the independent matrix coordinates.
Some square submatrix of these derivatives therefore has a
nonzero determinant. This determinant is analytic on the
connected raw parameter and probe space, so its zero set
has measure zero.

Passing to the independent matrix coordinates preserves this
almost-everywhere conclusion. We note that full-rank factorizations
are locally described by these coordinates together with an
arbitrary invertible change of basis in the factors.
Fubini applied to these extra coordinates gives the assertion.

Now, finally we bound the exceptional set.
Record both the probes and their outputs:
\[
 \mathcal G(\mP,\param)
 =\bigl(\mP,\final(\param,\mP)\bigr),
 \,
 m=\dim\mathbf P_0+\paramdim.
\]
Let $\mathcal B$ be the union of the pairs for which a
parameter matrix has smaller rank than allowed and the
remaining pairs where the matrix of first derivatives of
$\mathcal G$ has rank below $m$.

The smaller-rank matrices require fewer independent
coordinates. For matrices of the full allowed ranks,
the derivative condition fails on proper analytic zero sets,
by the preceding construction. Hence
$
 \dim_{\mathrm H}\mathcal B\le m-1.
$
Here we use the elementary fact that the zero set of a
nonzero analytic function on an open subset of $\R^m$
has Hausdorff dimension at most $m-1$.

We must also exclude parameters whose observations match
those of a point in $\mathcal B$. Define $\mathcal N\coloneqq\mathcal G^{-1}\bigl(\mathcal G(\mathcal B)\bigr).$
Smooth maps do not increase Hausdorff dimension. Therefore we have: $\dim_{\mathrm H}\mathcal G(\mathcal B)\le m-1$.
Outside $\mathcal B$, select $m$ coordinates of $\mathcal G$
whose derivative matrix is invertible.
The inverse function theorem gives a smooth local inverse.
Applying this inverse to the corresponding coordinates
of $\mathcal G(\mathcal B)$ gives
$\dim_{\mathrm H}(\mathcal N\setminus\mathcal B)\le \dim_{\mathrm H}\mathcal G(\mathcal B)\le m-1.$
A countable collection of these neighborhoods suffices.
Since $\mathcal B\subseteq\mathcal N$, we conclude that $\dim_{\mathrm H}\mathcal N\le m-1.$

The set $\mathcal B$ is closed: having the full allowed
matrix ranks and an invertible selected derivative matrix
are open conditions.
Thus $\mathcal B$ is a countable union of compact sets.
Its continuous image $\mathcal G(\mathcal B)$ is likewise
such a union, so $\mathcal N$ is measurable.

If $(\mP,\param)\notin\mathcal N$, every parameter with the
same observations lies outside $\mathcal B$.
The inverse function theorem therefore isolates each
matching parameter from all other matches.
There are at most countably many such isolated points.
Finally, Fubini in the independent matrix coordinates
gives the almost-everywhere conclusion in the stated order.
\end{proof}

\begin{lemma}[Prefixes are generically locally invertible]
\label{lem:att-prefix-open}
For each fixed length $n$ and prefix depth $\ell$, the input
Jacobian of $x_{\ell,n}$ is nonsingular almost everywhere in
the raw prefix parameters and input. Consequently, for almost
every prefix, its image contains a nonempty open set.
\end{lemma}
\begin{proof}
It suffices to construct one nonsingular Jacobian.
In each head take $\M_a=\varepsilon e_1e_1^\top$ and let
$\W_a$ be coordinate-block projections summing to $I_d$.
Writing $X=(z,X')$, the layer is
\[
 (z,X')\longmapsto(f_\varepsilon(z),A_\varepsilon(z)X'),
 \, A_\varepsilon=\operatorname{softmax}(\varepsilon zz^\top),
 \, f_\varepsilon=A_\varepsilon z.
\]
Choose distinct $z_i$. For $\varepsilon>0$, the matrix
$K_{ij}=e^{\varepsilon z_i z_j}$ is positive definite, because
\[
 c^\top Kc=\sum_{m\ge0}\frac{\varepsilon^m}{m!}
                 \left(\sum_i c_i z_i^m\right)^2>0
 \,(c\ne0)
\]
by the Vandermonde determinant. Thus $A_\varepsilon$ is invertible.
For $n\ge2$, put $w=z-\bar z\mathbf1$ and $v=\|w\|^2/n>0$.
Expansion at $\varepsilon=0$ gives
\[
 Df_\varepsilon
 =\frac{\mathbf1\mathbf1^\top}{n}
  +\varepsilon\left(vI+\frac{2zw^\top}{n}\right)
  +O(\varepsilon^2).
\]
Its first-order block on $\mathbf1^\perp$ is
$vI+2ww^\top/n\succ0$, so $Df_\varepsilon$ is invertible
for sufficiently small $\varepsilon>0$.
The full input Jacobian is block triangular with invertible
diagonal blocks $Df_\varepsilon$ and $d-1$ copies of $A_\varepsilon$.

The output first coordinates remain distinct: the function
$q\mapsto\sum_jz_je^{\varepsilon qz_j}/\sum_je^{\varepsilon qz_j}$
has derivative $\varepsilon$ times a positive weighted variance.
We can therefore repeat the construction through each layer.
For $n=1$ the witness is the identity.
Each prefix determinant is consequently a nonzero analytic
function. Its zero set is null; Fubini and the inverse function
theorem give the conclusion.
\end{proof}

\paragraph{The geometric strengthening.}
The following standard result is the only additional external
ingredient used to obtain the open-probes and closed-exception
claims. We state it explicitly rather than treating those
claims as consequences of analyticity.

\begin{theorem}[Real-exponential geometry; external input]
\label{thm:att-exp-geometry}
Call a set \emph{definable} if it can be described by a finite
formula using real constants, arithmetic, exponentials,
equalities, inequalities and quantifiers over real variables.
Such sets have integer Hausdorff dimension, satisfy
\[
 \dim_{\mathrm H}\overline A=\dim_{\mathrm H}A,
 \,
 A\subseteq\R^m\text{ is null}
 \ \Longleftrightarrow\ 
 \dim_{\mathrm H}A<m,
\]
and, for a definable family $B\subseteq\R^s\times\R^D$,
each set $\{u:\dim_{\mathrm H}B_u=j\}$ is definable.
Here $\dim\varnothing=-\infty$.
\end{theorem}

This is a consequence of Wilkie's o-minimality theorem for
$\R_{\exp}$ and its dimension theory
\citep{wilkie2009ominimal};
we do not reprove that foundational theorem.

\begin{restatable}[Structured probing for softmax attention]
{theorem}{thmattfixed}
\label{thm:att-fixed}
Define $H\coloneqq\max_\ell h_\ell$ and $N=d(2H+1)+1$.
There exist lengths $n_j\le\max\{3,1+d/H\}$ and an open,
dense, full-measure set
$\mathcal U\subseteq\mathcal S=\prod_{j=1}^N\R^{n_j\times d}$
such that every $\mP\in\mathcal U$ admits a closed exceptional
set $E_{\mP}\subseteq\Theta$, with
$\dim_{\mathrm H}E_{\mP}\le D-1$, satisfying
\[
 \all(\bparam;\mP)=\all(\bparam_0;\mP)
 \,\to\,
 x_{\ell,n}(\cdot;\bparam)
 \equiv x_{\ell,n}(\cdot;\bparam_0)
 \,(\ell\in[L],\ n\ge1)
\]
for every $\bparam_0\notin E_{\mP}$ and every $\bparam\in\Theta$.
The probe set $\mP$ is fixed independently of the target.
\end{restatable}
\begin{proof}
 We first state and prove an equivalent result. That with some more argument proves Theorem~\ref{thm:att-fixed}. 

 \begin{proposition}[Fixed-probes identification]
\label{prop:att-fixed-ae}
Put $H=\max_\ell h_\ell$ and $N=d(2H+1)+1$.
There exist lengths $n_j\le\max\{3,1+d/H\}$ such that,
for almost every probe tuples 
$\mP\in\mathcal S=\prod_{j=1}^N\R^{n_j\times d}$ and
almost every target $\bparam_0\in\Theta$, every competitor
with the same complete hidden transcript has the same canonical
representation in every layer. Consequently,
\[
 \all(\bparam;\mP)=\all(\bparam_0;\mP)
 \ \to\ 
 x_{\ell,n}(\cdot;\bparam)
 \equiv x_{\ell,n}(\cdot;\bparam_0)
 \,(\ell\in[L],\ n\ge1).
\]
\end{proposition}
\begin{proof}
Assume first $H\ge2$. Take the length profile of
Lemma~\ref{lem:att-countable-batch} and append one triple. Thus
\[
 N=d+d(H+1)+d(H-1)+1=d(2H+1)+1.
\]

We first use the hidden transcript to restrict each layer.
Let $G_\ell(\zeta;\bparam_{<\ell})$ be the first batch after
the target prefix, and let $\pi_\ell$ map raw factors to
effective parameters. The map
\[
 (\zeta,\bparam)\longmapsto
 \bigl(G_\ell(\zeta;\bparam_{<\ell}),\pi_\ell(\param_\ell)\bigr)
\]
is a submersion almost everywhere: its probes derivative is
invertible by Lemma~\ref{lem:att-prefix-open}, and its
independent current-layer factor map is a submersion on the
full-rank stratum. Pulling back the exceptional set in
Lemma~\ref{lem:att-countable-batch} therefore gives a null set.
For almost every first batch and target, each layer has only
countably many consistent effective candidates. This includes
every matching competitor: its hidden transcript fixes the
same layer inputs as those observed from the target.

Then we use the last triple to separate these candidates.
Fix such a first batch and target, with open prefix images at
length three. For a candidate $\eta$ at layer $\ell$ not
canonically equivalent to the target, the discrepancy
\[
 X\longmapsto
 F_3^{(\ell)}(x_{\ell-1,3}(X;\bparam_0);\eta)
 -x_{\ell,3}(X;\bparam_0)
\]
is nonzero and analytic. Otherwise the standalone layer maps
would agree on an open set, hence everywhere, contradicting
Lemma~\ref{lem:att-three-separation}.
Almost every last triple avoids the zero sets of all these
countably many discrepancies. Every matching stack then has the
target's canonical layers and agrees at every length.

If $H=1$, instead use $d$ singletons, one $(d+1)$-token input,
and $2d$ additional triples. The singleton states satisfy
$C_\ell=C_{\ell-1}\W_\ell$, so generically they determine
every $\W_\ell$. By Lemma~\ref{lem:att-prefix-open}, the
prefix image cannot lie in $\det[\mathbf1\ Z]=0$.
Analyticity therefore makes the propagated $(d+1)$-token input
$Z$ affinely independent at every depth, almost everywhere.
From its layer output $Y$ recover
\[
 A=[\mathbf1\ \ Y\W_\ell^{-1}]
       [\mathbf1\ \ Z]^{-1}.
\]
The equations
$\log(A_{ij}/A_{i1})=z_i^\top\M_\ell(z_j-z_1)$ determine
$\M_\ell$: the differences $z_j-z_1$, $j=2,\ldots,d+1$,
form a basis and the tokens span $\R^d$.

Finally we show that the failure set is measurable.
At each layer, give the competitor's heads positive signs and
the target's heads negative signs. Canonical agreement means that these
signed heads can be partitioned into groups with equal score
matrices and zero signed value sums. There are finitely many
partitions, each imposing polynomial equalities. Canonical
agreement is therefore closed.
The incidence set of matching transcripts but canonical
disagreement is closed intersected with open, hence a countable
union of compact sets. Its projection
\[
 B=\{(\mP,\bparam_0):\exists\bparam, \all(\bparam;\mP)=\all(\bparam_0;\mP), \bparam\not\sim_{\rm can}\bparam_0\}
\]
is likewise a countable union of compact sets and is measurable.
The preceding argument and Fubini show that $B$ is null.
A second application of Fubini gives the stated quantifier order.
\end{proof}

Take the length profile of Proposition~\ref{prop:att-fixed-ae}
and its null failure set $B$.
The set $B$ is definable: finite hidden transcripts involve
arithmetic and exponentials, and canonical agreement is given
by the finite polynomial description above.

By Fubini, almost every fibre $B_{\mP}$ is null. Therefore
$A=\{\mP\in\mathcal S:
          \dim_{\mathrm H}B_{\mP}=D\}$
is null and definable by Theorem~\ref{thm:att-exp-geometry}.
We define
\[
 \mathcal U=\mathcal S\setminus\overline A,
 \,
 E_{\mP}=\overline{B_{\mP}}
 \,(\mP\in\mathcal U).
\]
The same theorem makes $\mathcal U$ open, dense and full measure,
and $E_{\mP}$ closed with Hausdorff dimension at most $D-1$.
Outside $E_{\mP}$ there is no canonically different competitor
with the same transcript. Canonical agreement implies equality
of all layer maps, and thus of all composed functions at every
length.
\end{proof}

\paragraph{Multilayer MHA reconstruction from a fixed set of probes} The observations already supply each layer's training inputs:
they are the preceding \emph{target} states, not states generated
by the reconstructed prefix. Each layer can consequently be
fitted independently. Define its cached-data loss by
\begin{equation}
\label{eq:att-fixed-fit}
 \mathcal L_\ell(\param_\ell)
 =\sum_{j=1}^N
 \left\|F_{n_j}^{(\ell)}(Z_\ell^{(j)};\param_\ell)
              -Y_\ell^{(j)}\right\|_F^2.
\end{equation}
Optimization is over the specified raw factors, so the prescribed
head counts and rank bounds are maintained automatically.

\begin{algorithm}[H]
\caption{Fixed-probe set multilayer reconstruction}
\label{alg:att-fixed-decoder}
\begin{algorithmic}[1]
\Require Architecture; probe set $\mP=(X^{(j)})_{j=1}^N\in\mathcal U$;
complete hidden access to a target $\bparam_0\notin E_{\mP}$
\Ensure A stack with the target's composed hidden functions at every length
\For{$j=1,\ldots,N$}
  \State Query $X^{(j)}$ once; record
  $H_0^{(j)}\gets X^{(j)}$ and all returned states
  $H_1^{(j)},\ldots,H_L^{(j)}$
\EndFor
\For{$\ell=1,\ldots,L$} \Comment{These fits may run in parallel}
  \State Set $Z_\ell^{(j)}\gets H_{\ell-1}^{(j)}$ and
  $Y_\ell^{(j)}\gets H_\ell^{(j)}$ for every $j$
  \State Find any $\widehat\param_\ell$ in the layer's raw
  parameter space with $\mathcal L_\ell(\widehat\param_\ell)=0$
  \Comment{Fit nonlinear solver}
\EndFor
\State \Return $\widehat\bparam=(\widehat\param_1,\ldots,\widehat\param_L)$
\end{algorithmic}
\end{algorithm}

\paragraph{Correctness and scope.}
Each feasibility problem has a solution, the true layer.
Induction over $\ell$ shows that the returned stack reproduces
every recorded state. Theorem~\ref{thm:att-fixed} therefore gives
the all-length conclusion. 
In numerical implementations it is replaced by layerwise
least-squares optimization. It should be mentioned that neither its global convergence nor
an OOD error bound from a nonzero residual follows from the theorem. The theorem provides a guarantee, the implementation requires implementing a solver or independently fitting the layers.
No head alignment \emph{between layers} is required.

\paragraph{The special case of one layer:} The above statement in Theorem~\ref{thm:att-fixed} discusses identification of generic targets from probes in the multilayer case. However, we show in the following lemma that for a single layer, we can \emph{exactly identify every target} at a function level.

\begin{restatable}[One layer: every target from fixed probes]
{lemma}{lemattuniversal}
\label{lemma:att-universal}
Let $L=1$ and $N=4dh+1$. There exist an integer $n_*=n_*(d,h)$
and a fixed probe set $\mP\in(\R^{n_*\times d})^N$ such that,
simultaneously for all $\param_1,\param_2\in\Theta$, we have:
\begin{align*}
     \final(\param_1;\mP)=\final(\param_2;\mP)
 \,\iff\,
 \left[F_n(\cdot;\param_1)\equiv F_n(\cdot;\param_2)
       \text{ for every }n\ge1\right].
\end{align*}
\end{restatable}

\begin{proof}
Put $K=(d-1)(2h-1)+1$ and
$q_t=(1,t,\ldots,t^{d-1})^\top$ for $t=1,\ldots,K$,
and let $\mathcal D$ contain the distinct vectors among
$0,e_1,\ldots,e_d,q_1,\ldots,q_K$.
For positive integer multiplicities $m=(m_x)_{x\in\mathcal D}$,
let $X(m)$ contain $m_x$ copies of each $x$.
At a copy of $x_0$, output coordinate $c$ equals
\begin{equation}
\label{eq:att-mult}
 R(m;\xi)=\sum_{a=1}^h
 \frac{\sum_{x\in\mathcal D}m_x\tau_a^x x^\top w_a}
      {\sum_{x\in\mathcal D}m_x\tau_a^x},
 \,
 \tau_a=\exp(\M_a^\top x_0),\, w_a=\W_ae_c,
\end{equation}
where $\tau^x=\prod_i\tau_i^{x_i}$ and
$\xi=((\tau_a,w_a))_a\in\Xi=((0,\infty)^d\times\R^d)^h$.

Clearing the positive denominators of
$R(m;\xi)-R(m;\xi')$ gives a polynomial $P(m;\xi,\xi')$,
since every $x\in\mathcal D$ has nonnegative integer coordinates.
The Finite Witness Theorem \citep{amir2023neural}, applied to the $4dh$-dimensional pair $(\xi,\xi')$, implies that almost every choice of $N=4dh+1$ generic positive multiplicity vectors detects every nonzero discrepancy.

The bad witness set is semialgebraic:
writing $P=\sum_\alpha c_\alpha(\xi,\xi')m^\alpha$, it is the
projection of the set defined by $\sum_\alpha c_\alpha^2>0$
and $P(m^{(j)};\xi,\xi')=0$ for every $j$.
Its nullity implies that its closure has empty interior.
We choose a good rational tuple. Since
$R(\lambda m;\xi)=R(m;\xi)$, normalize each vector to sum one
and clear all denominators by a common integer. This gives
positive integer multiplicities of a common sum $n_*$.
The resulting probe set works for every row and coordinate, because
their parameters all belong to $\Xi$.

Suppose two layers agree on this set of probes. Let $B_1,\ldots,B_J$,
$J\le2h$, be their distinct score matrices, and let $\Delta_j$
be the difference of their summed value matrices at $B_j$.
For each $j$, some $q_t$ satisfies
$q_t^\top B_j\ne q_t^\top B_l$ for all $l\ne j$:
any $d$ moment-curve vectors are independent, so each proper
subspace $\{q:q^\top(B_j-B_l)=0\}$ contains at most $d-1$
of them, fewer than $K$ altogether.

Fix this $x_0=q_t$ and a coordinate $c$. The witness property gives
\begin{align*}
    \sum_{l=1}^J\frac{A_l(m)}{D_l(m)}\equiv0, \,
 D_l=\sum_xm_xe^{x_0^\top B_lx},\,
 A_l=\sum_xm_xe^{x_0^\top B_lx}x^\top\Delta_le_c.
\end{align*}
Each denominator has coefficient one at $m_0$, and its
coefficients at $m_{e_i}$ determine $x_0^\top B_l$.
After merging equal denominators, $D_j$ is a distinct linear
factor. Clearing denominators shows that this prime factor
divides $A_j$, so $A_j=\lambda D_j$.
The coefficient at $m_0$ gives $\lambda=0$; the coefficients
at $m_{e_i}$ then give $(\Delta_j)_{ic}=0$.
Thus all $\Delta_j=0$, proving 
functional agreement at every length.
\end{proof}

\subsection{Learning Functionals and Operators from Probes}
\label{app:proofs:approx}

We first prove a common approximation lemma. We then give formal
versions of Theorems~\ref{thm:functional-informal}
and~\ref{thm:hidden-functional-informal}, in that order.
All function spaces below carry the uniform norm. Matrix inputs,
outputs, and observations are identified with vectors by listing
their entries in a fixed order. All MLP readouts have biases,
an affine output layer, and a fixed continuous, non-polynomial
hidden activation $\tau$.

\paragraph{Separation and approximation.}
For a compact parameter set $K$ and a jointly continuous family
$f:\X\times K\to\RR^m$ on compact $\X$, write
\begin{align*}
\mathcal F_K\coloneqq\{f_\param=f(\cdot;\param):\param\in K\}.
\end{align*}
A map $R:K\to\RR^s$ separates the represented functions if
\begin{align*}
R(\param)=R(\param')
\,\implies\, f_\param=f_{\param'},
\, \param,\param'\in K.
\end{align*}
Note that equality of functions need not imply equality of $R$, because $R$ acts on weights.

\begin{lemma}[Separation implies universal approximation]
\label{lem:separation-universality}
Let $\X\subseteq\RR^d$ and $K\subseteq\RR^q$ be compact,
let $f:\X\times K\to\RR^m$ be continuous, and let
$R:K\to\RR^s$ be continuous and separate the represented functions.
Then $\mathcal F_K$ is compact and there is a continuous map
\begin{align*}
g:R(K)\to\mathcal F_K,
\, g(R(\param))=f_\param.
\end{align*}
For every continuous $\Lambda:\mathcal F_K\to\RR^k$,
every compact $\mathcal Y\subseteq\RR^a$, every continuous
$\mathcal T:\mathcal F_K\to C(\mathcal Y,\RR^k)$, and every
$\epsilon>0$, there are MLPs
$\phi:\RR^s\to\RR^k$ and $\psi:\RR^{s+a}\to\RR^k$ such that
\begin{align*}
\sup_{\param\in K}
\|\Lambda(f_\param)-\phi(R(\param))\|_2<\epsilon,\,\sup_{\param\in K}\sup_{y\in\mathcal Y}
\|\mathcal T(f_\param)(y)-\psi(R(\param),y)\|_2<\epsilon.
\end{align*}
\end{lemma}

\begin{proof}
Uniform continuity of $f$ on $\X\times K$ implies continuity
of the realization map $\real:K\to C(\X,\RR^m),\, \real(\param)=f_\param,$
in the uniform norm. Thus $\mathcal F_K=\real(K)$ is compact.
Put $S=R(K)$ and define $g(R(\param))=f_\param$.
Separation makes $g$ well defined. For every closed
$C\subseteq\mathcal F_K$,
\begin{align*}
g^{-1}(C)=R\left(\real^{-1}(C)\right)
\end{align*}
is compact and therefore closed in $S$.
Hence $g$ is continuous.

The map $\Lambda\circ g$ is continuous on the compact set $S$.
The universal approximation theorem
\cite{Leshno1993MultilayerFN} gives an MLP $\phi$ such that $\sup_{s\in S}\|\Lambda(g(s))-\phi(s)\|_2<\epsilon.$
Substituting $s=R(\param)$ proves the functional conclusion.
The compact-set, vector-valued form follows by applying the
Tietze extension theorem and the scalar approximation theorem
coordinatewise.

For the operator result, suppose $G(s,y)\coloneqq\mathcal T(g(s))(y),\, (s,y)\in S\times\mathcal Y.$
If $(s_j,y_j)\to(s,y)$, then we have 
\begin{align*}
\|G(s_j,y_j)-G(s,y)\|_2
&\le\|\mathcal T(g(s_j))-\mathcal T(g(s))\|_\infty\\
&\,+
\|\mathcal T(g(s))(y_j)-\mathcal T(g(s))(y)\|_2
\to0.
\end{align*}
Thus $G$ is continuous on the compact set $S\times\mathcal Y$.
We thus apply the same approximation theorem to $G$, and then
substitute $s=R(\param)$.
\end{proof}

\paragraph{Final-output observations.}
The following Theorem~\ref{thm:functional-formal} is the formal version of
Theorem~\ref{thm:functional-informal}.

\begin{theorem}[Approximation from final-output probes]
\label{thm:functional-formal}
Let $f:\RR^d\times\RR^q\to\RR^m$ be real analytic,
let $\X\subseteq\RR^d$ be compact with nonempty interior,
and let $\paramspace\subseteq\RR^q$ be compact.
For every $N\ge2q+1$, and for almost every
$\mP=(p_1,\cdots,p_N)\in\X^N$,
the map
\begin{align*}
T_\mP:\mathcal F_{\paramspace}\to\RR^{mN},
\, T_\mP(g)=(g(p_1),\cdots,g(p_N)),
\end{align*}
is one-to-one and has a continuous inverse on its image.

Then, for every continuous
$\Lambda:\mathcal F_{\paramspace}\to\RR^k$, every compact
$\mathcal Y\subseteq\RR^a$, every continuous
$\mathcal T:\mathcal F_{\paramspace}\to C(\mathcal Y,\RR^k)$,
and every $\epsilon>0$, there are MLPs
$\phi:\RR^{mN}\to\RR^k$ and $\psi:\RR^{mN+a}\to\RR^k$ satisfying
\begin{align*}
\sup_{\param\in\paramspace}
\|\Lambda(f_\param)-\phi(\final(\param,\mP))\|_2<\epsilon,\,\sup_{\param\in\paramspace}\sup_{y\in\mathcal Y}
\|\mathcal T(f_\param)(y)-\psi(\final(\param,\mP),y)\|_2<\epsilon.
\end{align*}
\end{theorem}

\begin{proof}
Since $f$ is globally analytic, we may apply
Theorem~\ref{thm:finalprobes} on sufficiently large input and
parameter balls containing $\X$ and $\paramspace$ in their
interiors. Restricting the parameters to $\paramspace$ and
intersecting the exceptional probe set with $\X^{2q+1}$
gives a Lebesgue-null $\mathcal N_0\subseteq\X^{2q+1}$ outside
which final-output observations separate all represented functions.
For $N>2q+1$, take
$\mathcal N=\mathcal N_0\times\X^{N-2q-1}$;
for $N=2q+1$, take $\mathcal N=\mathcal N_0$.
Fix $\mP\notin\mathcal N$. The map $R(\param)\coloneqq\final(\param,\mP)$
is continuous and separates functions on $\paramspace$.
Lemma~\ref{lem:separation-universality}, with $K=\paramspace$,
gives both approximation conclusions and a continuous map
$g:R(\paramspace)\to\mathcal F_{\paramspace}$.
Separation makes $T_\mP$ one-to-one, and
$g(T_\mP(f_\param))=f_\param$, so $g=T_\mP^{-1}$.
Also,
\begin{align*}
\|T_\mP(f)-T_\mP(f')\|_2
\le\sqrt N\|f-f'\|_\infty,
\, f,f'\in\mathcal F_{\paramspace}.
\end{align*}
The exceptional set comes solely from the identification theorem.
\end{proof}

\paragraph{Hidden-layer observations.}
The following Theorem~\ref{thm:hidden-functional-formal} is the formal version of
Theorem~\ref{thm:hidden-functional-informal}.
We retain the complete final probe responses in their coordinate
order, but discard the ordering of the hidden neurons or channels.
For MLPs and MHA stacks, respectively, define
\begin{align*}
v^{i,k}(\param)
&\coloneqq
\left(x_{i,k}(p_j;\param)\right)_{j\in[N]}
\in\RR^N,
\,i\in[L-1],\,k\in[d_i],\\
v^{i,k}(\param)
&\coloneqq
\left([x_{i,n_j}(P_j;\param)]_{a,k}\right)_{
j\in[N],\,a\in[n_j]}
\in\RR^{\sum_j n_j},
\,i\in[L-1],\,k\in[d].
\end{align*}
The coordinates within each vector retain the probe order and,
for attention, the token order. Let $S_i(\param)$ be the
multiset of these vectors, including repetitions
Thus $S_i(\param)$ contains $r_i$ vectors in $\RR^t$, where
$r_i=d_i$, $t=N$ for MLPs and $r_i=d$, $t=\sum_jn_j$ for MHA.

For a multiset $S$ of $r$ vectors in $\RR^t$, we define
\begin{align*}
\mathcal M_r(S)
&\coloneqq
\left(\sum_{v\in S}v^\alpha\right)_{
\alpha\in\mathbb N_0^t,\,1\le|\alpha|\le r},
\,\mathbb N_0\coloneqq\{0,1,\cdots\},\\
|\alpha|&\coloneqq\sum_{b=1}^t\alpha_b,
\,v^\alpha\coloneqq\prod_{b=1}^tv_b^{\alpha_b},\\
R_\mP(\param)
&\coloneqq
\left[
\mathcal M_{r_1}(S_1(\param));\cdots;
\mathcal M_{r_{L-1}}(S_{L-1}(\param));
\final(\param, \mP)
\right].
\end{align*}
Every occurrence of a repeated vector contributes to the sum.

For the MLP variant of \our, we use layerwise Set Transformers followed by a final MLP, as described in Section~\ref{app:hiddenprobe-master}.
For the pure-MHA specialization of \ourmha, we consider
the shared trajectory encoder and layerwise pooling given by
\begin{align*}
\widehat R_{\mP,\varphi}(\param)
&\coloneqq
\left[
\left(
\sum_{k=1}^d\varphi(v^{i,k}(\param))
\right)_{i=1}^{L-1};
\final(\param, \mP)
\right],\\
\mathcal H_{\mP,\phi}(\param)
&\coloneqq
\rho\left(\widehat R_{\mP,\varphi}(\param)\right),
\, \phi\coloneqq(\varphi,\rho).
\end{align*}
Here $\varphi$ and $\rho$ are MLPs satisfying the preceding
activation and bias assumptions, with unrestricted hidden
widths and encoder output dimension.
The same $\varphi$ is used for every hidden channel and layer.
The final responses $\final(\cdot, \mP)$ retain their coordinate order.
For operators, the prediction MLP additionally receives
the query input.

\begin{theorem}[Approximation from hidden-layer probes]
\label{thm:hidden-functional-formal}
Consider either of the following settings.
\begin{enumerate}
\item \emph{MLPs.}
Consider the MLP family of Theorem~\ref{thm:hidden-weaker}.
Its widths are $\Db=(d_0,\cdots,d_L)$ and its activation is
real analytic, injective, and non-polynomial.
Let $\X\subseteq\RR^{d_0}$ be compact with nonempty interior, let
$N\coloneqq1+\max_{0\le i<L}d_i,$
and let $\mathcal U\subseteq\X^N$ consist of pairwise distinct
tuples whose one-padded input matrix has rank $d_0+1$.
This set has full measure by
Lemma~\ref{lem:generic-layerwise-rank}.
Set $f_\param=x_L(\cdot;\param)|_\X$.

\item \emph{Softmax attention stacks.}
Let the family satisfy the assumptions of
Theorem~\ref{thm:att-fixed}, and define
$H\coloneqq\max_{\ell\in[L]}h_\ell,
\,N\coloneqq d(2H+1)+1.$
Let $n_1,\cdots,n_N$ and
$\mathcal U\subseteq\prod_{j=1}^N\RR^{n_j\times d}$
be the length profile and full-measure good-bank set
supplied by that theorem.
For an evaluation length $n\ge1$ and compact
$\X_n\subseteq\RR^{n\times d}$, set
$f_\param=x_{L,n}(\cdot;\param)|_{\X_n}$.
\end{enumerate}
Let $\Theta=\RR^D$ be the respective full raw parameter space.
For every $\mP\in\mathcal U$, there is a closed set
$E_\mP\subseteq\Theta$ with
$\dim_{\mathrm H}E_\mP\le D-1$ such that the following holds.

For every compact $K\subseteq\Theta\setminus E_\mP$,
the representation $R_\mP$ defined above admits a continuous
map $g_\mP:R_\mP(K)\to\mathcal F_K$ with
$g_\mP(R_\mP(\param))=f_\param,\,\param\in K.$
For every continuous $\Lambda:\mathcal F_K\to\RR^k$,
every compact $\mathcal Y\subseteq\RR^a$, every continuous
$\mathcal T:\mathcal F_K\to C(\mathcal Y,\RR^k)$,
and every $\epsilon>0$, there are models from the respective
\our or \ourmha family described above such that
\begin{align*}
\sup_{\param\in K}
\left\|\Lambda(f_\param)-\mathcal H_{\mP,\phi}(\param)\right\|_2
<\epsilon,\,\sup_{\param\in K}\sup_{y\in\mathcal Y}
\left\|\mathcal T(f_\param)(y)
-\mathcal H^{\mathrm{op}}_{\mP,\psi}(\param,y)\right\|_2<\epsilon.
\end{align*}
\end{theorem}

\begin{proof}
Fix $\mP\in\mathcal U$.
For MLPs, we define
\begin{align*}
E_\mP\coloneqq
\left\{\param\in\Theta:
\prod_{i=0}^{L-1}
\det\left(
\overline{\mathbf X}_i^\param(\mP)
\bigl(\overline{\mathbf X}_i^\param(\mP)\bigr)^\top
\right)=0\right\}.
\end{align*}
By Lemma~\ref{lem:generic-layerwise-rank}, the displayed
product is a nonzero real-analytic function of $\param$.
Its zero set is closed and has Hausdorff dimension at most
$D-1$. For $\param_0\notin E_\mP$, equality of the ordered
observations with any raw competitor gives, by injectivity,
$(B_i-B_i^0)\overline{\mathbf X}_i^{\param_0}(\mP)=0,\,0\le i<L.$
Full row rank gives $B_i=B_i^0$ for every $i$.
For an affine output layer, its equality follows directly
without inverting an activation.
For attention, use $E_\mP$ from Theorem~\ref{thm:att-fixed}.
That theorem also identifies every target outside $E_\mP$
against every raw competitor, at every sequence length.

We first show that $\mathcal M_r$ determines an $r$-element
multiset. If $\mathcal M_r(S)=\mathcal M_r(S')$, then for every
$a\in\RR^t$ and $b\in[r]$, expansion of $(a^\top v)^b$ gives
$\sum_{v\in S}(a^\top v)^b=\sum_{v\in S'}(a^\top v)^b.$
For real numbers $c_1,\cdots,c_r$, their power sums
$p_b=\sum_jc_j^b$ determine the coefficients of
$\prod_j(z-c_j)$: writing $e_b$ for the sum of all products
of $b$ distinct $c_j$ and $e_0=1$, Newton's identities give
\begin{align*}
b e_b=\sum_{j=1}^b(-1)^{j-1}e_{b-j}p_j,
\,b\in[r],
\,
\prod_{j=1}^r(z-c_j)=\sum_{b=0}^r(-1)^be_bz^{r-b}.
\end{align*}
Consequently the projected multisets agree, including
multiplicities. Choose $a$ outside the finitely many
hyperplanes $\{a:a^\top(v-v')=0\}$ indexed by distinct
vectors in $S\cup S'$. This projection is one-to-one on
that union, so $S=S'$.

Now fix compact $K\subseteq\Theta\setminus E_\mP$ and
$\param,\param_0\in K$ with
$R_\mP(\param)=R_\mP(\param_0)$.
The preceding argument gives equality of the hidden
trajectory multisets. Thus one permutation per hidden
layer aligns the competitor's responses at all probes.
For MLPs, choose these permutation matrices $P_i$, put
$P_0=P_L=I$, and transform the competing parameter by
\begin{align*}
W_i'=P_{i+1}W_iP_i^\top,
\,b_i'=P_{i+1}b_i,
\,0\le i<L.
\end{align*}
Induction gives $x_i(u;\param')=P_ix_i(u;\param)$.
Hence $\all(\param',\mP)=\all(\param_0,\mP)$ while
$f_{\param'}=f_\param$.

For attention, states are row matrices. Choose permutations
$\Pi_i$ that align their columns. Let $\Pi_0=\Pi_L=I$, and we define
\begin{align*}
Q_{ia}'=\Pi_{i-1}^\top Q_{ia},\,K_{ia}'=\Pi_{i-1}^\top K_{ia},\,V_{ia}'=\Pi_{i-1}^\top V_{ia},\,O_{ia}'=O_{ia}\Pi_i.
\end{align*}
Then $\M_{ia}'=\Pi_{i-1}^\top\M_{ia}\Pi_{i-1}$ and
$\W_{ia}'=\Pi_{i-1}^\top\W_{ia}\Pi_i$, so we get:
\begin{align*}
(X\Pi_{i-1})\M_{ia}'(X\Pi_{i-1})^\top=X\M_{ia}X^\top,\,(X\Pi_{i-1})\W_{ia}'=X\W_{ia}\Pi_i.
\end{align*}
The attention scores are unchanged, and induction gives us $x_{i,n}(X;\param')=x_{i,n}(X;\param)\Pi_i.$
Again, $\all(\param',\mP)=\all(\param_0,\mP)$ and
$f_{\param'}=f_\param$.
The all-competitor identification established above implies
$f_\param=f_{\param_0}$ in either case.
It is unnecessary that $\param'$ belong to $K$.

The map $R_\mP$ is continuous, since its coordinates are
polynomials in finitely many continuous observations.
The realization is jointly continuous on the chosen compact
input domain and $K$. Lemma~\ref{lem:separation-universality}
therefore gives $g_\mP$ and continuous functional and operator
readouts on $R_\mP(K)$ and $R_\mP(K)\times\mathcal Y$.
It remains to realize their approximations by the specified
architectures.

For \our, Set Transformer universality
\cite{settransformer}, with sufficient capacity and output
dimension, approximates each layer's continuous invariant
map $\mathcal M_{r_i}$. The final MLP approximates the resulting
continuous readout.

For \ourmha, the case $L=1$ follows directly from
Lemma~\ref{lem:separation-universality}, since there are
no hidden trajectories. Suppose $L\ge2$, and define
\begin{align*}
\nu(v)&\coloneqq
\left(v^\alpha\right)_{
\alpha\in\mathbb N_0^t,\,1\le|\alpha|\le d},
\quad t\coloneqq\sum_{j=1}^N n_j,\,C\coloneqq
\left\{
v^{i,k}(\param):
\param\in K,\ i\in[L-1],\ k\in[d]
\right\},\\
\mathcal S&\coloneqq R_\mP(K),\,
\mathcal S^1\coloneqq
\{z:\operatorname{dist}(z,\mathcal S)\le1\}.
\end{align*}
Both $C$ and $\mathcal S^1$ are compact.
By Lemma~\ref{lem:separation-universality}, there are
MLPs $\rho,\psi$ satisfying
\begin{align*}
\sup_{\param\in K}
\|\Lambda(f_\param)-\rho(R_\mP(\param))\|_2
<\epsilon/2,\,\sup_{\param\in K}\sup_{y\in\mathcal Y}
\|\mathcal T(f_\param)(y)-\psi(R_\mP(\param),y)\|_2
<\epsilon/2.
\end{align*}
Uniform continuity on $\mathcal S^1$ and
$\mathcal S^1\times\mathcal Y$ gives $\delta\in(0,1)$
such that, whenever $z,z'\in\mathcal S^1$ and
$\|z-z'\|_2<\delta$, we have:
\begin{align*}
\|\rho(z)-\rho(z')\|_2<\epsilon/2,\,\sup_{y\in\mathcal Y}
\|\psi(z,y)-\psi(z',y)\|_2<\epsilon/2.
\end{align*}
Choose an MLP $\varphi$ with
$
\sup_{v\in C}\|\varphi(v)-\nu(v)\|_2
<\frac{\delta}{1+d(L-1)}.
$
Since $\mathcal M_d(S_i)=\sum_{k=1}^d\nu(v^{i,k})$,
we have:
\begin{align*}
\sup_{\param\in K}
\|\widehat R_{\mP,\varphi}(\param)-R_\mP(\param)\|_2
&\le
\sum_{i=1}^{L-1}\sum_{k=1}^d
\sup_{\param\in K}
\|\varphi(v^{i,k}(\param))-\nu(v^{i,k}(\param))\|_2\\
&<\delta.
\end{align*}
In particular, $\widehat R_{\mP,\varphi}(K)\subseteq\mathcal S^1$.
Thus we get:
\begin{align*}
&\sup_{\param\in K}
\|\Lambda(f_\param)
-\rho(\widehat R_{\mP,\varphi}(\param))\|_2\\
&\le
\sup_{\param\in K}
\|\Lambda(f_\param)-\rho(R_\mP(\param))\|_2
+
\sup_{\param\in K}
\|\rho(R_\mP(\param))
-\rho(\widehat R_{\mP,\varphi}(\param))\|_2
<\epsilon.
\end{align*}
Replacing $\rho(z)$ by $\psi(z,y)$ and taking the
supremum over $y\in\mathcal Y$ gives the operator result.
The independence assertions follow from the identification
theorems.
\end{proof}

\section{Architectural details of \our and \ourmha}
\label{app:expt:hiddenprobe-master}

We describe the architectures and training procedures of
\our for MLPs and \ourmha for Transformers in separate subsections.

\subsection{\our{} for MLPs}
\label{app:expt:hiddenprobe-inr}

We describe the architectures used for INR classification
and Model Zoo regression tasks.

\subsubsection{INR Classification}

\paragraph{Target networks and learned probes.}
For INR classification, each target is a SIREN with sine activations. We use $N=128$ learned coordinate probes. Each probe has a
trainable code $\zeta_q\in\R^{32}$, initialized from
$\mathcal N(0,I_{32})$. Two shared affine maps generate
the coordinates:
\begin{align*}
 u_q=\tanh\bigl(A_2(A_1\zeta_q+b_1)+b_2\bigr)\in(-1,1)^2,
 \qquad
 A_1\in\R^{32\times32},\quad
 A_2\in\R^{2\times32}.
\end{align*}
There is no activation between the affine maps.
The same ordered probe bank is presented to every target;
generation is not conditioned on the target's weights
or responses.

\paragraph{Observed responses and neuron trajectories.}
For INR classification, during one forward pass, we record the two hidden
layer responses $X_i^\param\in\R^{N\times32}$,
$i=1,2$, and final outputs $O^\param\in\R^{N\times3}$.
For each hidden neuron, we collect its responses across
the complete bank:
\begin{align*}
 v_x^{i,c}(\param)
 =
 \bigl([X_i^\param]_{q,c}\bigr)_{n=1}^{N}
 \in\R^{128},
 \qquad
 i\in\{1,2\},\quad c\in[32].
\end{align*}
We design features by concatenating the activations with global features: mean, standard deviation, maximum,
minimum. Denoting these statistics by
$s(v)$, the resulting feature vector is $\chi(v)=[v; s(v)].$
No normalization is applied before the token projection.
Probe order is retained within every trajectory.

\paragraph{Neuron tokens and joint readout.}
A shared affine map projects every neuron feature to
$\R^{d_e}$, and a learned embedding $e_i$ identifies its
hidden layer:
\begin{align*}
 t_{i,c}^\param
 &=A\chi\bigl(v_x^{i,c}(\param)\bigr)+a+e_i,\\
 t_{\mathrm{out}}^\param
 &=A_{\mathrm{out}}\operatorname{vec}(O^\param)
   +a_{\mathrm{out}}.
\end{align*}

One joint Transformer processes all tokens.
We use three blocks: each block has eight attention heads,
feedforward, ReLU, dropout $0.1$ and
LayerNorm before its attention and feedforward sublayers.
There is no final LayerNorm after the last block.
An MLP maps the final CLS token to class logits.

There are no neuron-index embeddings or sequence
positional encodings. At evaluation, when dropout is
disabled, permuting the neurons within a hidden layer
only permutes their tokens and leaves the prediction
unchanged. Probe order, in contrast, is retained in both
the trajectories and final-output vector.

\paragraph{Relation to the theory.}
The neuron tokens preserve correspondence across probes,
as in the theoretical trajectory representation.
The implementation processes the layers jointly rather
than using separate layerwise set encoders; layer
embeddings retain their identities. It also observes
preactivations, from which the theoretical postactivation
responses can be computed.
Sine is analytic but not injective, so the
parameter-identification theorem stated for injective
activations does not apply directly to these targets.
The learned probes and finite-width token projections
are not assumed to satisfy exact separation conditions.
The universality results concern capacity-growing
architecture families under their stated assumptions,
not every fixed-width encoder or trained parameter choice.

\subsubsection{\our{} for CNN-zoo accuracy regression}
\label{app:expt:hiddenprobe-cnn}

\paragraph{Target networks and learned probes}
The regression targets are image-classification CNNs.
The grayscale zoos use three convolutional layers with
channels $1\to16\to16\to16$, kernel size $3$ and stride
$2$, followed by global average pooling and a ten-output
affine classifier. The Wild Park zoo contains targets
with different depths, channel counts, convolutional
settings and activations.

Here, we use $N=128$ learned image probes. Each probe is generated from a trainable code in $\R^{32\times1\times1}$ and has the appropriate image dimensions. A final $\tanh$ bounds the image values to $(-1,1)$. The same learned probe bank is used for every target network. The output and hidden-response branches described below share this probe bank and the same target forward pass. We follow a protocol similar to that of~\cite{kahana} for learning probes in classification and regression tasks.

\paragraph{Hidden aggregation and Prediction}
The regressor stacks the probes' output at each channel and then pools the channels within each layer. Like in the case of INR classification, we also use additional global statistics like channelwise means, maxima, standard deviation etc
over probes to design extra features to aggregate all the information from hidden layers.

One self-attention block mixes the layer vectors using
a learned relative-layer distance bias, followed by a
residual pre-LayerNorm feedforward block with GELU.
A final attention pool over layers, funal output and a projection
give $z_H\in\R^{d_z}$.
The attention projections in this aggregation have
inner dimension $16$; large hidden-branch projections
use factorized weights to reduce the number of parameters, and match the parameter count with that of ProbeGen.

The predictor is an MLP that taken in the hidden combined representation and gives a scalar that predicts the accuracy.

\paragraph{Relation to the theory}
The regressor follows the shared-probe principle and
uses hidden responses without additional target-network
queries. Its channel pooling respects channel
permutations. It is an empirical adaptation of hidden
probing in the context of CNN Model Zoo.

\subsection{\ourmha for Transformers}
\label{app:expt:hiddenprobe-mha}
\paragraph{Details on the learned probes.}
For MNIST, we use $N=256$ probes divided into two families.
The first contains $192$ learned sequences, each consisting
of $17$ token vectors in $\R^{32}$. The remaining
$64$ probes are learned $28\times28$ images.

Each encoder probe is parameterized by independently learned
$32$-dimensional token codes. Each image is parameterized by
$49$ learned $16$-dimensional codes, corresponding to its
$4\times4$ patches. Within each route, a shared residual
two-layer MLP transforms these codes: $G(z)=z+W_2\operatorname{GELU}(W_1z+b_1)+b_2.$
The transformed image patches are assembled into a
$28\times28$ image.

The same learned bank is presented to every target:
generation is not conditioned on the target's weights
or its responses.

\paragraph{Observed responses and channel trajectories.}
For MNIST and AGNews datasets, the Transformers have two blocks.
For every probe, we record the outputs of the target's two
transformer blocks, its final normalized token states,
and its classifier logits. Denote the three hidden layers
by $X_1$, $X_2$, and $X_{\mathrm{norm}}$.
Each hidden layer has $32$ channels, and each probe produces
a $10$-dimensional logit vector.

For each layer $i\in\{1,2,\mathrm{norm}\}$ and channel
$c\in[32]$, we concatenate that channel's responses across
all probes and token positions:
\[
v^{i,c}(\param)
=
\left(
[X_i(P_j;\param)]_{t,c}
\right)_{j\in[N],\,t\in[n_j]}
\in\R^T,
\,
\]
The concatenation uses a fixed order: encoder probes first,
then image probes, with token positions kept in order
within each probe. Consequently, each channel is represented
by its responses to the complete bank.

\paragraph{Encoding and combining channel features.}
One shared trajectory encoder $E$ is applied independently
to every channel in all three blocks: $u^{i,c}=E(v^{i,c})\in\R^{192}.$
The encoder consists of
$\operatorname{Linear}$, GELU,
$\operatorname{Linear}$, and LayerNorm.
It produces $32$ feature vectors for each view; it does
not pool across channels.

We then combine each block-2 feature with the
final-normalized feature of the same channel:
$\widetilde u^{2,c}=
\operatorname{Fuse}
\left([u^{2,c};u^{\mathrm{norm},c}]\right)
\in\R^{192}.$ This is concatenated with the $32$ block-1 features, giving
$64$ hidden-response tokens in total. 

\paragraph{Joint readout and pooling.}
Each probe's classifier logits are projected from
$\R^{10}$ to $\R^{192}$. Together with a learned CLS token, the readout input
contains $1+32+32+256=321$ tokens of dimension $192$. Two pre-Layernorm Transformer
blocks process these tokens jointly, a final LayerNorm is applied to their
outputs. 

Write the resulting tokens as
$z_{\mathrm{CLS}},z_1,\ldots,z_{320}\in\R^{192}$.
We form
\[
r_\param
=
\left[
z_{\mathrm{CLS}};
\frac{1}{320}\sum_{a=1}^{320}z_a;
\max_{a\in[320]}z_a
\right]
\in\R^{576},
\]
where the maximum is coordinatewise.
Thus, mean and maximum pooling act jointly over all
non-CLS tokens, including the logit tokens, rather than
separately within each layer.

A prediction head with dimensions $576\to192\to1$ and
a GELU activation maps $r_\param$ to standardized
accuracy. The predictor receives only the recorded responses and
fixed semantic identifiers, not target weights, target
identities, or generated probe coordinates.

\paragraph{Relation to the theory.}
The theoretical constructions motivate a shared bank
with different sequence lengths and the use of intermediate
responses. The experiments apply these principles to
complete transformer classifiers containing feedforward
and normalization layers. They do not assume that the
learned probes satisfy the exact reconstruction conditions
or identify the target function. Likewise, the universality
results concern capacity-growing architecture families,
not a guarantee for this particular finite-width
configuration or its learned parameters.

\section{Experimental details}
\subsection{Function-level reconstruction on synthetic neural networks}
\label{app:expt:synthetic}

To illustrate the benefits of hidden activations, we study function-level reconstruction of frozen neural networks from a limited number of input probes. We compare \emph{output-only regression}, which observes only the final network outputs, with \emph{intermediate-layer regression}, which independently fits each layer using observed teacher activations as its inputs and targets. Our goal is to examine how well each approach reconstructs the target function, both between the observed probes and outside the training interval.

\paragraph{Common Setup}
We consider function-level reconstruction of MLPs with different activations and stacks of MHA blocks, using either final outputs alone or complete hidden-layer observations. Let $\bparam\coloneqq(\param_1,\ldots,\param_L)$ denote the parameters of a sequential model
$f_{\bparam}=f_{\param_L}\circ\cdots\circ f_{\param_1}$.
Given a set of probes $\mP=(p_1,\ldots,p_N)$, we observe either the final outputs $\final(\bparam,\mP)$ or the complete hidden-layer observations $\all(\bparam,\mP)$.

For the output-only case we do MSE regression to interpolate on the probes. If $f_{\hat{\bparam}}$ is the ``student'' network trying to reconstruct the ``teacher'' network, we minimize the following losses: (i) For learning from $\final$, we minimize $\min_{\hat{\bparam}}\sum_{i=1}^N \|f_{\hat{\bparam}}(p_i) - f_{{\bparam}}(p_i)\|^2$ (ii) For learning from $\all$, we fit sum of MSE on each of the $L$ layers, given by $\sum_{\ell}\sum_{i=1}^N \|f_{\hat{\param_\ell}}(X^{\ell-1}_i) -X^{\ell}_i \|^2$, where $X^{\ell}_i \coloneqq f_{\param_\ell} \circ\cdots\circ f_{\theta_1}(p_i)$, i.e $\{X^{\ell}_i\}$ are the observed outputs from the $\ell$-th layer. Here we fit each hidden layer independently but simultaneously using MSE regression.

\paragraph{Analytic MLP:} We used a neural network with $\tanh$ activation of widths $(1,4,4,1)$, the training interval was $[-1,1]$, and the evaluation
interval is $[-3,3]$. As shown in \autoref{fig:combined_analytic}, access to the hidden activations enables essentially exact reconstruction once sufficiently many observations are available, whereas reconstruction from final outputs alone improves much more gradually and retains a substantially larger error even as the number of observed samples increases.

\begin{figure}[H]
    \centering
    \begin{subfigure}{0.32\textwidth}
        \centering
        \includegraphics[width=\linewidth]{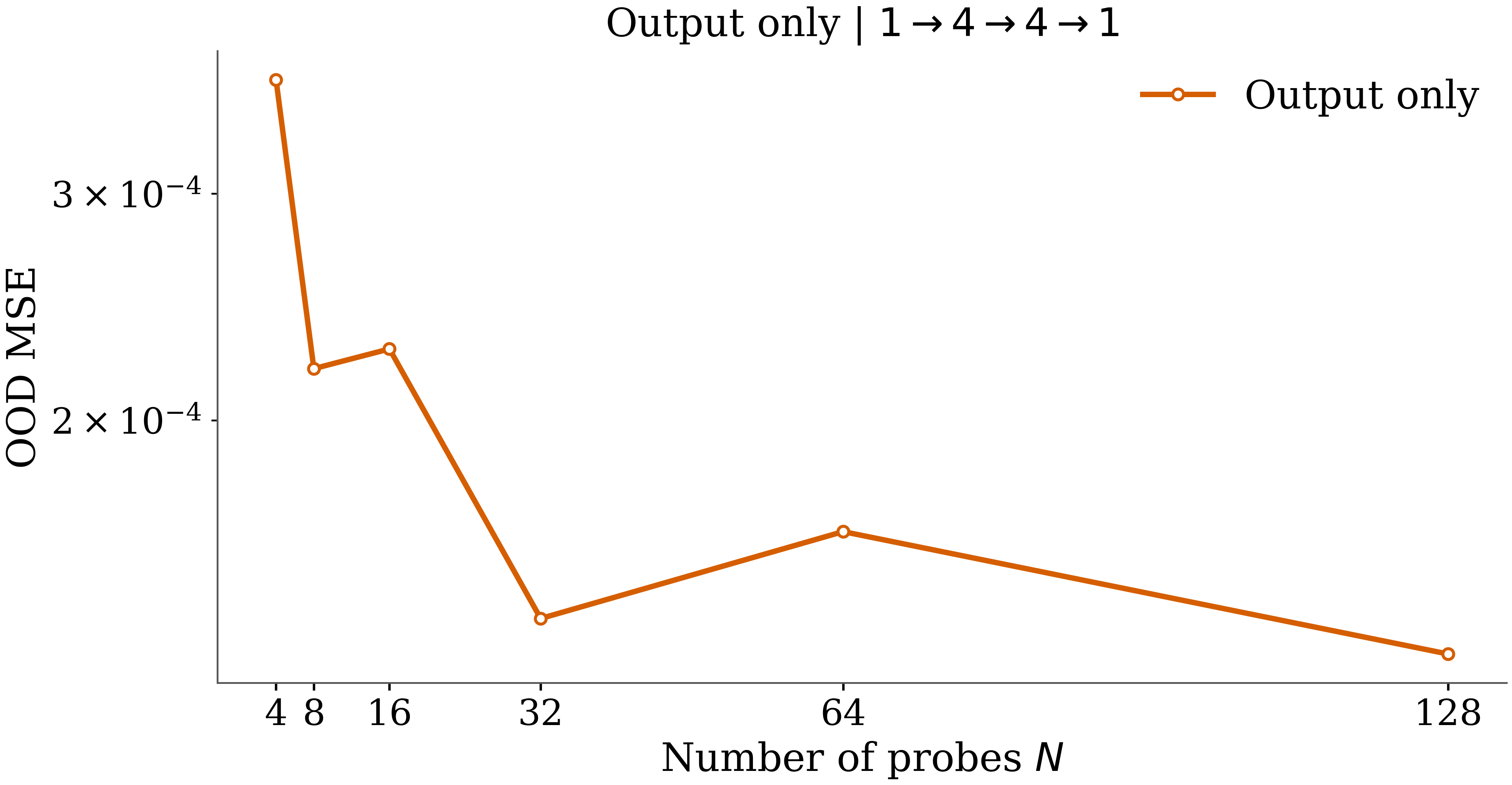}
        \caption{OOD MSE from $\final$}
        \label{fig:analytic-mse-out}
    \end{subfigure}\hfill
    \begin{subfigure}{0.32\textwidth}
        \centering
        \includegraphics[width=\linewidth]{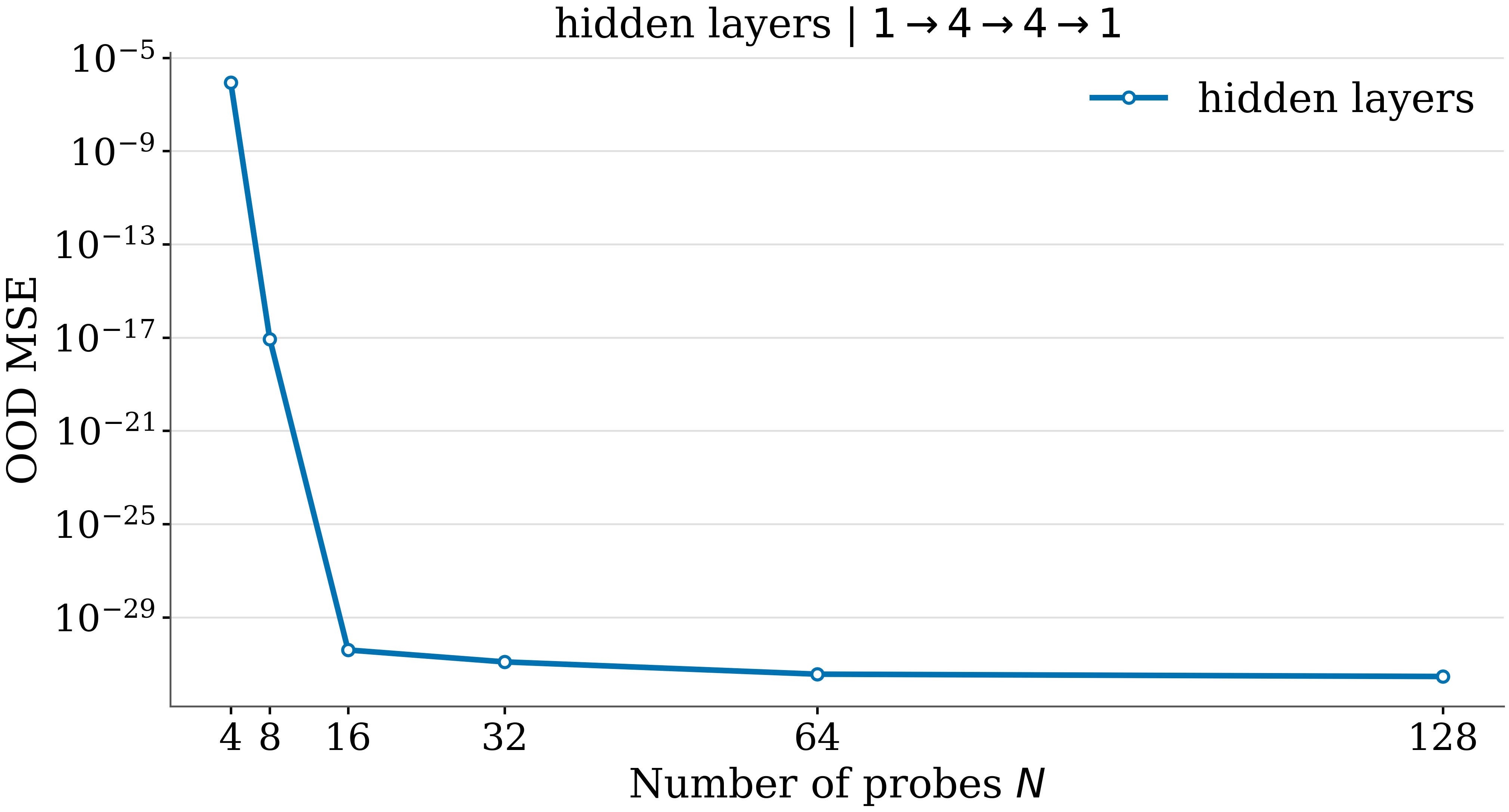}
        \caption{OOD MSE from $\all$}
        \label{fig:analytic-mse-all-1}
    \end{subfigure}
    \begin{subfigure}{0.32\textwidth}
        \centering
        \includegraphics[width=\linewidth]{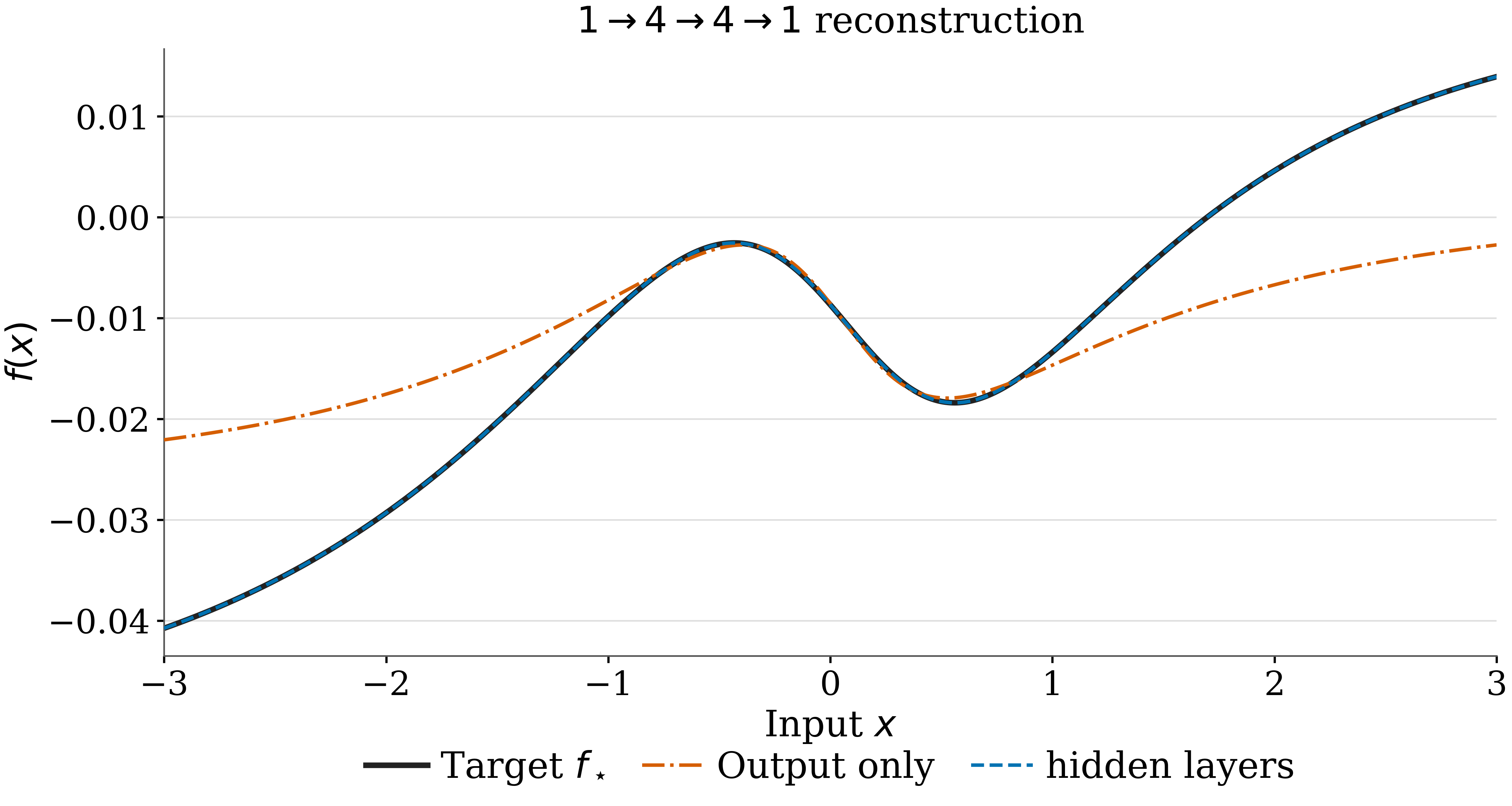}
        \caption{Reconctruction from $\all, \final$}
        \label{fig:analytic-mse-all-2}
    \end{subfigure}
    
    \caption{Reconstruction of Analytic Net from $\final, \all$.}
    \label{fig:combined_analytic}
\end{figure}

\paragraph{ReLU MLP:} As  in the analytic activation case, Theorem~\ref{thm:relu-pac-informal} suggests that hidden activations can reduce the number of probes necessary to (almost) identify a ReLU network.  We take a target ReLU network with 2 hidden layers, width 4 and weights that represent a member of the ``sawtooth'' function class, having a large number of breakpoints. As discussed in \cite{haninrolnick}, such a large number of breakpoints is atypical. This makes the target a challenging example for reconstruction from final outputs alone using MSE regression from random initialization. This matches the observation: we can't reconstruct the target with just the final output. But with hidden activations, we can reconstruct the target almost exactly with far fewer probes. The same is also reflected in the out-of-distribution MSE plots

\begin{figure}[H]
    \centering
    \begin{subfigure}{0.32\textwidth}
        \centering
        \includegraphics[width=\linewidth]{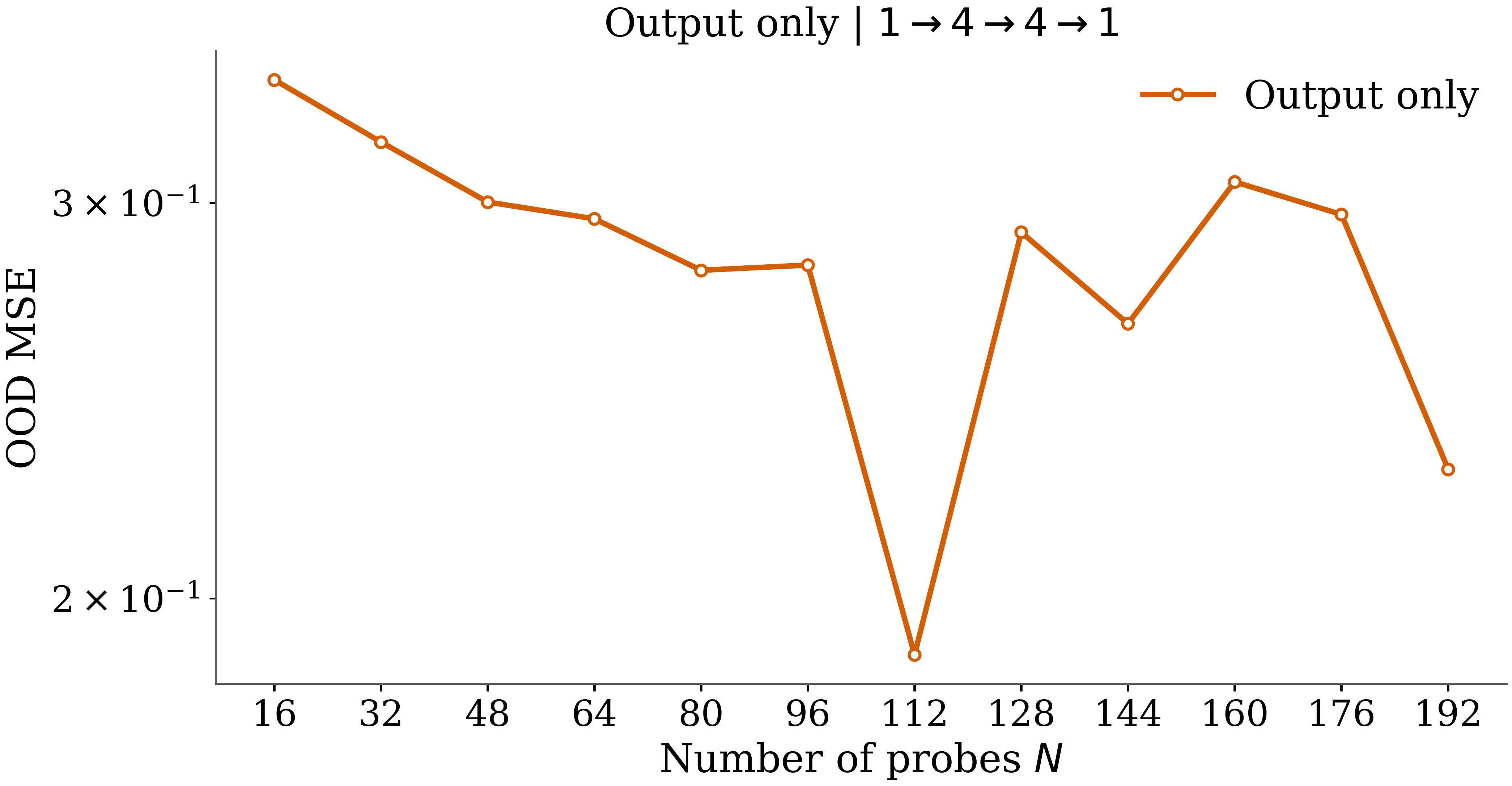}
        \caption{OOD MSE from $\final$}
        \label{fig:relu-mse-out}
    \end{subfigure}\hfill
    \begin{subfigure}{0.32\textwidth}
        \centering
        \includegraphics[width=\linewidth]{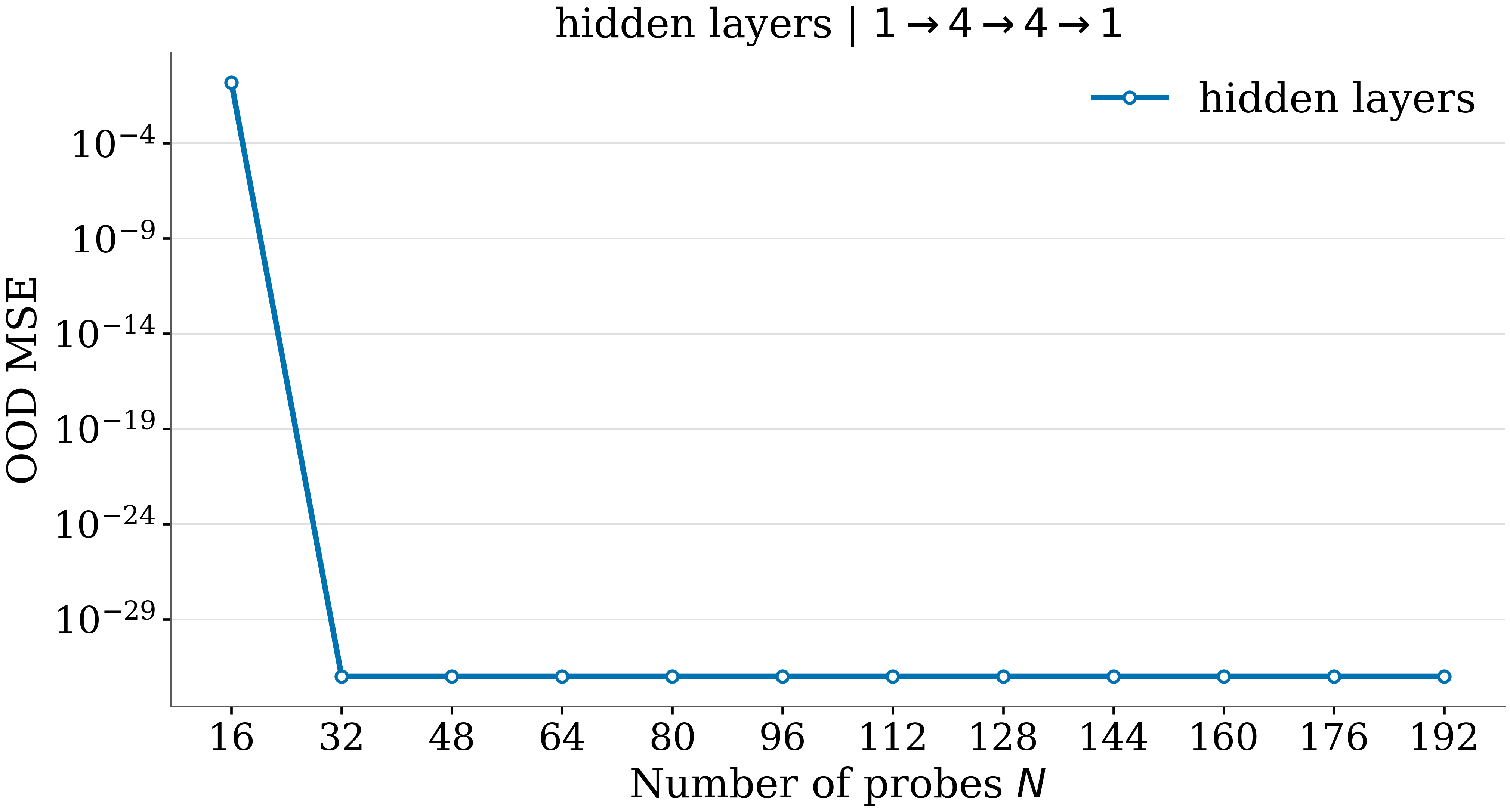}
        \caption{OOD MSE from $\all$}
        \label{fig:relu-mse-all-1}
    \end{subfigure}
    \begin{subfigure}{0.32\textwidth}
        \centering
        \includegraphics[width=\linewidth]{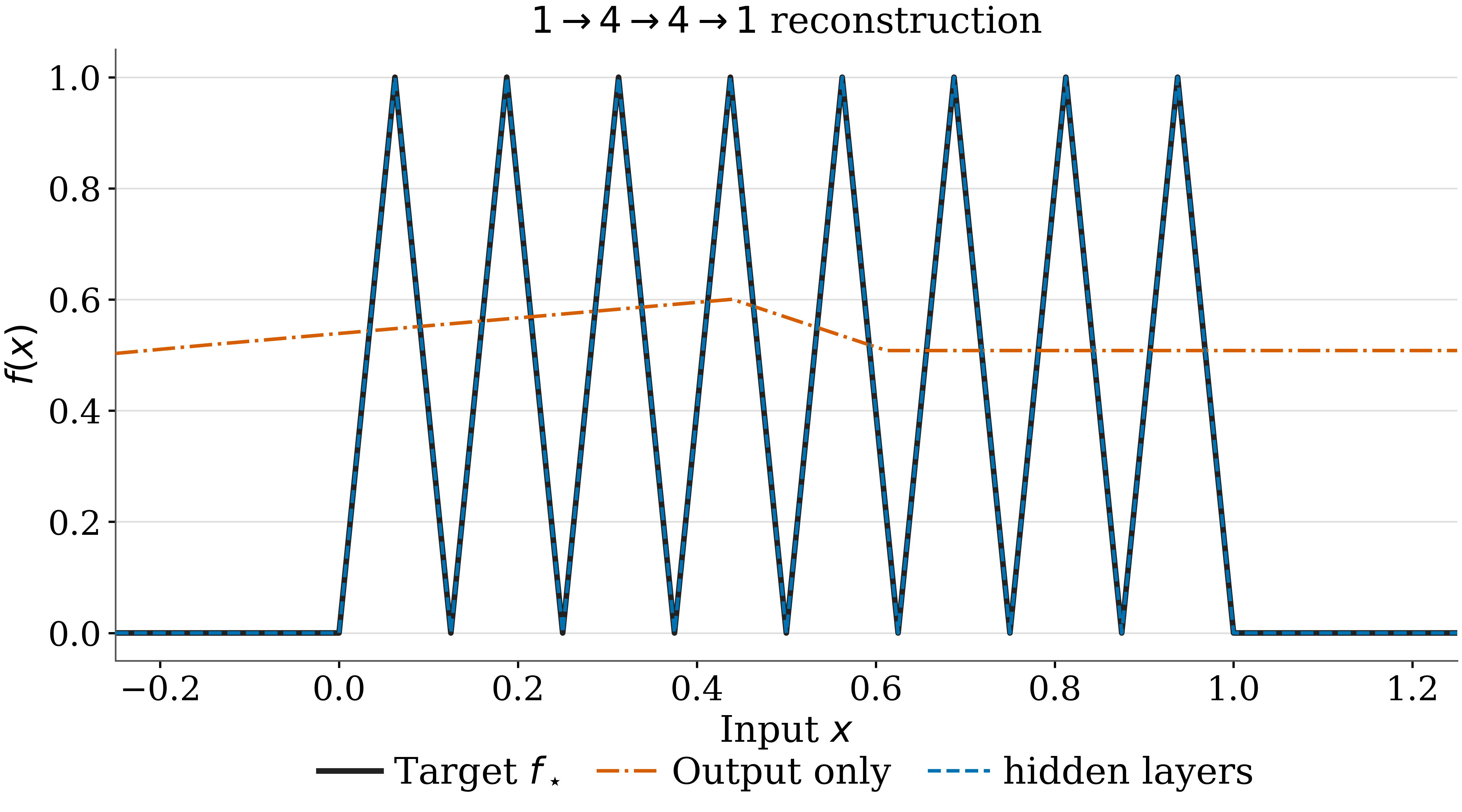}
        \caption{Reconctruction from $\all, \final$}
        \label{fig:relu-mse-all-2}
    \end{subfigure}
    
    \caption{Reconstruction of ReLU Net from $\final, \all$.}
    \label{fig:combined_relu}
\end{figure}

\paragraph{Multi Head Attention Stack} Here, we want to reconstruct an MHA stack. We randomly initialize a two-layer multi-head attention network with input width $d=4$, two heads per layer
and query/key and value widths $k=r=2$ (bias-free softmax attention, without residual connections, normalization or feed-forward blocks). We sample four single-token inputs
followed by generic three-token inputs, and a budget of $N$ probes uses the first $N$ of them, for
$N$ from 5 to 40. We fit the ``student'' MHA block jointly using $\final$, and each layer independently using $\all$, both using MSE regression as described above. 

\begin{figure}[H]
    \centering
    \includegraphics[width=0.5\linewidth]{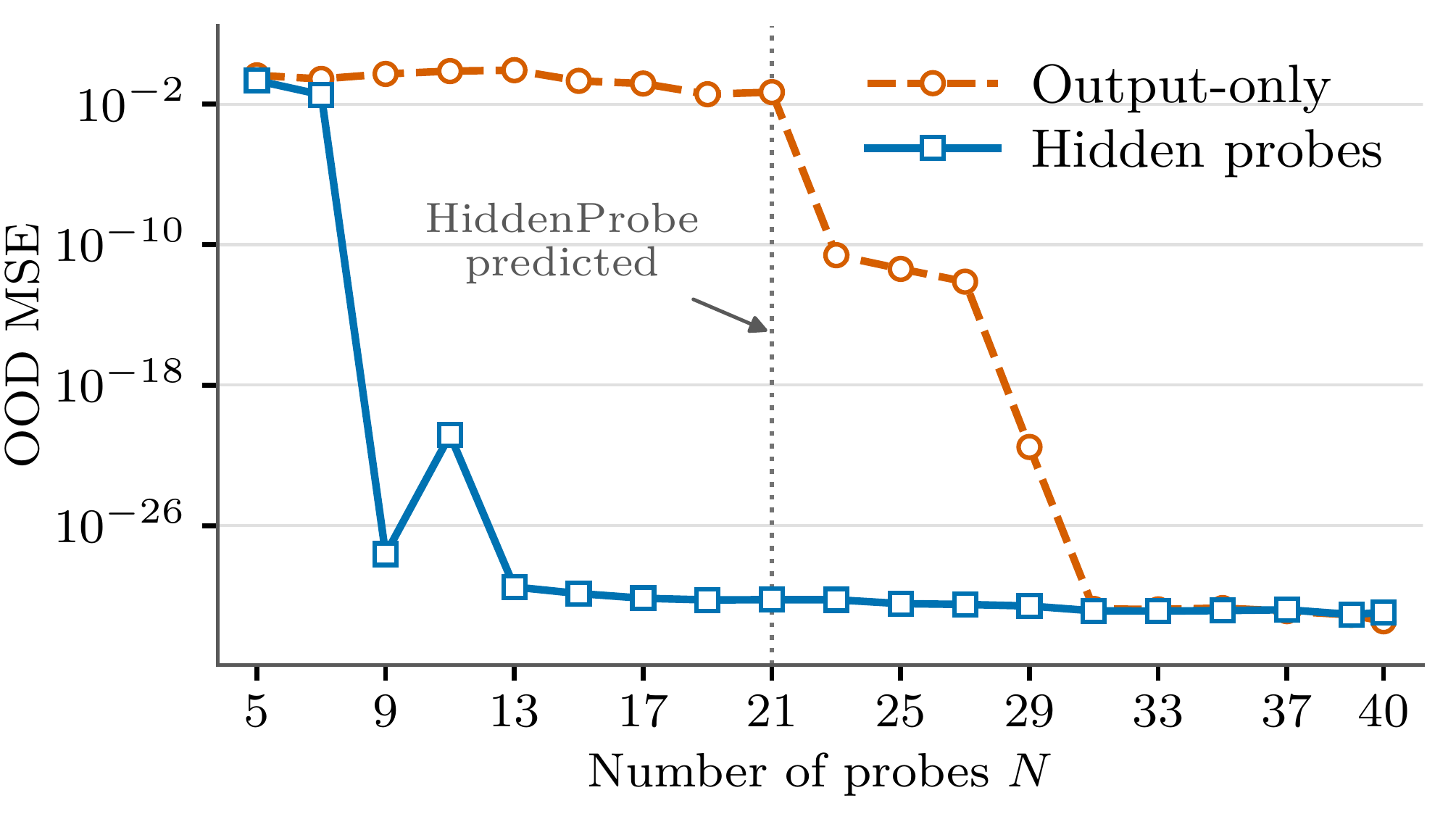}
    \caption{Reconstruction of MHA Stack from $\final, \all$.}
    \label{fig:mha}
\end{figure}

As shown in \autoref{fig:mha}, access to the hidden
activations enables essentially exact reconstruction from as few as nine probes, below the $N=21$ probes that Theorem~\ref{thm:att-fixed-informal} predicts, whereas reconstruction from final outputs alone becomes exact only at about 31 probes.

Together, these three experiments suggest that hidden activations enable more accurate function reconstruction with fewer probes than final outputs alone. This observation is consistent with our empirical results on neural functional learning, where \our, which uses hidden activations, outperforms ProbeGen, which relies only on final outputs.

\label{app:hyperparameters}
\subsection{Full INR Classification Comparison}
\label{app:inr_full}

Table~\ref{tab:inr_classification_full} provides an extended comparison on the INR classification benchmarks, including both recent weight-space methods and several earlier baselines omitted from the main table for brevity. The MLP, MLP with permutation augmentation, MLP with alignment, INR2Vec, Transformer, and DWS baselines are taken from the DWSNets evaluation~\citep{dws_paper}; these results were originally reported on the MNIST and Fashion-MNIST INR benchmarks. For methods evaluated on CIFAR-10 and augmented CIFAR-10, we additionally use the comparison reported by ScaleGMN~\citep{kalogeropoulos2024scale}, which includes DWS, NFN, NG-GNN, ScaleGMN, and ScaleGMN-B on all four INR benchmarks. We further include results from Neural Graphs~\citep{kofinas}, NFT~\citep{zhou2023neural}, Monomial-NFN~\citep{tran2024monomial}, Quasi-Equivariant Monomial-NFN~\citep{tran2026quasi}, MAGEP-NFN~\citep{vo2025magep}, and ProbeGen~\citep{kahana}. Results unavailable in the corresponding source are marked with ``--''.
\begin{table*}[t]
\centering
\caption{Classification accuracy ($\uparrow$) on INR weight-space benchmarks.}
\label{tab:inr_classification_full}
\resizebox{\textwidth}{!}{
\begin{tabular}{lcccc}
\toprule
\textbf{Method}
& \textbf{MNIST}
& \textbf{FMNIST}
& \textbf{CIFAR-10}
& \textbf{CIFAR-10 Aug} \\
\midrule

StatNN
& $0.398 \pm 0.001$
& $0.418 \pm 0.002$
& --
& -- \\

DWS
& $0.857 \pm 0.006$
& $0.671 \pm 0.003$
& $0.3445 \pm 0.004$
& $0.4127 \pm 0.001$ \\

NFN
& $0.791 \pm 0.008$
& $0.689 \pm 0.006$
& $0.3341 \pm 0.001$
& $0.4660 \pm 0.001$ \\

ScaleGMN
& $0.966 \pm 0.002$
& $0.808 \pm 0.001$
& $0.3882 \pm 0.001$
& $0.5695 \pm 0.005$ \\

NG-GNN
& $0.914 \pm 0.006$
& $0.680 \pm 0.002$
& $0.3604 \pm 0.004$
& $0.4570 \pm 0.002$ \\

NG-T
& $0.924 \pm 0.003$
& $0.727 \pm 0.006$
& --
& -- \\

NFT
& --
& --
& --
& \best{$0.6340 \pm 0.000$} \\

Neural Graphs(128 probes)
& $0.976 \pm 0.001$
& $0.745 \pm 0.008$
& -
& -\\

Monomial-NFN
& --
& --
& $0.3423 \pm 0.003$
& -- \\

Monomial-NFN Quasi
& --
& --
& $0.3532 \pm 0.005$
& -- \\

MAGEP-NFN
& --
& --
& $0.3718 \pm 0.003$
& -- \\

\midrule

ProbeGen
& \second{$0.984 \pm 0.001$}
& \second{$0.877 \pm 0.003$}
& \second{$0.570 \pm 0.006$}
& $0.563 \pm 0.003$ \\

HiddenProbe
& \best{$0.986 \pm 0.0006$}
& \best{$0.886 \pm 0.005$}
& \best{$0.580 \pm 0.002$}
& \second{$0.606 \pm 0.002$} \\

\bottomrule
\end{tabular}}
\end{table*}

\subsection{Full Test-Accuracy Prediction Comparison}
\label{app:regression_full}

Table~\ref{tab:regression_full} provides an extended comparison on the neural-network test-accuracy prediction benchmarks. We include a broad range of approaches for learning from neural-network weights. MLP is a simple baseline operating directly on flattened network parameters, while StatNN~\citep{unterthiner2020predicting} predicts performance from summary statistics of the weights and biases. SNE~\citep{andreis2024set} represents neural networks using a hierarchical set-based encoder. We additionally include the permutation-equivariant NFN variants, NFN$_{\mathrm{NP}}$ and NFN$_{\mathrm{HNP}}$~\citep{zhou2023permutation}, as well as the scale- and permutation-aware ScaleGMN and ScaleGMN-B architectures~\citep{kalogeropoulos2024scale}.

We further compare with graph-based weight-space models. NG-GNN and NG-T~\citep{kofinas} represent the input network as a graph and process it using a graph neural network or relational Transformer, respectively. DNG-Encoder~\citep{wu2026dynamic} extends this approach by constructing a dynamic neural graph that models inference-time information flow. We also report the probe-enhanced Neural Graphs variants with 64 and 128 probes, as well as vanilla probing with 64 and 128 probes, following the evaluation of ProbeGen~\citep{kahana}. Finally, we include NFT~\citep{zhou2023neural}, ProbeGen~\citep{kahana}, and our HiddenProbe model. Results that were not reported for a particular benchmark are marked with ``--''.

\begin{table}[H]
\centering
\small
\setlength{\tabcolsep}{4pt}
\resizebox{\textwidth}{!}{
\begin{tabular}{lccccc}
\toprule
\textbf{Method}
& \textbf{MNIST}
& \textbf{FMNIST}
& \textbf{SVHN}
& \textbf{CIFAR-WP}
& \textbf{CIFAR-GS} \\
\midrule

MLP
& $0.878 \pm 0.001$
& $0.874 \pm 0.001$
& $0.809 \pm 0.003$
& --
& $0.880 \pm 0.000$ \\

StatNN
& $0.926 \pm 0.000$
& $0.915 \pm 0.000$
& $0.843 \pm 0.000$
& $0.719 \pm 0.010$
& $0.914 \pm 0.000$ \\

SNE
& $0.941 \pm 0.000$
& $0.928 \pm 0.001$
& $0.858 \pm 0.003$
& --
& $0.927 \pm 0.000$ \\

\midrule

NFN$_{\mathrm{NP}}$
& $0.937 \pm 0.000$
& $0.922 \pm 0.001$
& $0.856 \pm 0.001$
& --
& $0.922 \pm 0.001$ \\

NFN$_{\mathrm{HNP}}$
& $0.942 \pm 0.001$
& $0.935 \pm 0.000$
& \second{$0.931 \pm 0.005$}
& --
& $0.934 \pm 0.001$ \\

NFT
& --
& --
& $0.858 \pm 0.000$
& --
& $0.926 \pm 0.001$ \\

ScaleGMN
& --
& --
& --
& --
& $0.941 \pm 0.006$ \\

ScaleGMN-B
& --
& --
& --
& --
& $0.941 \pm 0.000$ \\

\midrule

NG-GNN
& --
& --
& $0.863 \pm 0.002$
& $0.804 \pm 0.009$
& $0.930 \pm 0.001$ \\

NG-T
& --
& --
& $0.872 \pm 0.001$
& $0.817 \pm 0.007$
& $0.935 \pm 0.000$ \\

DNG-Encoder
& --
& --
& $0.867 \pm 0.002$
& $0.8743 \pm 0.0021$
& $0.936 \pm 0.000$ \\

Neural Graphs (64 probes)
& --
& --
& --
& $0.888 \pm 0.009$
& $0.938 \pm 0.001$ \\

Neural Graphs (128 probes)
& --
& --
& --
& $0.885 \pm 0.005$
& $0.938 \pm 0.000$ \\

\midrule

Vanilla Probing (64 probes)
& --
& --
& --
& $0.885 \pm 0.008$
& $0.933 \pm 0.001$ \\

Vanilla Probing (128 probes)
& --
& --
& --
& $0.889 \pm 0.008$
& $0.936 \pm 0.001$ \\

\midrule

ProbeGen
& \second{$0.954 \pm 0.001$}
& \second{$0.947 \pm 0.003$}
& $0.873 \pm 0.008$
& \second{$0.932 \pm 0.006$}
& \second{$0.957 \pm 0.001$} \\

HiddenProbe
& \best{$0.966 \pm 0.004$}
& \best{$0.957 \pm 0.002
$}
& \best{$ 0.957 \pm 0.005
$}
& \best{$0.940 \pm 0.006$}
& \best{$0.965 \pm 0.000$} \\

\bottomrule
\end{tabular}}
\caption{Full comparison of test-accuracy prediction performance measured by Kendall's $\tau$ ($\uparrow$). Results that were not reported for a particular benchmark are denoted by ``--''.}
\label{tab:regression_full}
\end{table}

\subsection{Transformer Accuracy Prediction Across Thresholds}
\label{app:transformer_thresholds}
Following the evaluation protocol of~\citet{tran2026quasi}, we additionally evaluate \ourmha on subsets of networks obtained by applying accuracy thresholds of $20\%$, $40\%$, $60\%$, and $80\%$. Tables~\ref{tab:mnist_transformers} and~\ref{tab:agnews_transformers} report the results for MNIST-Transformers and AGNews-Transformers, respectively. \ourmha consistently outperforms the previous state-of-the-art method, Transformer-NFN Quasi, achieving the highest Kendall's $\tau$ across all evaluated thresholds on both benchmarks.
\begin{table}[H]
\centering
\resizebox{\textwidth}{!}{
\begin{tabular}{lccccc}
\toprule
\textbf{Accuracy threshold}
& \textbf{No threshold}
& \textbf{20\%}
& \textbf{40\%}
& \textbf{60\%}
& \textbf{80\%} \\
\midrule

MLP
& $0.866\pm0.002$
& $0.873\pm0.001$
& $0.874\pm0.003$
& $0.874\pm0.006$
& $0.873\pm0.007$ \\

STATNN \citep{unterthiner2020predicting}
& $0.881\pm0.001$
& $0.872\pm0.001$
& $0.868\pm0.001$
& $0.860\pm0.001$
& $0.856\pm0.001$ \\

XGBoost \citep{chen2016xgboost}
& $0.860\pm0.002$
& $0.839\pm0.004$
& $0.869\pm0.003$
& $0.846\pm0.001$
& $0.884\pm0.001$ \\

LightGBM \citep{ke2017lightgbm}
& $0.858\pm0.002$
& $0.835\pm0.001$
& $0.847\pm0.001$
& $0.822\pm0.001$
& $0.830\pm0.001$ \\

Random Forest \citep{breiman2001random}
& $0.772\pm0.002$
& $0.758\pm0.004$
& $0.769\pm0.001$
& $0.752\pm0.001$
& $0.759\pm0.001$ \\

\midrule

Transformer-NFN \citep{tran2025b}
& $0.905\pm0.002$
& $0.899\pm0.001$
& $0.895\pm0.001$
& $0.895\pm0.002$
& $0.888\pm0.002$ \\

Transformer-NFN large \citep{tran2026quasi}
& $0.907\pm0.001$
& $0.904\pm0.002$
& $0.897\pm0.002$
& \second{$0.897\pm0.002$}
& $0.890\pm0.001$ \\

Transformer-NFN Quasi \citep{tran2026quasi}
& \second{$0.911\pm0.001$}
& \second{$0.905\pm0.001$}
& \second{$0.898\pm0.002$}
& \second{$0.897\pm0.001$}
& \second{$0.892\pm0.001$} \\

\midrule

ProbeGen
& $0.898\pm0.001$
& $0.878 \pm 0.003$
& $0.866 \pm 0.003$
& $0.866 \pm 0.003$
& $0.851 \pm 0.003$ \\
HiddenProbe
& \best{$0.920 \pm 0.003$}
& \best{$0.910 \pm 0.002$}
& \best{$0.904 \pm 0.001$}
& \best{$0.900 \pm 0.002$}
& \best{$0.898 \pm 0.002$}
 \\
\bottomrule
\end{tabular}
}
\caption{Performance measured by Kendall's $\tau$ on the MNIST-Transformers dataset.
Uncertainties indicate the standard error over 5 seeds.}
\label{tab:mnist_transformers}
\end{table}

\begin{table}[H]
\centering
\resizebox{\textwidth}{!}{
\begin{tabular}{lccccc}
\toprule
\textbf{Accuracy threshold}
& \textbf{No threshold}
& \textbf{20\%}
& \textbf{40\%}
& \textbf{60\%}
& \textbf{80\%} \\
\midrule

MLP
& $0.879\pm0.006$
& $0.875\pm0.001$
& $0.841\pm0.012$
& $0.842\pm0.001$
& $0.862\pm0.006$ \\

STATNN \citep{unterthiner2020predicting}
& $0.841\pm0.002$
& $0.839\pm0.003$
& $0.812\pm0.003$
& $0.813\pm0.001$
& $0.812\pm0.001$ \\

XGBoost \citep{chen2016xgboost}
& $0.859\pm0.001$
& $0.852\pm0.002$
& $0.872\pm0.002$
& $0.874\pm0.001$
& $0.872\pm0.001$ \\

LightGBM \citep{ke2017lightgbm}
& $0.835\pm0.001$
& $0.845\pm0.001$
& $0.837\pm0.001$
& $0.835\pm0.001$
& $0.820\pm0.001$ \\

Random Forest \citep{breiman2001random}
& $0.774\pm0.003$
& $0.801\pm0.001$
& $0.797\pm0.001$
& $0.798\pm0.002$
& $0.773\pm0.001$ \\

\midrule

Transformer-NFN \citep{tran2025b}
& $0.910\pm0.001$
& $0.908\pm0.001$
& $0.897\pm0.001$
& $0.896\pm0.001$
& $0.890\pm0.001$ \\

Transformer-NFN large \citep{tran2026quasi}
& $0.913\pm0.001$
& $0.910\pm0.002$
& $0.898\pm0.002$
& $0.898\pm0.001$
& $0.893\pm0.002$ \\

Transformer-NFN Quasi \citep{tran2026quasi}
& \second{$0.914\pm0.001$}
& \second{$0.913\pm0.002$}
& \second{$0.901\pm0.001$}
& \second{$0.903\pm0.002$}
& \second{$0.896\pm0.001$} \\

\midrule

ProbeGen
& $0.882 \pm 0.005$
& $0.882 \pm0.002$
& $0.859 \pm 0.002$
& $0.859 \pm 0.004$
& $0.855 \pm 0.002$ \\
HiddenProbe
& \best{$0.916 \pm 0.002$}
& \best{$0.918 \pm 0.002$}
& \best{$0.906 \pm 0.002$}
& \best{$0.904 \pm 0.002$}
& \best{$ 0.901 \pm 0.002$}
 \\
\bottomrule
\end{tabular}
}
\caption{Performance measured by Kendall's $\tau$ on the AGNews-Transformers dataset.
Error bars indicate the standard error over 5 seeds.}
\label{tab:agnews_transformers}
\end{table}

\subsection{Number of parameters comparison}
\label{num_param}
Here we show the comparison of the number of parameters for different popular models:
%
%
\begin{table}[H]
\centering
\small
\setlength{\tabcolsep}{4pt}
\resizebox{\textwidth}{!}{%
\begin{tabular}{lccccc}
\toprule
\textbf{Method}
& \multicolumn{3}{c}{\textbf{INR classification}}
& \textbf{CNN regression}
& \textbf{Transformer zoo} \\
\cmidrule(lr){2-4}\cmidrule(lr){5-5}\cmidrule(l){6-6}
& \textbf{MNIST} & \textbf{FMNIST} & \textbf{CIFAR-10}
& \textbf{CIFAR-10-GS} & \textbf{MNIST-Trans.} \\
\midrule
MLP (Monomial-NFN setup)   & $2.00$  & $2.00$  & $2.00$  & --     & --      \\
MLP (Transformer setup)   & --      & --      & --      & --     & $0.933$ \\
StatNN                    & --      & --      & --      & $1.06$ & $0.203$ \\
DWS                       & --      & $0.553$ & --      & --     & --      \\
NG-GNN (no probes)        & --      & $0.331$ & --      & --     & --      \\
NFN$_{\mathrm{NP}}$       & $15.0$ & $15.0$ & $16.0$ & $2.03$ & --      \\
NFN$_{\mathrm{HNP}}$      & $22.0$ & $22.0$ & $42.0$ & $2.81$ & --      \\
ScaleGMN                  & --      & $1.128$ & --      & --     & --      \\
ScaleGMN-B                & --      & $1.689$ & --      & --     & --      \\
Monomial-NFN              & $22.0$ & $20.0$ & $16.0$ & $0.25$ & --      \\
Monomial-NFN Quasi        & $22.2$ & $20.5$ & $16.3$ & $0.26$ & --      \\
Transformer-NFN           & --      & --      & --      & --     & $1.812$ \\
Transformer-NFN Large     & --      & --      & --      & --     & $2.857$ \\
Transformer-NFN Quasi     & --      & --      & --      & --     & $1.894$ \\
\midrule
ProbeGen                  & $0.4$ & $0.4$ & $0.45$ & $1.0$ & 1.8     \\
HiddenProbe               & $0.4$ & $0.4$ & $0.45$ & $1.0$ & 1.8      \\
\bottomrule
\end{tabular}%
}
\caption{Number of trainable parameters (millions, M) for published baselines
on three INR classification tasks, CIFAR-10-GS regression,
and MNIST-Transformers. INR and CNN counts are taken from
\citet{tran2024monomial,tran2026quasi} and the ScaleGMN supplement
\citep{kalogeropoulos2024scale}; MNIST-Transformers counts are from
\citet{tran2025b,tran2026quasi}. Published model configurations may
vary across sources and need not exactly match those used in our accuracy
comparisons. ProbeGen and HiddenProbe counts are provided by the authors.
A dash indicates that a count has not been verified for the specified task.}
\label{tab:parameter_counts_selected}
\end{table}

\subsection{Time measurement}
\label{sec:times}
Here we show the time comparison of popular deep weight space models and ProbeGen and \our on the MNIST INR benchmark. 

\begin{table}[H]
\centering
\caption{Runtime on the full MNIST-INR training split
(55{,}000 target INRs, batch size 32).}
\label{tab:mnist_inr_runtime}
\begin{tabular}{lrr}
\toprule
Method & Time (s) & Relative to MLP \\
\midrule
MLP         & 103.461 & 1.00$\times$ \\
NFN         & 90.755  & 0.88$\times$ \\
ProbeGen    & 109.122 & 1.05$\times$ \\
HiddenProbe & 155.508 & 1.50$\times$ \\
NFT         & 164.774 & 1.59$\times$ \\
DWS         & 174.766 & 1.69$\times$ \\
\bottomrule
\end{tabular}
\end{table}
\subsection{Tuning budget}
For benchmarks on which ProbeGen had not previously been evaluated, we ran ProbeGen ourselves. We used the same hyper-parameter selection procedure for ProbeGen and for \our, ensuring that both methods were tuned under the same validation-based protocol. In all cases, we used the original benchmark train/validation/test splits and selected hyper-parameters using only the validation split.

For both methods, we fixed the batch size to \(32\), the number of learned probes to \(128\), and the weight decay to zero. All models were optimized with Adam.

For Transformer experiments, we used $256$ probes.

For each dataset, we performed a 27-configuration sweep over the learning rate and learning-rate scheduler. We considered learning rates

$$
    \{3\times10^{-4},\,5\times10^{-4},\,7\times10^{-4}\}.
$$

For the ReduceLROnPlateau scheduler, we evaluated all combinations of

$$
    \text{patience}\in\{3,5\},
    \,
    \text{factor}\in\{0.2,0.3,0.5,0.7\},
$$

for each learning rate, resulting in \(24\) configurations. We additionally evaluated a cosine-annealing schedule for each of the three learning rates, yielding \(3\) further configurations and \(27\) configurations in total.

Each hyper-parameter configuration was first trained for \(20\) epochs using seed \(0\). For test-accuracy prediction, configurations were ranked exclusively according to validation Kendall's \(\tau\); test performance was not used for model selection. The best configuration for each method was then retrained using five seeds, \(0,\ldots,4\), for \(60\) epochs, and we report the mean and standard error of the resulting test Kendall's \(\tau\).

For CIFAR-10 INR classification, we use the same 27-configuration optimization grid for both ProbeGen and \our, but select the best configuration according to validation classification accuracy. The selected configuration is subsequently evaluated over five seeds. The final CIFAR-10 INR runs are trained for \(30\) epochs.

\end{document}